\documentclass{article}

\usepackage[main, final]{neurips_2026}
\usepackage[utf8]{inputenc}
\usepackage[T1]{fontenc}
\usepackage{hyperref}
\hypersetup{
  hidelinks,
  pdftitle={Commutator Memory: Sparse, Path-Local Reading and Steering in Language Models},
  pdfauthor={John Sweeney}
}
\usepackage{url}
\usepackage{booktabs}
\usepackage{amsfonts}
\usepackage{amsmath}
\usepackage{amssymb}
\usepackage{amsthm}
\newtheorem{theorem}{Theorem}
\newtheorem{proposition}{Proposition}
\newtheorem{lemma}{Lemma}
\newtheorem{corollary}{Corollary}
\usepackage{xcolor}
\usepackage{graphicx}
\usepackage{multirow}
\usepackage{array}
\usepackage{placeins}

\newcommand{\R}{\mathbb{R}}
\newcommand{\grad}{\nabla}

\title{Commutator Memory: Sparse, Path-Local Reading and Steering in Language Models}

\author{%
  John Sweeney \\
  Sideplane AI \\
  \texttt{john.sweeney@sideplane.ai}
}

\begin{document}

\maketitle

\begin{abstract}
Gradient updates on different data generally do not commute: training a language model on two data sources in opposite orders gives different weights, even with the same data and total exposure. Loss or benchmark deltas show that the models differ, not where. We ask whether this path dependence leaves a \emph{parametric training-history memory}: a weight component that flips sign when the two sources are swapped, is localized in output space, changes the held-out loss gap between the two orders under targeted interventions, and reveals which trained model came from which order. For one small SGD step of size $\eta$ on each of sources $A$ and $B$, the weight difference $\theta_{AB}-\theta_{BA}$ is, to leading order, $\eta^2 b_{AB}$, where $b_{AB}=H_Bg_A-H_Ag_B$ is the Lie bracket of the two gradient fields at the base model. We define \emph{commutator memory} by projecting the bracket through the logits into one score per vocabulary token; the scores sum to the bracket's prediction of the gap. The scores are localized: on three models, the same readout of the measured $\theta_{AB}-\theta_{BA}$, or of a bracket from disjoint batches, shares 82--99\% of the original top-20 tokens, versus 35--49\% for norm-matched random directions. They are causally actionable: in Qwen-3-4B SFT, downweighting the ten tokens with the largest predicted share of the gap closes a median 32\% of the measured gap, while frequency-matched tokens with near-zero scores have almost no effect. The weights themselves carry the component: projecting the difference between the two trained models onto $b_{AB}$ identifies which came from which order in 92\% of cases across four LLMs (chance 50\%). Controlled tests also cover matched-batch DPO, a frozen-rollout GRPO-style objective, and an AdamW endpoint check. The memory is defined per source pair, not per example, and its projection on $b_{AB}$ decays with further training.
\end{abstract}

\section{Introduction}

Modern LLM post-training composes multiple data sources sequentially: domain pre-training, SFT, preference tuning, safety updates. Prior fine-tuning and training-order studies show that data order and task order can change final performance, and that recently learned information can leave a detectable representation-space trace~\citep{dodge2020finetuning,chen2024order_finetuning,krasheninnikov2026fresh}. In current practice, however, order effects are usually observed only through aggregate benchmark deltas, leaving the practitioner with no view of \emph{where} the interaction between training stages lands in output space. The natural worry is that the order-dependent residue, what the two endpoints disagree on, is diffuse optimizer noise: spread thinly across capabilities, interpretable only in retrospect, and untargetable. Figure~\ref{fig:method} previews our answer: it is not.

\begin{figure}[!t]
\centering
\includegraphics[width=\textwidth]{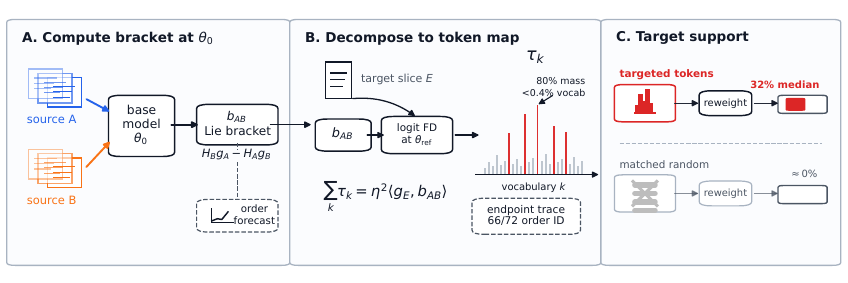}
\caption{\textbf{Commutator memory pipeline.} One bracket produces an order forecast, token map, intervention target, and endpoint trace.}
\label{fig:method}
\end{figure}

This paper shows the residue is not diffuse at the readout level. The order-asymmetric component of source interaction has a concentrated, interpretable, vocabulary-level signature. At the calibrated step sizes studied here (BCH-local: small enough that the second-order expansion of $\theta_{AB}-\theta_{BA}$ holds), the few tokens that carry most of the effect can be estimated from the base model, specified source data, and evaluation slice before both orderings are run. We use ``memory'' operationally, as \emph{parametric training-history memory}. With training content and total exposure held fixed, composing two updates leaves an order-specific component in the weights; exchanging their order flips its sign. To count as memory, this component should be \emph{localized} (concentrated on identifiable outputs), \emph{causally actionable} (intervening on the predicted support changes the order gap), and \emph{assignable} from paired endpoints (the endpoint weights retain an antisymmetric trace of the order that produced them). Access is pair-conditioned: the two candidate sources must be supplied, and we do not claim item-addressable recall, personalization, or long-horizon agent memory. Table~\ref{tab:map} maps the central objects in plain language and marks which are inherited from prior work and which are new here.

\begin{table}[t]
\centering
\small
\caption{Plain-language map of the central objects ($A,B$: the two training data sources; $E$: the held-out evaluation slice). Inherited from prior work: the endpoint-difference identity~\citep[Lemma~2.1]{geometry_sequential_learning}, the scalar criterion at $\theta_0$~\citep{rukhovich2025commute} and at $\theta_{\rm ref}$~\citep{geometry_sequential_learning}, the bracket-only step~\citep[App.~E.8]{geometry_sequential_learning}, and the fixed-clock AdamW regime, in which Adam's buffers advance per step regardless of $\eta$~\citep[Thm.~1]{sweeney2026optimizermemory}. This paper develops the token-level readout, order-gap interventions, paired-endpoint assignment, bracket-only guarantee, and AdamW closure test.}
\label{tab:map}
\begin{tabular}{@{}>{\raggedright\arraybackslash}p{0.24\textwidth}>{\raggedright\arraybackslash}p{0.53\textwidth}>{\raggedright\arraybackslash}p{0.16\textwidth}@{}}
\toprule
Object & Role & Origin \\
\midrule
$b_{AB}=H_Bg_A-H_Ag_B$ & effect of swapping $A$ and $B$: $\theta_{AB}-\theta_{BA}=\eta^2b_{AB}+O(\eta^3)$ & inherited \\
$\sigma=\langle g_E(\theta_{\rm ref}),b_{AB}\rangle$ & one scalar predicting which order has lower held-out loss & inherited \\
$\tau_k$ (commutator memory) & token-level decomposition of $\sigma$; $\sum_k\tau_k=\eta^2\sigma$ & new \\
order-gap closure & fraction of the AB/BA NLL gap removed by an intervention & new \\
$\Delta s=\langle\theta_{AB}-\theta_{BA},b_{AB}\rangle$ & which order produced each of two endpoints & new \\
$\theta_{\rm BO}$ & bracket-only step; beats both orders under Prop.~\ref{prop:bo} & partly inherited \\
$\tau_{\rm Adam}$ & AdamW analogue on the optimizer state (parameters, moments, step counter) & partly inherited \\
\bottomrule
\end{tabular}
\end{table}

\paragraph{Setup and token-level decomposition.} Let $\theta_0$ be the base model and $\eta$ the step size, and let $g_D:=\nabla\mathcal{L}_D(\theta_0)$ and $H_D:=\nabla^2\mathcal{L}_D(\theta_0)$ be the gradient and Hessian of the training loss of source $D$. For two gradient sources $A,B$ with single-step SGD, the endpoint difference obeys $\theta_{AB}-\theta_{BA}=\eta^2 b_{AB}(\theta_0)+O(\eta^3)$, where $b_{AB}:=H_Bg_A-H_Ag_B$~\citep[Lemma~2.1]{geometry_sequential_learning}; we give a self-contained Euler/BCH derivation in Appendix~\ref{app:bracket_derivation}. Projecting this bracket onto a target gradient gives a scalar order criterion~\citep{rukhovich2025commute}. Evaluated at the Trotter reference $\theta_{\text{ref}}:=\theta_0-\eta(g_A+g_B)$, the scalar $\sigma=\langle g_E(\theta_{\text{ref}}),b_{AB}(\theta_0)\rangle$ predicts the better pairwise order with $82$--$92\%$ overall sign accuracy and $82$--$100\%$ on the decisions with the largest predicted gaps~\citep{geometry_sequential_learning}. The scalar says which order, not where. Let $e(x,y;\theta):=\mathrm{softmax}(z(x;\theta))-\mathrm{onehot}(y)$ be the cross-entropy error in logit space and $\delta z(x;v):=J_\theta z(x)\,v$ the logit-space directional derivative along a parameter displacement $v$. We decompose $\sigma$ across vocabulary by the chain rule:
\begin{equation}
  \tau_k(A, B; E) := \mathbb{E}_{(x, y) \sim E}\!\bigl[\, e_k(x, y; \theta_{\text{ref}}) \cdot \delta z_k(x;\, \eta^2 b_{AB}(\theta_0)) \,\bigr], \qquad \sum_k \tau_k = \eta^2 \sigma.
\label{eq:tau_exact}
\end{equation}
The identity is exact for the true logit JVP; the implementation uses central finite differences and reports the corresponding finite-difference token readout (Appendix~\ref{app:tau_implementation}). We call the map from a parameter displacement to the scores $\tau_k$ the \emph{token readout}, and $\tau$ itself the \emph{commutator memory} of source pair $(A,B)$ on slice $E$: a property of the training path, not a state-invariant property of the trained endpoint. Operationally, it is a per-token attribution of the commutator of the two updates, distinct from activation-level mechanistic interpretability: given two sources and an evaluation slice, the method returns a scalar order forecast, a signed token report, and a set of tokens to downweight.

\paragraph{Contributions.} We introduce \emph{commutator memory}, a token-level readout of the Lie bracket of two training updates. We evaluate it by three tests: \textbf{localization}, where bracket-derived token supports are sparse, pair-specific, and recover empirical order-effect supports in endpoint checks; \textbf{causal effect}, where interventions on predicted harmful tokens reduce held-out ordering gaps while $\tau$-neutral controls do not; and \textbf{paired-endpoint assignment}, where an antisymmetric statistic tells which order produced each of two $k{=}1$ SGD endpoints, across four LLMs. Additional DPO, GRPO-style, and AdamW experiments test matched-path extensions beyond the main SFT setting.

\paragraph{Evidence map.} The main SFT experiments use empirically BCH-local SGD with $k=1$--$8$ steps on two-source pairs. We then ask whether the same bracket remains structured under matched-batch DPO, a frozen-rollout reward surrogate, iterative correction toward the observed reverse-order endpoint, and AdamW (an endpoint check on the optimizer state); matched means that the bracket is computed from the same batches or rollouts that produced $\theta_{AB}$ and $\theta_{BA}$. This sequence of tests keeps the main claims fixed (localized, causally actionable, paired-endpoint-readable commutator memory), while the later regimes test the same object on realized update paths. The closest prior work on training-history memory is the activation recency probe of \citet{krasheninnikov2026fresh}; our object is different because it is a signed, pair-specific bracket readout computed before either order is run, projected to output tokens, and used for causal order-gap interventions. The closest causal-localization comparison is mechanistic data attribution~\citep{chen2026mda}: both find that an apparently diffuse training phenomenon is carried by a sparse causal support, but their support consists of training samples for circuit emergence, while ours consists of output tokens for the antisymmetric order residue.

\section{Related Work}
\label{sec:related}

\textbf{Training order, curricula, and temporal traces.} Curriculum learning studies deliberately chosen presentation orders and frames curricula as continuation methods that can speed convergence or guide non-convex optimization~\citep{bengio2009curriculum}. Fine-tuning work shows that data-order seeds contribute substantially to BERT fine-tuning variance~\citep{dodge2020finetuning}, and that sequential intermediate-task order can help or hurt CodeBERT transfer in software-engineering tasks~\citep{chen2024order_finetuning}. Most directly, \citet{krasheninnikov2026fresh} show that sequential fine-tuning leaves a linear activation-space recency signal: probes can distinguish early vs.\ late entity datasets and the model can be trained to report a stage. Their work and ours both ask what training history survives later learning. We study the weight-space noncommutativity of two training updates: the bracket gives a signed AB-vs-BA displacement, an output-token decomposition, causal token interventions, and a paired-endpoint assignment statistic. We also measure how this displacement attenuates under continued training. We use ``memory'' in this parametric, pair-conditioned sense, not in the sense of agent memory systems with storage and retrieval.

\textbf{Lie brackets and splitting geometry.} The BCH/Magnus/splitting viewpoint is classical in differential equations and geometric numerical integration~\citep{magnus1954exponential,blanes2009magnus,hairer2006geometric}. The ICML paper by \citet{geometry_sequential_learning} proves the SGD bracket identity and shows that the scalar $\sigma=\langle g_E,b_{AB}\rangle$ predicts ordering quality across multi-domain suites; we inherit this identity and lift it to token-level localization, intervention, and paired-endpoint readability. \citet{rukhovich2025commute} project the same bracket onto a target-loss gradient as a local optimality criterion for multi-domain learning. \citet{sweeney2026optimizermemory} shows that optimizer state that advances with the step count, such as AdamW moment buffers and de-biasing counters, makes the effect of reordering the same data first order in $\eta$ rather than second order; our AdamW appendix replays this optimizer state, but only as an endpoint check after training.

\textbf{Continual learning and editing.} Catastrophic interference and continual-learning methods such as EWC, SI, GEM, and PCGrad~\citep{mccloskey1989catastrophic,kirkpatrick2017overcoming,zenke2017continual,lopez2017gradient,yu2020gradient} address forgetting or gradient conflict; ROME/MEMIT~\citep{meng2022locating,meng2022mass} show how factual content can be written into weights; our complementary question is what order-specific interaction is created when parameter writes are composed. Commutator memory instead resolves the antisymmetric path residue $\theta_{AB}-\theta_{BA}$ into output-token support. Task arithmetic~\citep{ilharco2023task} composes whole-model fine-tuning vectors as task differences; the bracket $b_{AB}$ is an order-asymmetric residue \emph{within} a single training sequence, and loss responses to small masked bracket edits on Qwen-3-4B are near-additive (4.3\% mean relative error; App.~\ref{app:triplet_additivity}).

\textbf{Mechanistic localization and data attribution.} Mechanistic-interpretability work localizes behavior to circuits, activation features, and editable internal representations~\citep{olah2020zoom,elhage2021mathematical,wang2022interpretability,conmy2023towards,bricken2023monosemanticity,templeton2024scaling,belrose2023tunedlens}. Influence functions, datamodels, TRAK, and recent LLM influence studies~\citep{koh2017understanding,ilyas2022datamodels,park2023trak,grosse2023studying} attribute predictions or behaviors to training examples. Mechanistic data attribution~\citep{chen2026mda} traces interpretable units to high-influence training samples; our attribution is not sample-level or activation-level, but a token-level decomposition of the commutator of two training updates.

\textbf{Attribution, fingerprinting, and training protocols.} Watermarking~\citep{kirchenbauer2023watermark,zhao2024provable} embeds a signal during generation; commutator memory arises without one. Membership inference~\citep{shokri2017membership,carlini2022membership} recovers set membership, not order. Seed fingerprinting~\citep{tong2025seedprints} recovers the initialization seed from seed-induced prediction biases; our paired-endpoint statistic identifies which of two candidate sources was trained first ($k{=}1$ SGD). DPO~\citep{rafailov2023direct}, PPO~\citep{schulman2017ppo}, and GRPO~\citep{shao2024deepseekmath} provide the preference/RL objectives used only in the controlled matched-path tests. Adam and AdamW~\citep{kingma2014adam,loshchilov2019decoupled} motivate the lifted-state optimizer check rather than the main SGD claims.

\section{Background: Lie Bracket Analysis of Alternating Training}
\label{sec:background}

\subsection{Problem Setup}

Let $\theta \in \R^p$ denote model parameters and $\mathcal{L}_A, \mathcal{L}_B$ denote loss functions on domains A and B. Consider alternating gradient descent:
\begin{align}
\theta_1 &= \theta_0 - \eta \grad \mathcal{L}_A(\theta_0) \\
\theta_2 &= \theta_1 - \eta \grad \mathcal{L}_B(\theta_1)
\end{align}
versus the reverse ordering (B then A).

The \emph{ordering gap} measures the difference:
\begin{equation}
\Delta(\theta_0) = \mathcal{L}_E(\theta_{AB}) - \mathcal{L}_E(\theta_{BA})
\end{equation}
where $\mathcal{L}_E$ is a held-out evaluation loss and $\theta_{AB}, \theta_{BA}$ denote final parameters after sequences AB and BA.

\subsection{Lie Bracket Decomposition}

\emph{Smoothness assumption (used throughout).} The source and target losses considered below are $C^3$ on the local neighborhood of $\theta_0$ and $\theta_{\rm ref}$ with bounded third derivatives, so the $O(\eta^3)$ endpoint remainders and target-loss Taylor remainders are uniform; this is the regularity inherited from \citet{geometry_sequential_learning} Lemma~2.1.

For discrete gradient descent with step size $\eta$, the parameter difference after two alternating steps satisfies
\begin{equation}
\theta_{AB} - \theta_{BA} = \eta^2 b_{AB}(\theta_0) + O(\eta^3),
\label{eq:endpoint_diff}
\end{equation}
where the \emph{bracket} $b_{AB}$ is defined via Hessian-vector products:
\begin{equation}
b_{AB}(\theta) := H_B(\theta) g_A(\theta) - H_A(\theta) g_B(\theta)
\label{eq:bracket}
\end{equation}
with $g_D=\nabla_\theta\mathcal{L}_D$ and $H_D=\nabla^2_\theta\mathcal{L}_D$ (Lemma 2.1 of \citet{geometry_sequential_learning}; Euler/BCH derivation in Appendix~\ref{app:bracket_derivation}). Appendix~\ref{app:bracket_derivation} also reads $b_{AB}$ as the curvature of an optimizer connection. We compute $b_{AB}$ via exact HVPs~\citep{pearlmutter1994fast}, cast to fp32 on bf16 models.

\subsection{Target Score and Trotter Reference}

Following \citet{geometry_sequential_learning} and the standard splitting perspective in geometric numerical integration~\citep{hairer2006geometric}, define the first-order Trotter reference $\theta_{\text{ref}} := \theta_0 - \eta (g_A(\theta_0) + g_B(\theta_0))$, the shared first-order endpoint of the two schedules. Taylor-expanding around $\theta_{\text{ref}}$ and substituting eq.~\eqref{eq:endpoint_diff} yields the bracket score
\begin{equation}
\Delta_E(A,B)
  := \mathcal{L}_E(\theta_{AB}) - \mathcal{L}_E(\theta_{BA})
  = \eta^2 \langle g_E(\theta_{\text{ref}}), b_{AB}(\theta_0) \rangle + O(\eta^3).
\label{eq:target_score}
\end{equation}
The $O(\eta^4)$ loss-linearization remainder for the \emph{actual} endpoint displacement is dominated by the $O(\eta^3)$ Euler truncation in $\theta_{AB}-\theta_{BA}\approx\eta^2 b_{AB}$. Proposition~\ref{prop:bo} (\S\ref{sec:discussion}, App.~\ref{app:bo_proof}) gives a conditional leading-order comparison: when the symmetric drift on $E$ raises the target loss ($\mu>0$, \S\ref{sec:discussion}) and the sign of $\sigma$ is estimated correctly, the bracket-only step at $\theta_{\rm ref}$ has lower target loss than both pure orders up to the stated remainder. The companion closure identity $\theta_{AB}-\eta^2 b_{AB}=\theta_{BA}+O(\eta^3)$ underlies the matched-batch correction results summarized in Section~\ref{sec:controlled_tests} and detailed in Appendix~\ref{app:dpo_extended} and Appendix~\ref{app:grpo}.

\subsection{Token-Level Decomposition}

The bracket $b_{AB}\in\mathbb{R}^p$ acts on parameters; Eq.~\eqref{eq:tau_exact} maps it to tokens through the logit Jacobian. Token $k$ contributes $\tau_k$ to the leading-order gap; $\tau_k>0$ means $AB$ disadvantages token $k$ relative to $BA$ at the bracket-prediction level. Individual entries are reported in the model's native logit coordinates; the aggregate identity is invariant to position-wise logit shifts because $\sum_k e_k=0$. In experiments we compute the logit JVP by central finite difference with step $1.0$ after exact HVPs. All downstream sections characterize this readout: sparsity (\S\ref{sec:sparsity}), support interpretation (\S\ref{sec:interpretability}), and the controlled tests beyond supervised fine-tuning in Section~\ref{sec:controlled_tests}.

\section{Experimental Setup}
\label{sec:setup}

\subsection{Models and Domains}

Models: Qwen-3-4B~\citep{qwen3_2025}, Qwen-2.5-1.5B~\citep{qwen25_2024}, Llama-3.1-8B~\citep{llama31_2024}, Llama-3.2-1B~\citep{llama32_2024}; DPO sparsity also uses Qwen-3-8B, and the GRPO-style reward surrogate uses Qwen-2.5-0.5B-Instruct. SFT domains: \texttt{code, news, math, legal, biomedical, wikipedia} from The Pile~\citep{gao2020pile}; 512 tokens/sequence. Sparsity and token-prediction runs use all $C(6,2){=}15$ pairs $\times$ 3 seeds. The SFT intervention rows in Table~\ref{tab:interventions} use three fixed domain pairs $\times$ 30 seeds ($90$ trials/model; a trial is one seeded run of one pair); single-step and iterative endpoint correction (edits of $\theta_{AB}$ toward $\theta_{BA}$) use the fixed six-pair subset stated in their tables. DPO uses pairs of UltraFeedback~\citep{cui2024ultrafeedback} source partitions. The GRPO-style protocol uses frozen matched rollouts and two analytic reward sources (math correctness and brevity/format compliance). Learning rates $\eta$ are selected by an empirical BCH-locality check (Qwen-3-4B $6.42\text{e-}4$, Qwen-2.5-1.5B $1.91\text{e-}3$, Llama-3.1-8B $1.28\text{e-}3$, Llama-3.2-1B $8.0\text{e-}4$; Appendix~\ref{app:eta_autopilot}); for the DPO and GRPO-style correction experiments the local $\eta$ is recalibrated per protocol (Appendix~\ref{app:eta_calibration} and Appendix~\ref{app:grpo}). HVPs are computed in fp32 via upcast; held-out evaluation uses documents disjoint from training.

\subsection{Closure metric and measurement floor}
\label{sec:floor}

The closure metric used in the intervention and correction sections is $C := 1 - |\Delta_{\text{edited}}| / |\Delta_{\text{baseline}}|$, where $\Delta_{\text{baseline}} = \mathcal{L}_E(\theta_{AB}) - \mathcal{L}_E(\theta_{BA})$ and $\Delta_{\text{edited}}$ is the same quantity after the bracket-based edit. The metric is bounded above by $1$ (perfect closure) and unbounded below; its sampling variance diverges as $|\Delta_{\text{baseline}}| \to 0$, since the denominator approaches the per-trial NLL noise floor (fp32 $\sim 2{\times}10^{-5}$; SFT bf16 $\sim 10^{-4}$). We adopt three complementary reporting conventions that are inherited by every downstream result.

\textbf{(i) Above-floor subset.} For each protocol, a floor on $|\Delta_{\text{baseline}}|$ is fixed before evaluation: $10^{-4}$ for the Qwen-2.5-1.5B fp32 SFT intervention (28 of 90 trials retained), $5{\times}10^{-3}$ for Qwen-3-4B iterative correction (5 of 6 pairs retained: news vs.\ biomedical has $|\Delta|=8.5{\times}10^{-4}$ at $k{=}1$, below every other pair), and $2{\times}10^{-5}$ as the DPO diagnostic denominator floor used in Appendix~\ref{app:r79_freshpairs}. We report all-trial and above-floor statistics in parallel where both are computable.

\textbf{(ii) Winsorized effect-size summaries.} Because closure ratios are heavy-tailed near the measurement floor, we report the median and a winsorized 5/95 mean (App.~\ref{app:dpo_stats_convention}) as descriptive effect-size summaries. Inferential claims use paired comparisons or pair-level tests rather than the raw trial-level mean.

\textbf{(iii) Pair-level summary.} The DPO closure summary uses source-pair medians of the bracket-minus-random closure difference. The criterion fixed in advance, a per-trial bootstrap $95\%$ CI strictly above zero, holds for $11/12$ source pairs (the original six UltraFeedback pairs and six disjoint fresh pairs; App.~\ref{app:r79_seed_cluster}); all $\mathbf{12/12}$ pair medians are positive. The pairs share source domains, so this sign agreement is descriptive, not a significance test. The trial-level count (316/360 = 87.8\% matched trials with bracket $>$ random) is within-cluster consistency, not an independent-sample p-value. Winsorized 5/95 means and seed-clustered bootstrap CIs are effect-size summaries, not decision criteria.

\section{Results}

The results follow the three tests set out in the introduction. Sections~\ref{sec:sparsity}--\ref{sec:interpretability} test whether the antisymmetric residue is \emph{localized}; Section~\ref{sec:interventions} tests whether the localized support is \emph{causal}; Section~\ref{sec:forensic} tests whether the endpoint weights retain a \emph{readable} antisymmetric trace. The other sections are validity checks and controlled matched-path tests: agreement with the prior scalar predictor, persistence across multiple SGD steps, matched-path correction, and the lack of transfer to unmatched GRPO rollouts.

\subsection{Localized: Sparse, Pair-Specific Token Readout}
\label{sec:sparsity}

\begin{figure}[t]
\centering
\includegraphics[width=0.8\textwidth]{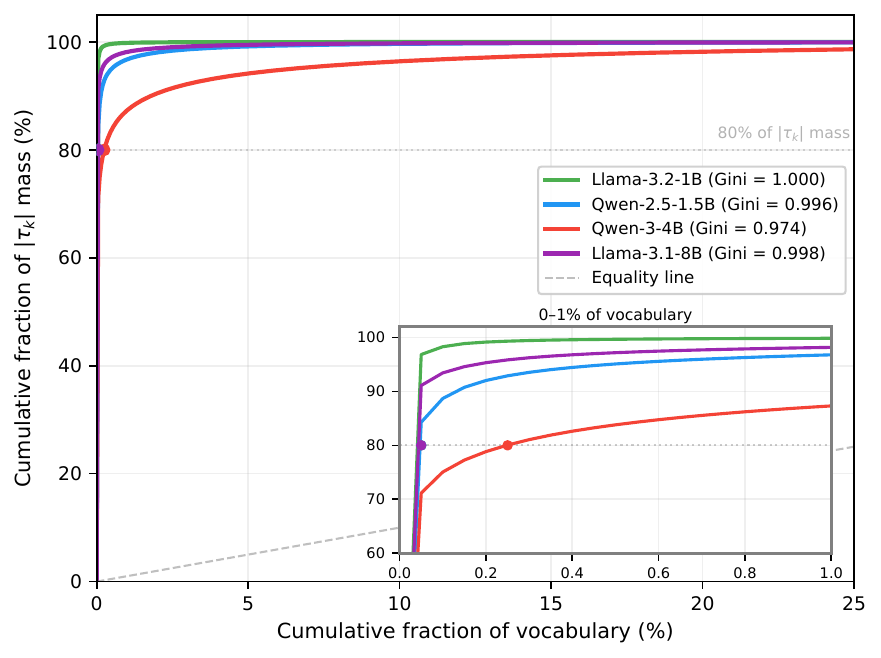}
\caption{Concentration diagnostics for the token-level bracket readout. Lorenz curves show the operational finite-difference $\tau$ readout across four reported main runs; Table~\ref{tab:sparsity} reports the separate fp32 random-direction control subset used for the support-specificity claim.}
\label{fig:sparsity}
\end{figure}

The finite-difference readout is highly concentrated (Fig.~\ref{fig:sparsity}), but concentration alone is not the specificity claim: random directions can also produce heavy-tailed logit readouts. Our bracket-specific evidence is support alignment. Bracket-derived supports are pair-specific and, when the measured $\theta_{AB}-\theta_{BA}$ is passed through the same readout, recover empirical order-effect supports far above the corresponding random-direction controls (App.~\ref{app:sparsity_nulls}).

\begin{table}[h]
\centering
\caption{Concentration and support-specificity checks from the fp32 random-direction control subset (3 pairs $\times$ 3 seeds/model). ``80\% mass'': fraction of the vocabulary holding $80\%$ of the $|\tau|$ mass (mean/median). Endpoint (the measured $\theta_{AB}-\theta_{BA}$) and resampled (a bracket from disjoint batches) columns report top-$20$ overlap with the bracket support; the random column reports mean top-$20$ overlap for norm-matched random controls.}
\label{tab:sparsity}
\begin{tabular}{@{}p{0.17\textwidth}p{0.10\textwidth}p{0.17\textwidth}p{0.22\textwidth}p{0.13\textwidth}@{}}
\toprule
Model & Precision & 80\% mass mean/med. & Endpoint/resampled overlap & Random overlap \\
\midrule
Llama-3.2-1B & fp32 & 1.47\% / 1.47\% & 99\% / 97\% & 35\% \\
Qwen-2.5-1.5B & fp32 & 0.99\% / 1.16\% & 99\% / 97\% & 39--40\% \\
Qwen-3-4B & fp32 & 0.30\% / 0.042\% & 82\% / 93\% & 36--49\% \\
\bottomrule
\end{tabular}
\end{table}

\paragraph{Support specificity, not Gini alone.} Top-token sets are pair-specific rather than a universal high-influence vocabulary set: cross-pair Jaccard is $0.004$ on Qwen-3-4B (Fig.~\ref{fig:token_heatmap}). In the endpoint checks used in Table~\ref{tab:sparsity}, empirical endpoint displacements and batch-resampled brackets preserve the bracket's top-token support, while norm-matched random controls share only the tokens that any direction excites. Frequency-normalized $\kappa_k=\tau_k/(\Pr(y{=}k)+\varepsilon)$ is reported as a label-frequency diagnostic; support recovery, not $\kappa$-Gini, is the specificity test. Appendix~\ref{app:obs_gauge} checks how much of $b_{AB}$ the token readout can see.

\subsection{Localized: The Sparse Support Is Interpretable}
\label{sec:interpretability}

Top-20 $|\tau_k|$ tokens fall into three token categories, each pair with its own signature. \textbf{Domain-marker tokens} dominate code-vs-news at $\theta_{\rm ref}$ (lexical proper nouns: ``Facebook'', ``Amazon'', ``Washington''; a $\theta_0$-reference variant of the same pair surfaces syntactic markers \texttt{xml}, \texttt{\$}, \texttt{php}, App.~\ref{app:reachability}) and code-vs-math (academic metadata + LaTeX source structure). \textbf{Digit asymmetry} dominates code-vs-legal (9/20 are digits 0--9; token ``0'' alone has $\tau=+23.1$) plausibly array indices vs.\ section numbers. \textbf{BPE morpheme fragments} dominate code-vs-biomedical (``-ervation'' from \emph{observation/conservation/preservation}; ``-fluence''; ``-ulatory''): polysyllabic domain terms collapse into a few shared subword pieces, concentrating $|\tau_k|$ via many-to-few mapping. Token heatmap and per-pair token examples: Appendix~\ref{app:interpretability_extended}, Fig.~\ref{fig:token_heatmap}.

\paragraph{Category summary.} The three categories identify tokens with the largest error-weighted logit response $e_k\,\delta z_k$ along the bracket. The prediction and label parts of $\tau_k$ (App.~\ref{app:interpretability_extended}) are comparable in aggregate (ratio $0.88$); $30\%$ of $|\tau|$ mass comes from prediction-only tokens. Smoothed $\kappa_k = \tau_k/(\Pr(y{=}k)+\varepsilon)$ is a label-frequency diagnostic. The operational readout is consistently heavy-tailed across the tested architectures; the bracket-specific claim is the stability and endpoint alignment of the high-mass support. The categories recur across pairs; reference and eval-slice choices can change the surface tokens, but the observed high-mass categories remain domain-linked. Per-model detail: Appendix~\ref{app:interpretability_extended}.

\paragraph{Concrete examples.} In held-out code-vs-news text, ``Ask HN: Should I include my GPA and/or transcripts when applying for jobs?'' starts with \texttt{Ask}, the largest-$|\tau|$ token of that slice ($\tau=-3.51$); in code-vs-biomedical text, ``the ambulatory health care environment'' contains \texttt{ulatory}, a top-10 $|\tau|$ subword ($\tau=0.437$). The readout is aggregated over the slice, not scored per sentence (Fig.~\ref{fig:examples}).

\subsection{Agreement with the Scalar Bracket Predictor}
\label{sec:baselines}

The scalar bracket-projection predictor $\sigma = \langle g_E, b_{AB}\rangle$ for pairwise order quality was validated as an order predictor in prior ICML work~\citep{geometry_sequential_learning}, which reports $82$--$92\%$ overall sign accuracy and $82$--$100\%$ accuracy on highest-impact decisions. Our identity gives $\sum_k \tau_k = \eta^2\sigma$ (Eq.~\ref{eq:tau_exact}); we verify that our pipeline reproduces its sign accuracy on three models. Bracket sign-prediction accuracy on 18 pair-seed units (6 domain pairs $\times$ 3 seeds): \textbf{Qwen-3-4B 72\%}, \textbf{Qwen-2.5-1.5B 100\%}, \textbf{Llama-3.1-8B 100\%}, vs.\ first-order baselines (grad cosine, grad norm) at $11$--$56\%$ and random at $39$--$50\%$. Magnus-expansion validation ($R^2=0.994$ on Qwen-3-4B) and full baseline table: Appendix~\ref{app:magnus}, \ref{app:baselines}.

\subsection{Causal: Token Interventions Close the SFT Gap}
\label{sec:interventions}

To test whether the predicted support causally affects the ordering gap, we intervene on the ten tokens whose $\tau_k$ contributes most to the measured baseline gap (App.~\ref{app:intervention_types}) during alternating training. On Qwen-3-4B we use mean-preserving loss reweighting ($\alpha = 0.1$, i.e., 90\% downweight) on these label tokens; on Qwen-2.5-1.5B fp32 we use the lower-noise variant that scales the learning rate of those tokens' output-layer (lm-head) rows, on the predefined above-floor subset. Controls are random token sets with $\tau_k$ near zero, drawn to match the harmful tokens' label-frequency bins (30 trials per domain pair, held-out evaluation data disjoint from training).

\begin{table}[h]
\centering
\caption{Intervention results on two models (30 trials/pair, held-out evaluation). Qwen-3-4B reports all 90 trials. $^\dagger$Qwen-2.5 denotes Qwen-2.5-1.5B fp32 on the predefined above-floor subset (28/90 trials with $|\Delta_{\mathrm{baseline}}|>10^{-4}$). Harmful: the ten tokens whose $\tau_k$ contributes most to the measured gap; Random: frequency-matched tokens with $\tau_k\approx0$; entries are closure $C$ in \%; Winsor.\ mean: $5/95$ winsorized mean, with its bootstrap CI; Gap shrinks: fraction of trials with a smaller ordering gap after the intervention ($68/90$, $45/90$, $27/28$, $11/28$ by row).}
\label{tab:interventions}
\begin{tabular}{llcccc}
\toprule
Model & Condition & Winsor.\ mean & Median & 95\% CI & Gap shrinks \\
\midrule
Qwen-3-4B & Harmful & +31.2\% & +32.0\% & [+21.4\%, +39.7\%] & 75.6\% \\
Qwen-3-4B & Random & +1.2\% & +0.2\% & [--5.8\%, +8.7\%] & 50.0\% \\
\midrule
Qwen-2.5$^\dagger$ & Harmful & +48.6\% & +47.9\% & [+40.5\%, +57.1\%] & 96.4\% \\
Qwen-2.5$^\dagger$ & Random & --0.006\% & --0.003\% & [--0.03\%, +0.02\%] & 39.3\% \\
\bottomrule
\end{tabular}
\end{table}

On Qwen-3-4B (held-out, all $90$ trials), intervening on predicted harmful tokens reduces the ordering gap by a median $32\%$ (the winsorized-mean CI excludes zero), while random controls show no effect. The paired harmful-vs-random comparison yields a CI entirely above zero (harmful-minus-random gap reduction $[0.00096,\,0.0022]$ nats, $p < 0.001$, Cohen's $d = 0.53$). Both HVP terms matter: selecting tokens with the full bracket $H_Bg_A-H_Ag_B$ gives $d = 0.53$, whereas the $H_Bg_A$ term alone gives $d = 0.25$. On Qwen-2.5-1.5B (full fp32 with lm-head learning-rate scaling), the predefined above-floor protocol keeps only the $28$ of $90$ trials with $|\Delta_{\text{baseline}}| > 10^{-4}$ (the remaining $62$ trials are below the reporting floor); on this above-floor fp32 subset, learning-rate scaling on the harmful tokens yields a $+47.9\%$ median gap reduction, with the random control near zero. The paper's comparison is the within-model bracket-vs-random separation; cross-model magnitudes are not directly compared. The $32\%$ figure is the effect of this reweighting edit, not the share of the ordering gap that commutator memory explains: after the edit, the leading-order bracket still predicts $71\%$ of the original gap, and finite-step corrections, which change sign across pairs, do not explain a common remainder (App.~\ref{app:residual_gap}).

\subsection{Readable: The Endpoint Pair Identifies Which Order Was Trained}
\label{sec:forensic}

A complementary stringent test of ``memory'' is paired-endpoint assignment: after both orders have been run, can the bracket assign which endpoint came from which order? Given a public base $\theta_0$, candidate training-domain pair $(A, B)$, and paired alternate-order endpoints $\theta_{AB}, \theta_{BA}$, Theorem~\ref{thm:forensic_id} and Corollary~\ref{cor:delta_s} (App.~\ref{app:bracket_derivation}) give the $k=1$ SGD criterion: the antisymmetric statistic
\begin{equation}
\Delta s := \langle \theta_{AB} - \theta_{BA},\, b_{AB}\rangle
\end{equation}
has leading term $\eta^2\|b_{AB}\|^2$ when the first endpoint is $\theta_{AB}$, with the sign reversed if the endpoints are swapped. Thus the sign assigns the endpoints whenever the $O(\eta^2)$ term dominates finite-$\eta$, sampling, and numerical remainders. The single-endpoint estimator $\mathrm{sign}\,s(\theta)$, with $s(\theta):=\langle \theta - \theta_{\mathrm{ref}},\, b_{AB}\rangle$, reaches $77.8\%$ only on Qwen-3-4B and falls to chance at $1$B, $1.5$B, and $8$B (both $s(\theta_{AB})$ and $s(\theta_{BA})$ land on the same side of zero). The antisymmetric difference removes the shared first-order reference offset and the second-order symmetric term $c_{AB}$ (\S\ref{sec:discussion}) in the expansion, and it stays above chance at all four scales.

\begin{table}[h]
\centering
\caption{Paired-endpoint training-order assignment via $\mathrm{sign}(\Delta s)$ across four LLMs (6 pairs $\times$ 3 seeds = 18 pair-seed units per model). Replication across precision regimes (bf16 vs.\ fp32) in Appendix~\ref{app:forensic}.}
\label{tab:forensic_body}
\begin{tabular}{lcc}
\toprule
Model (precision) & $\mathrm{sign}(\Delta s)$ correct & Wilson 95\% CI \\
\midrule
Llama-3.2-1B (1B, fp32) & 14/18 = 77.8\% & $[55\%, 91\%]$ \\
Qwen-2.5-1.5B (1.5B, fp32) & 16/18 = 88.9\% & $[67\%, 97\%]$ \\
Qwen-3-4B (4B, bf16) & 18/18 = 100.0\% & $[82\%, 100\%]$ \\
Llama-3.1-8B (8B, bf16) & 18/18 = 100.0\% & $[82\%, 100\%]$ \\
\midrule
\textbf{Combined} & \textbf{66/72 = 91.7\%} & $\mathbf{[83\%, 96\%]}$ \\
\bottomrule
\end{tabular}
\end{table}

Assignment is binary, so chance is $50\%$; the combined Wilson interval $[83\%, 96\%]$ (Table~\ref{tab:forensic_body}) lies above it. In this paired-endpoint setting, antisymmetry supplies the leading-order sign criterion; the empirical results test whether it remains readable under finite-sample and finite-step effects. Scope: $k=1$ SGD; AdamW transfer is in App.~\ref{app:adamw}.

\subsection{Controlled Tests Beyond Supervised Fine-Tuning}
\label{sec:controlled_tests}

The bracket's top-token set persists beyond the single-step derivation: across $k\in\{1,2,4,8\}$ SGD steps per source, both Qwen models retain concentrated $\tau$ readouts, and $\geq98\%$ of the $k{=}8$ empirical top-$1\%$ tokens (among those occurring as labels) lie in the single-step top-$1\%$ $\tau$ set (Appendix~\ref{app:multistep}). This asymmetric recall is distinct from the equal-cardinality support-prediction check in Appendix~\ref{app:token_prediction}. The scalar BCH forecast degrades on Qwen-3-4B at the largest calibrated $\eta$, but the support remains useful: the same model still gives strong support recovery, output-row prediction (Spearman $\rho$ between $\tau_k$ and the output-row displacement, from $0.83$ at $k{=}1$ to $0.71$ at $k{=}8$), and $32\%$ intervention closure.

\textbf{Parameter-space correction.} Edits to the trained endpoint using $-\eta^2 b_{AB}$ recover the reversed endpoint only in matched-path settings: subtracting $\eta^2 b_{AB}$ once from $\theta_{AB}$ closes amounts that vary by pair, but $\tau_k$ predicts the actual per-token output-row displacement on three models (per-pair Spearman $\rho=0.78$--$0.90$; App.~\ref{app:endpoint_perpair}), and iterative projection toward the observed $\theta_{BA}$ closes about $90\%$ of the gap on the $5/6$ SFT pairs above floor. The conditional bracket-only update at $\theta_{\rm ref}$ also beats both pure orders when the symmetric drift on $E$ is target-disadvantageous ($\mu>0$, \S\ref{sec:discussion}) and the sign of $\sigma$ is estimated correctly: on Qwen-2.5-1.5B fp32, $12$ of the $13$ units with $\hat\mu>0$ satisfy $\mathcal{L}_E(\theta_{\rm BO})<\min(\mathcal{L}_E(\theta_{AB}),\mathcal{L}_E(\theta_{BA}))$, and the improvement over the better order regresses on the predicted $-\eta^2\hat\mu$ ($\hat\mu$ the estimated $\mu$) with slope $1.15$ (proof in App.~\ref{app:bo_proof}).

\textbf{Preference and reward-surrogate tests.} Matched-batch DPO gives the most direct correction test: subtracting the second-order $\eta^2 b_{AB}$ update from $\theta_{AB}$ closes a median $77\%$ of the held-out DPO gap to $\theta_{BA}$ (App.~\ref{app:r79_perpair}). The flat-bootstrap strict-CI criterion holds for $11/12$ source pairs, all $12/12$ pair medians are positive, and trial-level $316/360$ is reported only as within-cluster consistency. DPO concentration checks on Qwen-3-8B and Qwen-2.5-1.5B are descriptive (App.~\ref{app:dpo_details}). A frozen matched-rollout reward surrogate on Qwen-2.5-0.5B-Instruct tests whether the closure identity holds under analytic rewards through $T=16$ sequential update blocks. Detailed DPO/GRPO controls, random, wrong-sign ($+\eta^2 b_{AB}$) and first-order (norm-matched $\eta(g_A-g_B)$) comparisons, and the AdamW endpoint check are in Appendices~\ref{app:dpo_extended}, \ref{app:grpo}, and~\ref{app:adamw}.

\textbf{Persistence under continued training.} In a new-seed Qwen-3-4B study ($20$ runs over ten source pairs), the $AB$ and $BA$ endpoints each take sixteen more SGD updates on the same third-domain batches. After four, eight, and sixteen updates, $15$, $14$, and $13$ of $20$ runs retain the aligned trace, a positive projection of $\theta_{AB}-\theta_{BA}$ on the original commutator direction; some runs reverse sign, and over the $18$ runs that start aligned the median retained fraction falls from $27\%$ to $14\%$ to $7\%$. The aligned trace persists briefly and then fades; this does not establish behavioral retrieval or long-term memory (App.~\ref{app:continuation}).

\section{Discussion}
\label{sec:discussion}

\paragraph{Path dependence.} $b_{AB}$ is a local correction for a realized update path. Matched-batch SFT/DPO and matched-rollout GRPO closures succeed because the bracket comes from the same updates that produced the two endpoints; calibration on unmatched rollouts is the negative control (App.~\ref{app:grpo}). A residue concentrated on few tokens and rows is consistent with aggregate benchmarks missing it. With $\theta_{BA}$ as target, iterative projection closes about $90\%$ of the gap on the $5/6$ SFT pairs above floor, removing the antisymmetric residue after training with the observed $\theta_{BA}$ as target.

\paragraph{Why this matters.} The practical target is controlled measurement of sequential post-training interactions. Instead of discovering source interactions only after expensive runs as scalar regressions, a pipeline could estimate which evaluation slices and output tokens carry the antisymmetric residue, then test reordering, token-local reweighting, or matched-path bracket correction in a calibrated BCH-local window. In the current evidence these are measurements, not deployment tools: the bracket remains tied to realized batches, rollouts, optimizer state, and the calibrated step-size window.

\paragraph{Why bracket-driven updates are more than order prediction.} At the drift-matched reference $\theta_{\rm ref}=\theta_0-\eta(g_A+g_B)$, define $c_{AB}:=\tfrac12(H_B g_A + H_A g_B)$, $\mu:=\langle g_E(\theta_{\rm ref}), c_{AB}\rangle$, and $\sigma:=\langle g_E(\theta_{\rm ref}), b_{AB}\rangle$. For the bracket-only step $\theta_{\rm BO} := \theta_{\rm ref} - \tfrac12\eta^2 \cdot \mathrm{sign}(\hat\sigma)\cdot b_{AB}$ ($\hat\sigma$ an estimate of $\sigma$), Proposition~\ref{prop:bo} gives
\begin{equation}
\mathcal{L}_E(\theta_{\rm BO}) - \min\{\mathcal{L}_E(\theta_{AB}), \mathcal{L}_E(\theta_{BA})\} \;=\; -\eta^2 \mu + O(\eta^3).
\label{eq:bo_beats}
\end{equation}
When $\mu > 0$ (the symmetric drift on $E$ is target-disadvantageous) and the sign of $\sigma$ is estimated correctly, this update beats the better pure order by $\eta^2\mu$ to leading order in $\eta$ (App.~\ref{app:bo_proof}). Thus the correction experiments are not redundant with order prediction, which forgoes the $\eta^2\mu$ term. The companion closure identity $\theta_{AB}-\eta^2 b_{AB}=\theta_{BA}+O(\eta^3)$ motivates the matched-batch DPO edit $\theta_{AB}-\eta^2 b_{AB}$; GRPO-style results are closure checks on matched rollouts (App.~\ref{app:grpo}).

\section{Limitations}
\label{sec:limitations}
The framework measures the antisymmetric, path-local component of source interaction. The SFT intervention metric is order-gap NLL closure, not absolute loss improvement or discrete capability preservation; at $k{=}1$, we make no claim about GSM8K exact match or HumanEval pass@1 (App.~\ref{app:iterative_extended}). The reported SFT protocol uses SGD with $k=1$--$8$ and two sources per path; DPO correction requires fp32 and is matched-batch; iterative recovery requires $\theta_{BA}$. GRPO-style evidence comes from a frozen matched-rollout KL surrogate on a 0.5B model with analytic rewards (Appendix~\ref{app:grpo}). AdamW evidence is limited. The unmodified SGD bracket is at most weakly predictive of AdamW order ($57.1\%$ at $k{=}20$; \citealp[App.~D.7]{geometry_sequential_learning}), so we use the commutator on the optimizer state instead: we find token concentration of the gradient-space readout on Qwen-3-4B (Gini $0.980$) and endpoint closure in one subspace of Qwen-2.5-1.5B fp32 (parameter cosine $0.998$, NLL gap closure $79.5\%$ $[72.6\%, 85.9\%]$; App.~\ref{app:adamw}). DPO bracket-direct edits and GRPO-style closure checks in the main text use SGD, where the BCH identity is exact at $O(\eta^2)$. ``Memory'' here is parametric and pair-conditioned: the two candidate sources are supplied, the strongest test compares both orderings, and the trace attenuates under continued training (App.~\ref{app:continuation}); we claim neither item-addressable recall nor broad capability gains. App.~\ref{app:cost} gives the per-pair cost, exact reuse across pairs, and cheaper curvature approximations.

\section{Conclusion}
\label{sec:conclusion}
Commutator memory turns the noncommutativity of two training updates into a measurable object with four uses of one bracket: a scalar order forecast, a sparse token report, a path-local parameter-space direction, and a paired-endpoint statistic that tells which order produced each $k{=}1$ SGD endpoint, correct $91.7\%$ of the time across four LLMs (Corollary~\ref{cor:delta_s}).

\begin{ack}
No funding or competing interests to declare.
\end{ack}

\bibliographystyle{plainnat}
\bibliography{references}

\appendix
\makeatletter
\renewcommand{\thesection}{\ifnum\c@section>26 A\@Alph{\numexpr\c@section-26\relax}\else\@Alph\c@section\fi}
\renewcommand{\theHsection}{appendix.\arabic{section}}
\makeatother
\clearpage
\section*{Broader impact and release statement}
\label{app:impact_release}
The framework reveals where in vocabulary space the noncommutativity of training updates concentrates and could support more interpretable diagnosis of capability shifts in sequential post-training pipelines. Corollary~\ref{cor:delta_s} (App.~\ref{app:bracket_derivation}) establishes that, given a public base $\theta_0$, candidate training-domain pair $(A, B)$, and access to both alternate-order endpoints $\theta_{AB}, \theta_{BA}$, the antisymmetric statistic $\Delta s := \langle \theta_{AB} - \theta_{BA},\, b_{AB}\rangle$ identifies which order produced each $k=1$ SGD endpoint via $\mathrm{sign}(\Delta s)$ whenever its leading $\eta^2\|b_{AB}\|^2$ term dominates the $O(\eta^3)$ remainder (eq.~\ref{eq:delta_s}; the paired difference cancels the shared drift). Across $72$ pair-seed units on four models (Qwen-3-4B, Qwen-2.5-1.5B, Llama-3.2-1B, Llama-3.1-8B; two architectures, four scales), combined sign accuracy is $66/72 = 91.7\%$ (Wilson $95\%$ CI $[83\%, 96\%]$); both $4$B and $8$B units reach $100\%$ (\S\ref{sec:forensic}, Appendix~\ref{app:forensic}). The main paired-endpoint accuracy is for $k=1$ SGD; AdamW evidence is in App.~\ref{app:adamw}, and we do not release endpoint-assignment tooling. We have not evaluated bypassing safety training, private-data recovery, or selective capability editing; the released material is limited to order-gap measurement.

The supplementary archive includes source code, fixed protocol definitions, result files for reported claims, and analysis scripts. Dataset versions, model checkpoints, hyperparameters, seed schedules, learning-rate calibrations, and compute budgets are documented in Appendix~\ref{app:reproducibility}. Token-level interventions (Appendix~\ref{app:intervention_types}), the DPO correction (Appendix~\ref{app:dpo_extended}), and the GRPO-style test (Appendix~\ref{app:grpo}) follow these fixed protocol definitions.

\section{Bracket derivation and optimizer connection}
\label{app:bracket_derivation}

\paragraph{Optimizer connection and training anholonomy.} It is useful to separate the \emph{training schedule} from the \emph{model state}. Let $\Theta$ be parameter space and let the exposure base space be $\mathcal{B}=\R_{\geq 0}^N$, with coordinate $s_i$ measuring cumulative exposure to source $i$ in learning-rate units. The total space is the trivial bundle $\mathcal{B}\times\Theta\to\mathcal{B}$. Each source defines an optimizer vector field on the fiber; for SGD, $f_i(\theta) = -\nabla_\theta \mathcal{L}_i(\theta) = -g_i(\theta)$. We define the horizontal lift of the exposure direction $\partial_{s_i}$ by $\widetilde X_i = \partial_{s_i} + f_i(\theta)$. The resulting optimizer connection has vertical-valued curvature $\Omega_{ij} := \operatorname{ver}[\widetilde X_i,\widetilde X_j]$. Because the exposure coordinates commute and the losses are stationary in $s$, this reduces to the optimizer-vector-field bracket $\Omega_{ij} = [f_i,f_j]$. For SGD this curvature component is exactly $\Omega_{AB} = Df_B f_A - Df_A f_B = H_B g_A - H_A g_B$.

The two schedules $A\!\to\!B$ and $B\!\to\!A$ are paths in exposure space with the same endpoints $(0,0)$ and $(\eta,\eta)$, but their horizontal lifts end at different model parameters. The endpoint difference is the small-rectangle anholonomy of this optimizer connection.

\paragraph{Self-contained derivation of eq.~\eqref{eq:endpoint_diff}.} Let $F_A(\theta)=\theta+\eta f_A(\theta)$ and $F_B(\theta)=\theta+\eta f_B(\theta)$ denote one Euler optimizer step, with $f_i=-g_i$. Taylor expansion gives
\begin{align}
F_B(F_A(\theta_0)) &= \theta_0 + \eta f_A + \eta f_B + \eta^2 Df_B f_A + O(\eta^3),\\
F_A(F_B(\theta_0)) &= \theta_0 + \eta f_B + \eta f_A + \eta^2 Df_A f_B + O(\eta^3),
\end{align}
where all vector fields and derivatives are evaluated at $\theta_0$. Subtracting cancels the first-order drift and leaves
\[
\theta_{AB}-\theta_{BA} = \eta^2(Df_Bf_A-Df_Af_B)+O(\eta^3) = \eta^2(H_Bg_A-H_Ag_B)+O(\eta^3).
\]
This is Lemma 2.1 of \citet{geometry_sequential_learning}. The sign convention differs from continuous-time exponential-flow commutators; for the discrete update $\theta \mapsto \theta - \eta g$, the leading-order AB-vs-BA defect is $\eta^2 b_{AB}$ with no $1/2$ factor.

\begin{lemma}[Symmetric/antisymmetric Trotter expansion]
\label{lem:sym_antisym}
Let $\theta_{\rm ref} := \theta_0 - \eta(g_A + g_B)$, $c_{AB} := \tfrac12(H_B g_A + H_A g_B)$, and $b_{AB} := H_B g_A - H_A g_B$, all evaluated at $\theta_0$. Then
\[
\theta_{AB} = \theta_{\rm ref} + \eta^2 \bigl(c_{AB} + \tfrac12 b_{AB}\bigr) + O(\eta^3), \quad
\theta_{BA} = \theta_{\rm ref} + \eta^2 \bigl(c_{AB} - \tfrac12 b_{AB}\bigr) + O(\eta^3).
\]
The antisymmetric component $\eta^2 b_{AB}$ is the AB-vs-BA defect (Lemma~2.1 of \citet{geometry_sequential_learning}); the symmetric component $\eta^2 c_{AB}$ is shared by both orders.
\end{lemma}

\begin{proof}
From the two expansions above, $\theta_{AB} = \theta_0 + \eta(f_A + f_B) + \eta^2 Df_B f_A + O(\eta^3)$ with $f_i = -g_i$, so $\theta_{AB} - \theta_{\rm ref} = \eta^2 H_B g_A + O(\eta^3) = \eta^2(c_{AB} + \tfrac12 b_{AB}) + O(\eta^3)$. Symmetrically for $\theta_{BA}$.
\end{proof}

This expansion is used in the proof of Proposition~\ref{prop:bo} and Theorem~\ref{thm:forensic_id}.

\paragraph{Trotter-reference Taylor expansion (eq.~\eqref{eq:target_score}).} Let $\delta\theta_{\mathrm{exact}}:=\theta_{AB}-\theta_{BA}$. Taylor expansion of $\mathcal{L}_E$ around $\theta_{\text{ref}}$ gives $\Delta_E(A,B) = \langle g_E(\theta_{\text{ref}}), \delta\theta_{\mathrm{exact}}\rangle + O(\eta^4)$. Substituting $\delta\theta_{\mathrm{exact}}=\eta^2 b_{AB}+O(\eta^3)$ yields eq.~\eqref{eq:target_score}; $O(\eta^4)$ is the loss-linearization error for the actual endpoint displacement, while the bracket-only predictor has the expected $O(\eta^3)$ Euler truncation.

\paragraph{Proposition: when the bracket-only update beats both pure orders.}
\label{app:bo_proof}
We give a proof of Equation~\ref{eq:bo_beats} and the associated proposition stated in Section~\ref{sec:discussion}. The result is a direct corollary of Lemma~\ref{lem:sym_antisym}. \citet[App.~E.8]{geometry_sequential_learning} evaluated this bracket-only step empirically. The proposition below gives its conditional leading-order guarantee; we state it because it clarifies the symmetric/antisymmetric decomposition used by the matched-batch DPO closure (Appendix~\ref{app:dpo_extended}) and the GRPO-style closure check (Appendix~\ref{app:grpo}).

\begin{proposition}[Bracket-only update can beat both pure orders]
\label{prop:bo}
Let $\theta_{\rm ref} := \theta_0 - \eta(g_A + g_B)$ be the drift-matched first-order reference. Let
\[
c_{AB} := \tfrac12 (H_B g_A + H_A g_B), \qquad b_{AB} := H_B g_A - H_A g_B
\]
be the symmetric and antisymmetric components of the second-order endpoint correction at $\theta_0$, with $\mu := \langle g_E(\theta_{\rm ref}), c_{AB}\rangle$, $\sigma := \langle g_E(\theta_{\rm ref}), b_{AB}\rangle$. Define the bracket-only update
\[
\theta_{\rm BO} := \theta_{\rm ref} - \tfrac12 \eta^2 \cdot \mathrm{sign}(\hat\sigma) \cdot b_{AB},
\]
where $\hat\sigma$ is an estimate of the sign of $\sigma$ (Section~\ref{sec:baselines} reports $\sigma$-correctness on the SGD predictor). Under the smoothness assumptions of \citet{geometry_sequential_learning} Lemma~2.1 (which gives an $O(\eta^3)$ parameter remainder for $\theta_{AB}, \theta_{BA}$ relative to their second-order Taylor truncations) and a loss-linearization Taylor remainder bounded at $O(\eta^4)$,
\begin{equation}
\mathcal{L}_E(\theta_{\rm BO}) - \min\{\mathcal{L}_E(\theta_{AB}), \mathcal{L}_E(\theta_{BA})\} = -\eta^2 \mu + O(\eta^3),
\label{eq:bo_minus_best}
\end{equation}
on the support where $\sigma \neq 0$ and $\mathrm{sign}(\hat\sigma)$ is correct. In particular, when $\mu > 0$ and $\eta$ is small enough that the leading $\eta^2\mu$ term dominates, $\theta_{\rm BO}$ has strictly lower target loss than the better pure order.
\end{proposition}

\begin{proof}
By Lemma~\ref{lem:sym_antisym}, the actual two-step SGD endpoints satisfy
\[
\theta_{AB} - \theta_{\rm ref} = \eta^2(c_{AB} + \tfrac12 b_{AB}) + r_{AB}, \qquad
\theta_{BA} - \theta_{\rm ref} = \eta^2(c_{AB} - \tfrac12 b_{AB}) + r_{BA},
\]
with $\|r_{AB}\|,\|r_{BA}\|=O(\eta^3)$ uniformly in the local smoothness neighborhood.
The bracket-only construction $\theta_{\rm BO} - \theta_{\rm ref} = -\tfrac12 \eta^2 \cdot \mathrm{sign}(\hat\sigma) \cdot b_{AB}$ has \emph{no} $\eta^3$ parameter remainder: it is defined exactly by an arithmetic combination of $g_A, g_B$, and their HVPs at $\theta_0$.
Let $s:=\mathrm{sign}(\sigma)$; on the correct-sign support, $\mathrm{sign}(\hat\sigma)=s$.

Taylor expansion of $\mathcal{L}_E$ around $\theta_{\rm ref}$ gives, for any displacement $\Delta\theta$ with $\|\Delta\theta\| = O(\eta^2)$,
\[
\mathcal{L}_E(\theta_{\rm ref} + \Delta\theta) = \mathcal{L}_E(\theta_{\rm ref}) + \langle g_E(\theta_{\rm ref}), \Delta\theta \rangle + \tfrac12 \langle \Delta\theta, H_E \Delta\theta\rangle + O(\|\Delta\theta\|^3),
\]
where the quadratic term is $O(\eta^4)$ and the cubic remainder is $O(\eta^6)$. Substituting the pure-order displacements and absorbing the linear contribution of $r_{AB},r_{BA}$ into $O(\eta^3)$ gives
\[
\mathcal{L}_E(\theta_{AB}) = \mathcal{L}_E(\theta_{\rm ref}) + \eta^2(\mu + \tfrac12 \sigma) + O(\eta^3),
\]
\[
\mathcal{L}_E(\theta_{BA}) = \mathcal{L}_E(\theta_{\rm ref}) + \eta^2(\mu - \tfrac12 \sigma) + O(\eta^3).
\]
Therefore
\[
\min\{\mathcal{L}_E(\theta_{AB}),\mathcal{L}_E(\theta_{BA})\}
= \mathcal{L}_E(\theta_{\rm ref}) + \eta^2\mu - \tfrac12\eta^2|\sigma| + O(\eta^3).
\]
This min expansion does not require a separation condition on $|\sigma|$: the map $(x,y)\mapsto \min\{x,y\}$ is 1-Lipschitz, so the $O(\eta^3)$ endpoint remainders perturb the leading-order minimum by at most $O(\eta^3)$. A nonzero $\sigma$ is only needed to define the sign-selected BO direction; in finite-sample experiments, a margin $|\sigma|$ above the estimator noise floor is the practical sign-stability condition.
For the bracket-only point,
\[
\mathcal{L}_E(\theta_{\rm BO}) = \mathcal{L}_E(\theta_{\rm ref}) - \tfrac12 \eta^2 s\sigma + O(\eta^4)
= \mathcal{L}_E(\theta_{\rm ref}) - \tfrac12 \eta^2 |\sigma| + O(\eta^4),
\]
because $\mathrm{sign}(\hat\sigma)$ is correct. Subtracting gives Equation~\ref{eq:bo_minus_best}. The $O(\eta^3)$ residual is the Euler endpoint truncation that the bracket-only construction does not encode; it would improve to $O(\eta^4)$ only against the truncated second-order pure endpoints (or for exactly quadratic losses). When $\mu > 0$ and $\eta$ is small enough that the leading $-\eta^2\mu$ term governs the sign of the difference, $\theta_{\rm BO}$ improves on the better pure order by $\eta^2\mu$ at leading order, and therefore improves on both pure orders. The empirical regression slope of $1.15$ and 12/13 successful trials on the $\mu > 0$ subset (Section~\ref{sec:controlled_tests}) at $\eta = 1.91 \cdot 10^{-3}$ indicate that this regime is realized in practice.
\end{proof}

\paragraph{Connection to the matched-batch endpoint edit.} The DPO closure (Appendix~\ref{app:dpo_extended}) and SFT iterative correction (Appendix~\ref{app:iterative_extended}) both apply the operation $\theta_{AB} \mapsto \theta_{AB} - \eta^2 b_{AB}$ rather than the BO construction at $\theta_{\rm ref}$. Direct expansion gives $\theta_{AB} - \eta^2 b_{AB} = \theta_{\rm ref} + \eta^2(c_{AB} - \tfrac12 b_{AB}) + O(\eta^3) = \theta_{BA} + O(\eta^3)$, recovering the reversed-order endpoint at $O(\eta^3)$. Proposition~\ref{prop:bo} is the complementary statement: at the drift-matched reference rather than at the trained $\theta_{AB}$, the sign-selected bracket-only step removes the symmetric target-loss drift and beats the better pure order by $\eta^2\mu$ at leading order. Both constructions exploit the same antisymmetric/symmetric decomposition; the matched-batch DPO edit $\theta_{AB}-\lambda\eta^2 b_{AB}$ at $\lambda{=}1$ is the closure-targeting case ($\lambda{=}1$ recovers $\theta_{BA}$), while for $0<\lambda<1$ the edit lies, to leading order, between the two endpoints and keeps the symmetric $c_{AB}$ term that the BO step omits. The non-monotone $\lambda$-sweep peaking at $\lambda{=}1$ (App.~\ref{app:r79_lambda}) is therefore a signature of the second-order prescribed step rather than ``any small move along $b_{AB}$ helps.''

\begin{theorem}[The order of two updates is identifiable from the endpoint]
\label{thm:forensic_id}
Let $\theta_0 \in \R^d$ be a public base model. Let $g_A := \nabla \mathcal{L}_A(\theta_0)$, $g_B := \nabla \mathcal{L}_B(\theta_0)$ be one-step SGD gradients on candidate training-domain pair $(A, B)$ at $\theta_0$, and let $H_A := \nabla^2 \mathcal{L}_A(\theta_0)$, $H_B := \nabla^2 \mathcal{L}_B(\theta_0)$ be the corresponding Hessians. Let $\theta_{\mathrm{query}} \in \{\theta_{AB}, \theta_{BA}\}$ result from $k=1$ SGD on either ordering with step size $\eta$. Define
\[
b_{AB} := H_B g_A - H_A g_B, \qquad c_{AB} := \tfrac12 (H_B g_A + H_A g_B), \qquad \theta_{\mathrm{ref}} := \theta_0 - \eta(g_A + g_B),
\]
\[
s(\theta) := \langle \theta - \theta_{\mathrm{ref}},\, b_{AB} \rangle, \qquad \mathrm{SCR} := \frac{\tfrac12 \|b_{AB}\|^2}{|\langle c_{AB},\, b_{AB}\rangle|},
\]
where SCR is the ratio of the antisymmetric term to the symmetric one, with $\mathrm{SCR}:=\infty$ when the denominator is zero and $b_{AB}\neq0$. Under the smoothness conditions of Section~\ref{sec:background} and the inequality $\mathrm{SCR} > 1$ (equivalently $\|b_{AB}\|^2 > 2|\langle c_{AB}, b_{AB}\rangle|$),
\[
\mathrm{sign}\bigl(s(\theta_{AB})\bigr) = +1 \qquad \text{and} \qquad \mathrm{sign}\bigl(s(\theta_{BA})\bigr) = -1
\]
to leading order as $\eta\to0$. At finite $\eta$, the conclusion requires the displayed $O(\eta^2)$ margins to dominate the $O(\eta^3 \|b_{AB}\|)$ remainders and estimator error.
\end{theorem}

\begin{proof}
By the second-order endpoint expansion derived above (eq.~\eqref{eq:endpoint_diff} and the symmetric companion),
\[
\theta_{AB} = \theta_{\mathrm{ref}} + \eta^2 (c_{AB} + \tfrac12 b_{AB}) + O(\eta^3),
\qquad
\theta_{BA} = \theta_{\mathrm{ref}} + \eta^2 (c_{AB} - \tfrac12 b_{AB}) + O(\eta^3).
\]
Substituting into $s$:
\begin{align*}
s(\theta_{AB}) &= \eta^2 \langle c_{AB} + \tfrac12 b_{AB},\, b_{AB} \rangle + O(\eta^3 \|b_{AB}\|)\\
&= \eta^2 \left( \langle c_{AB},\, b_{AB}\rangle + \tfrac12 \|b_{AB}\|^2 \right) + O(\eta^3 \|b_{AB}\|),\\
s(\theta_{BA}) &= \eta^2 \left( \langle c_{AB},\, b_{AB}\rangle - \tfrac12 \|b_{AB}\|^2 \right) + O(\eta^3 \|b_{AB}\|).
\end{align*}
For $\mathrm{sign}(s(\theta_{AB})) = +1$ at leading order we need $\langle c_{AB}, b_{AB}\rangle + \tfrac12 \|b_{AB}\|^2 > 0$. If $\langle c_{AB}, b_{AB}\rangle \geq 0$, this is trivially satisfied since $\tfrac12\|b_{AB}\|^2 \geq 0$. If $\langle c_{AB}, b_{AB}\rangle < 0$, the condition reduces to $\tfrac12 \|b_{AB}\|^2 > |\langle c_{AB}, b_{AB}\rangle|$, which is exactly $\mathrm{SCR} > 1$. The argument for $\mathrm{sign}(s(\theta_{BA})) = -1$ is symmetric.
\end{proof}

\begin{corollary}[Paired-endpoint identifiability with drift cancellation]
\label{cor:delta_s}
Under the smoothness hypotheses of Theorem~\ref{thm:forensic_id} and given an unordered pair of alternate-order endpoints $\{\theta^{(1)},\theta^{(2)}\}=\{\theta_{AB},\theta_{BA}\}$, define the antisymmetric statistic
\begin{equation}
\Delta s_{12} := \langle \theta^{(1)} - \theta^{(2)},\, b_{AB} \rangle.
\label{eq:delta_s}
\end{equation}
If $\theta^{(1)}=\theta_{AB}$ and $\theta^{(2)}=\theta_{BA}$, then $\Delta s_{12} = \eta^2 \|b_{AB}\|^2 + O(\eta^3 \|b_{AB}\|)$; if the presentation is swapped, the sign is reversed. Thus the sign assigns the two endpoints whenever the leading term dominates the finite-$\eta$ and estimator remainders, with no $\mathrm{SCR}$ condition required.
\end{corollary}

\begin{proof}
By Lemma~\ref{lem:sym_antisym}, $\theta_{AB}-\theta_{BA} = \eta^2 b_{AB} + O(\eta^3)$. Therefore $\langle\theta_{AB}-\theta_{BA},b_{AB}\rangle = \eta^2\|b_{AB}\|^2 + O(\eta^3\|b_{AB}\|)$. Swapping the two endpoints multiplies the statistic by $-1$. The leading term is strictly positive when $b_{AB}\neq 0$.
\end{proof}

\paragraph{When the per-query test fails.} The per-query sign condition in Theorem~\ref{thm:forensic_id} suffers from a finite-$\eta$ failure mode: when the symmetric drift $\eta^2 \langle c_{AB}, b_{AB}\rangle$ approaches $\eta^2 \cdot \tfrac12 \|b_{AB}\|^2$ in absolute value (i.e.\ $\mathrm{SCR}$ near $1$), both $s(\theta_{AB})$ and $s(\theta_{BA})$ can lie on the same side of zero, breaking the per-query sign test. Corollary~\ref{cor:delta_s} removes this leading symmetric contribution in the paired subtraction, leaving $\eta^2\|b_{AB}\|^2$ as the leading signal. Equivalently $\Delta s_{12}$ projects the antisymmetric endpoint difference onto $b_{AB}$ itself. We adopt $\mathrm{sign}(\Delta s_{12})$ as the practical paired-endpoint assignment statistic.

\paragraph{Reference-point invariance of $\Delta s$.} The Trotter reference $\theta_{\mathrm{ref}} = \theta_0 - \eta(g_A + g_B)$ enters the per-query score $s(\theta) = \langle \theta - \theta_{\mathrm{ref}}, b_{AB}\rangle$, but cancels in the paired subtraction: substituting $\theta_0$ for $\theta_{\mathrm{ref}}$ in the per-query score subtracts the same $\eta\langle g_A+g_B, b_{AB}\rangle$ from both $s(\theta_{AB})$ and $s(\theta_{BA})$. The implementation may equivalently use $\theta_0$ or $\theta_{\mathrm{ref}}$ for the paired statistic. This contrasts with the scalar $\sigma$ predictor of \citet{geometry_sequential_learning}, a single inner product $\langle g_E, b_{AB}\rangle$ with no endpoint subtraction; evaluating its target gradient at the drift-matched $\theta_{\mathrm{ref}}$ rather than at $\theta_0$ reduces the drift error from $O(\eta^3)$ to $O(\eta^4)$ (their Remark~2.6).

\paragraph{Empirical regime.} At finite $\eta$ the $O(\eta^3 \|b_{AB}\|)$ remainder, sampling error, HVP error, and finite-difference error are non-negligible, so the per-query $\mathrm{SCR} > 1$ test is leading-order necessary but not sufficient on the 1B, 1.5B and 8B runs, where $\|b_{AB}\|^2$ is small relative to noise. The paired $\Delta s$ estimator identifies the order across model scales: combined accuracy $66/72 = 91.7\%$ (Wilson $95\%$ CI $[83\%, 96\%]$) over $72$ pair-seed units across four models (Qwen-3-4B, Qwen-2.5-1.5B, Llama-3.2-1B, Llama-3.1-8B) and two architectures, with per-model rates $77.8\%$--$100\%$ (Appendix~\ref{app:forensic}). Both $4$B and $8$B units reach $100\%$; $1$B-class accuracy is $77.8\%$--$88.9\%$. Per-query $\mathrm{SCR}$ statistics: $72\%$--$89\%$ of units satisfy $\mathrm{SCR} > 1$ across runs. Replication across precision regimes (bf16 storage and full fp32) on the $1$B and $1.5$B models gives identical per-model sign-recovery rates (Table~\ref{tab:forensic_cross_model}, App.~\ref{app:forensic}); the main count uses each model once.

\paragraph{Practical relevance and prior work.} Theorem~\ref{thm:forensic_id} requires only the public base model $\theta_0$, candidate training-domain pair $(A, B)$ and corresponding gradient/HVP access at $\theta_0$, and the trained query $\theta_{\mathrm{query}}$ (the paired statistic of Corollary~\ref{cor:delta_s} needs both endpoints); no training-time signal is injected. This differs from adjacent settings: watermarking~\citep{kirchenbauer2023watermark, zhao2024provable} embeds a signal during generation, membership-inference attacks~\citep{shokri2017membership, carlini2022membership} target individual training examples, and influence-function methods~\citep{koh2017understanding, grosse2023studying, park2023trak} estimate per-sample influence on predictions.

\section{$\tau$ implementation: native-logit convention and finite differences}
\label{app:tau_implementation}

\paragraph{Native-logit convention.} The aggregate $\sum_k\tau_k$ is invariant to position-wise constant shifts of $\delta z$, because $\sum_k e_k=0$. Individual token attributions are reported in the model's native logit coordinates. Thus $\tau_k$ depends on the model's logit parameterization; only the sum is shift-invariant.

\paragraph{Finite-difference $\delta z$.} The analytic identity in Eq.~\eqref{eq:tau_exact} is exact for the true JVP $J_\theta z\,v$. In experiments we compute an operational finite-difference readout
\[
\widetilde{\delta z}_\epsilon(v)=\frac{z(\theta_{\text{ref}}+\epsilon v)-z(\theta_{\text{ref}}-\epsilon v)}{2\epsilon},
\]
with $\epsilon=1.0$ for SFT and $\epsilon=0.1$ in the corrected DPO identity checks. This is applied \emph{after} computing $b_{AB}$ via exact HVPs; only the logit-space JVP is finite-differenced. Consequently $\sum_k\widetilde{\tau}_{k,\epsilon}$ approximates the analytic JVP $\mathbb{E}[e(\theta_{\rm ref})^\top J_\theta z\,v]$ with $O(\epsilon^2)$ logit-truncation error and converges to the analytic aggregate $\eta^2\sigma$ as $\epsilon\to0$ for smooth logits. Individual token entries can vary with $\epsilon$ because finite differences integrate local curvature along the supplied displacement.

\paragraph{Choice of displacement $v$.} The decomposition over $k$ is exact for the displacement supplied to the linear readout. Using $v=\eta^2b_{AB}$ gives the leading bracket readout used throughout; using $v=\theta_{AB}-\theta_{BA}$ would give an endpoint-displacement readout whose aggregate matches $\Delta_E(A,B)$ up to the loss-linearization remainder.

\section{Precision policy for HVP computation}
\label{app:precision}

We distinguish model storage dtype, HVP compute dtype, the dtype in which the finite-difference perturbations $\theta_{\rm ref}\pm\epsilon v$ are formed, and loss/readout dtype. Bracket computations (double-backward HVPs for $H_B g_A - H_A g_B$) use float32. For top-token rank and support claims on bf16-stored models, the perturbed parameters are formed in fp32 during the readout check. The choice of end-to-end loss/intervention dtype is determined by whether the measured ordering gap is safely above the numerical resolution of the loss evaluation:

\begin{itemize}
    \item \textbf{Qwen-3-4B} has an ordering gap of $\sim\!7 \times 10^{-3}$ in the reported intervention setting. We run the intervention with bf16 model storage and use fp32 HVPs and fp32 perturbed parameters for the bracket readouts.
    \item \textbf{Qwen-2.5-1.5B} has a much smaller fp32 gap of $\sim\!10^{-4}$. We therefore run the intervention arm in full fp32 and report only the predefined above-floor subset (28/90 trials) in Table~\ref{tab:interventions}.
\end{itemize}

\section{Matched-batch DPO correction: per-pair table and figure (Qwen-3-4B fp32)}
\label{app:r79_perpair}

\begin{table}[h]
\centering
\small
\caption{Matched-batch bracket-correction closure on the six original DPO source pairs (Qwen-3-4B fp32, $\beta_{\text{DPO}}{=}0.1$, $\eta{=}2\text{e-}3$, 3 seeds $\times$ 10 trials each). Closure is heavy-tailed near the fp32 noise floor; per-pair seed-cluster CIs are descriptive diagnostics of within-pair variability with $n_{\rm clusters}=3$, not formal 95\% inference. Columns: mean and median are the winsorized $5/95$ mean and the median closure of the $\lambda{=}1$ edit; flat CI and cluster CI are per-trial and seed-clustered bootstrap $95\%$ intervals of the mean; wrong-sign mean and random mean are the mean closures of the $+\eta^2 b_{AB}$ edit and of a norm-matched random edit. Pair abbreviations: fq=\texttt{false\_qa}, flan=\texttt{flan\_v2\_p3}, share=\texttt{sharegpt}, ultra=\texttt{ultrachat}.}
\label{tab:r79_bracket_closure}
\begin{tabular*}{\textwidth}{@{\extracolsep{\fill}}lcccccc}
\toprule
Pair & mean & flat CI & cluster CI & median & wrong-sign mean & random mean \\
\midrule
fq:flan     & 0.675 & [0.561,\,0.767] & \textbf{[0.425,\,0.772]} & 0.796 & $-1.083$ & $+0.001$ \\
fq:share    & 0.454 & [0.157,\,0.703] & \textbf{[0.126,\,0.536]} & 0.739 & $-1.104$ & $+0.001$ \\
fq:ultra    & 0.647 & [0.477,\,0.798] & $[-0.386,\,0.669]$ & 0.778 & $-1.874$ & $-0.001$ \\
flan:share  & 0.545 & [0.247,\,0.776] & $[-2.306,\,0.825]$ & 0.807 & $-1.915$ & $-0.000$ \\
flan:ultra  & 0.720 & [0.645,\,0.788] & \textbf{[0.455,\,0.754]} & 0.785 & $-1.174$ & $+0.001$ \\
ultra:share & 0.454 & [0.219,\,0.665] & $[-0.107,\,0.604]$ & 0.735 & $-1.928$ & $+0.003$ \\
\midrule
\multicolumn{7}{p{0.95\textwidth}}{\textbf{Combined flat (180 trials):} winsor mean $\mathbf{=0.588}$ [CI $0.507,\,0.660$], median $=0.770$, raw $=0.318$.} \\
\multicolumn{7}{p{0.95\textwidth}}{\textbf{Seed-cluster bootstrap (pairs fixed):} winsor mean $\mathbf{=0.583}$ [CI $0.469,\,0.676$]. Seed-cluster strict-CI: \textbf{3/6}.} \\
\bottomrule
\end{tabular*}
\end{table}

\begin{figure}[h]
\centering
\includegraphics[width=\linewidth]{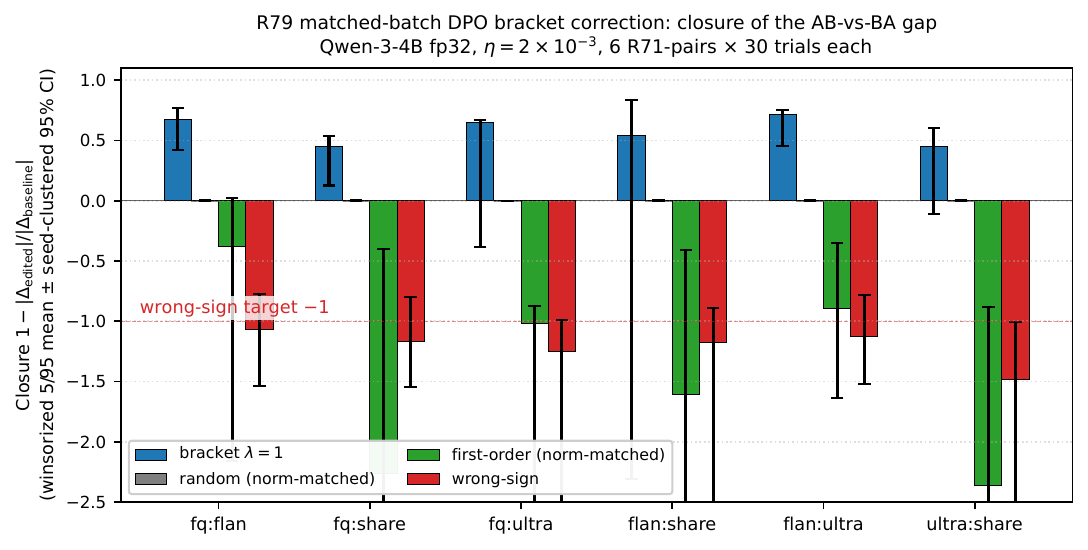}
\caption{Matched-batch DPO bracket-correction closure, per original pair. Bars are winsorized (5/95) means with seed-clustered bootstrap intervals (3 seed clusters per pair) shown as descriptive diagnostics of within-pair variability, not formal 95\% inference. Random direction near zero; wrong-sign edit at or below the Lemma~2.1 prediction of $-1$ (red dashed line); first-order norm-matched control underperforms the bracket.}
\label{fig:dpo_closure}
\end{figure}

\section{Matched-batch DPO correction: $\lambda$-sweep and control-direction signs}
\label{app:r79_lambda}

The matched-batch correction protocol prescribes $\lambda = 1$ in the bracket subtraction $\theta_{AB} - \lambda \cdot \eta^2 b_{AB}$ (Lemma 2.1 prediction; not tuned). We additionally report diagnostic $\lambda \in \{0.5, 2.0\}$ to test that the closure curve in $\lambda$ is non-monotone with a peak at $\lambda = 1$---the signature of a second-order predicted step rather than a generic ``any small move along $b_{AB}$ helps'' regime.

Across the six original pairs, per-pair median closures track the linearized prediction $C(\lambda) = 1 - |1 - \lambda|$ from $\theta_{AB} - \lambda\eta^2 b_{AB} = \theta_{BA} + (1-\lambda)\eta^2 b_{AB} + O(\eta^3)$: $\lambda = 0.5$ closes about half the gap ($6/6$ pair medians positive, range $[+0.45,+0.55]$); $\lambda = 1$ peaks at $0.78$ (median of pair medians); $\lambda = 2$ falls back to $-0.01$ on average. The symmetric peak at $\lambda{=}1$ rules out a generic ``any move along $b_{AB}$ helps'' alternative.

The control-direction signs in Table~\ref{tab:r79_bracket_closure} are accordingly:
\begin{itemize}
\item \textbf{Random direction (norm-matched)}: closure $\approx 0$ on every pair (within $10^{-3}$). Random Gaussian directions of the same Frobenius norm as $\eta^2 b_{AB}$ produce no systematic gap closure.
\item \textbf{Wrong-sign edit ($\theta_{AB} + \eta^2 b_{AB}$)}: winsorized mean closure $-1.21$ across pairs, individual-pair medians in $[-1.24, -0.89]$. The Lemma 2.1 prediction is exactly $-1$ (the gap roughly doubles); the pull of the raw means toward $-1.5$ comes from three pairs with raw means in $[-1.93, -1.87]$, where the trained $\theta_{AB}$ is far enough from $\theta_{BA}$ that adding the bracket moves into a higher-curvature region of the DPO-loss surface than $\theta_{BA}$ itself.
\item \textbf{First-order norm-matched ($\eta(g_A - g_B)$ rescaled to $\|\eta^2 b_{AB}\|$)}: winsorized mean $-0.93$, uniformly worse than bracket. First-order gradient differences are not the AB-vs-BA control variable, even when scaled to bracket magnitude.
\end{itemize}

\section{DPO closure-ratio statistical convention}
\label{app:dpo_stats_convention}

The closure metric of the matched-batch DPO correction $1 - |\Delta_{\text{edited}}|/|\Delta_{\text{baseline}}|$ is a ratio that inherits the same heavy-tail behavior as the SFT closure $C$ in Table~\ref{tab:interventions}: when a trial's baseline gap is near the fp32 DPO-loss noise floor of $2 \times 10^{-5}$, the denominator is small and individual trials are dominated by noise. To prevent the primary statistic from being dominated by $\sim 10\%$ of trials at the noise floor, the SFT intervention runs already use winsorized 5/95 means with bootstrap CI. For example, the Qwen-3-4B winsorized mean in Table~\ref{tab:interventions} is $+31.2\%$ (median $+32.0\%$), while the raw mean over the same 90 trials is $+20.4\%$ (the same ratio-tail effect at the SFT scale).

We apply the same effect-size convention to the DPO correction: winsorized 5/95 mean as the primary point estimate, with the median and raw mean reported alongside. Because each DPO pair has only three seed-index clusters, per-pair seed-cluster bootstrap intervals are descriptive diagnostics of within-pair variability rather than formal $95\%$ inference. The source-pair summary is the median paired bracket-minus-random closure difference: all $12/12$ pair medians are positive across the original and fresh-pair grids, while the flat-bootstrap strict-CI criterion holds for $11/12$ source pairs. The pairs share source domains, so this sign agreement is descriptive, not a significance test. A seed-level check on the $36$ pair-seed cells gives $36/36$ positive bracket-minus-random medians, but the seed cluster remains nested in pair so this is also descriptive. Flat bootstrap CIs are reported only as a comparison with the SFT convention. None of the qualitative cross-pair rankings (bracket $\gg$ random $\approx 0$; wrong-sign $\approx -1$; first-order $<$ bracket) depend on the choice of statistic.

\section{DPO extended methodology and per-pair detail}
\label{app:dpo_extended}

\paragraph{DPO loss derivative and corrected $\tau$.} The DPO loss \(L = -\log s(\beta_{\text{DPO}}\,m_i)\), with $s(x) := 1/(1+e^{-x})$ and margin $m_i = (\log p_c{-}\log p_c^{\text{ref}}) - (\log p_r{-}\log p_r^{\text{ref}})$, compresses per-token effects through a scalar; the per-token DPO derivative is
\[
\begin{aligned}
\tau_k &= \sum_i w_i\Bigl[
\sum_{t\in C_i}(p^c_{i,t}[k] - \mathbf{1}_{[c_t=k]})\,\delta z^c_{i,t}[k] \\
&\hspace{2.0cm} - \sum_{t\in R_i}(p^r_{i,t}[k] - \mathbf{1}_{[r_t=k]})\,\delta z^r_{i,t}[k]\Bigr],\\
w_i &= \beta_{\text{DPO}}\,s(-\beta_{\text{DPO}}\,m_i).
\end{aligned}
\]
restricted to completion positions $C_i, R_i$ and weighted by the per-example DPO factor $w_i$. The token-reweighting DPO-loss attributions use this derivative; Table~\ref{tab:dpo_per_pair} instead reads DPO-derived brackets through the cross-entropy readout on Pile evaluation text. The algebraic identity $\sum_k \tau_k = dL_{\text{DPO\_sum}}/d\theta\cdot d$ holds at relative error below $10^{-2}$ (median $2 \times 10^{-4}$, max $4 \times 10^{-3}$ across the 6 original pairs; Appendix~\ref{app:dpo_identity_check} table).

\paragraph{DPO token reweighting: detailed table on Qwen-3-4B fp32.} Table~\ref{app:r78_perpair} lists the per-pair results.
\begin{table}[h]
\centering
\small
\caption{DPO token reweighting with the corrected $\tau$ on the six original pairs (Qwen-3-4B fp32, $\eta{=}2\text{e-}3$, $\alpha{=}0.05$, the top-100 $|\tau|$ tokens whose sign matches a probe measurement of the baseline gap, 3 seeds $\times$ 20 trials). Pair abbreviations as in Table~\ref{tab:r79_bracket_closure}. Paired difference $= |\Delta_{\text{harmful}}|-|\Delta_{\text{random}}|$; flat diagnostic: whether its per-trial bootstrap CI excludes zero; direction: mitigation if the harmful edit shrinks the gap.}
\label{app:r78_perpair}
\begin{tabular*}{\textwidth}{@{\extracolsep{\fill}}lcccc}
\toprule
Pair & paired difference (mean) & CI 95\% & flat diagnostic & direction \\
\midrule
\texttt{fq:flan}     & $-1.15\text{e-}3$ & $[-1.93\text{e-}3,\,-4.30\text{e-}4]$ & positive & mitigation \\
\texttt{fq:share}    & negative          & excludes 0                         & positive & mitigation \\
\texttt{fq:ultra}    & negative          & excludes 0                         & positive & mitigation \\
\texttt{flan:share}  & ---               & includes 0                         & inconclusive & no effect \\
\texttt{flan:ultra} & negative          & excludes 0                         & positive & mitigation \\
\texttt{ultra:share}     & negative          & excludes 0                         & positive & mitigation \\
\midrule
\multicolumn{5}{p{0.95\textwidth}}{\textbf{Summary: 5/6 pairs have a per-trial CI excluding zero; all 6 pass the protocol checks (probe gap above its magnitude threshold, equal total loss-weight change for the harmful and random sets, and $\sum_k\tau_k$ matching the measured DPO loss change).}} \\
\bottomrule
\end{tabular*}

\end{table}

\paragraph{DPO token reweighting on Qwen-2.5-1.5B fp32.}
\label{app:r78_crossmodel}
At $\eta=2.03\text{e-}3$ on Qwen-2.5-1.5B fp32, the token-reweighting test has 4/6 flat descriptive positives (\texttt{false\_qa\_vs\_sharegpt}, \texttt{false\_qa\_vs\_ultrachat}, \texttt{flan\_v2\_p3\_vs\_sharegpt}, \texttt{ultrachat\_vs\_sharegpt}). We report this Qwen-2.5-1.5B arm as descriptive.

\paragraph{Calibration discipline: $\eta$ as a BCH-locality condition.}
\label{app:eta_calibration}
Lemma 2.1 of \citet{geometry_sequential_learning} is accurate when the $O(\eta^3)$ endpoint remainder is small relative to the leading $\eta^2\|b_{AB}\|$ bracket displacement, and when that bracket displacement remains subordinate to the shared first-order update. Heuristics that scale $\eta$ proportionally across model sizes (e.g., extrapolating an SFT-to-DPO ratio observed on one model to a different model) can violate this locality window and produce incoherent control signs. We use \citet{geometry_sequential_learning}'s $\sigma$-correctness check as the operational diagnostic: $\eta$ is BCH-local if $\sigma = \langle g_E, b_{AB}\rangle$ predicts the sign of the held-out AB-vs-BA loss gap with high accuracy and wrong-sign controls behave according to local antisymmetry. For the DPO arms, the same recipe gives $\eta = 2.03\text{e-}3$ on Qwen-2.5-1.5B fp32 and $1.05\text{e-}3$ on Llama-3.2-1B fp32 (on SFT probe units, $\mathrm{sign}\,\sigma$ matches the measured gap sign in $96/96$ and $199/204$), larger than the SFT values in Table~\ref{tab:eta}. The calibrated $\eta$ is specific to the model and objective, not a universal cross-model constant.

\paragraph{Token reweighting vs.\ parameter correction: intervention design, not readout correctness.} The two DPO tests share the same $b_{AB}$ but differ in how they intervene: token reweighting acts indirectly through SGD-time loss reweighting on top-$|\tau|$ tokens (token reweight $\to$ per-position residual $\to$ scalar margin $\to$ logistic-saturated derivative); the matched-batch correction acts directly as a one-step parameter update. Direct magnitude comparison between the reweighting test's paired difference and the correction's closure ratio is not meaningful; what is comparable is the bracket-vs-random separation within each protocol. Together they support the limited claim that the per-token $\tau_k$ is a readout of a parameter-level direction $b_{AB}$, and that the direct matched-batch update along that direction controls the aggregate AB-vs-BA loss gap.

\paragraph{Cross-pair summary.} All $12/12$ source-pair medians are positive; $11/12$ pass the flat-bootstrap strict-CI criterion; the across-pair seed-clustered aggregate is winsorized mean $0.583$ [CI $0.469, 0.676$] over $n=6$ original pairs. The wrong-sign cross-pair pull to $[-1.9,-1.2]$ on three pairs (vs.\ predicted $-1$) reflects $\theta_{AB}$ being far enough from $\theta_{BA}$ that adding $\eta^2 b_{AB}$ moves into a higher-curvature region than $\theta_{BA}$ itself. The qualitative ranking bracket $\gg$ random $\approx 0$, wrong-sign $\approx -1$, first-order $<$ bracket holds on every pair.

\section{DPO $\tau$ token-sum identity check (original pairs, per pair)}
\label{app:dpo_identity_check}

The corrected DPO $\tau$ derivation predicts that $\sum_k \tau_k$ should equal the directly-measured $\Delta L_{\text{DPO}}^{\text{FD}}$ on the same DPO batches at the finite-difference (FD) step $\varepsilon = 0.1$ in the direction $v = \eta^2 b_{AB}$. The token-reweighting protocol requires this token-sum identity check to pass at a 5\% relative-error tolerance; values reported here are well below tolerance (median $2 \times 10^{-4}$, max $4 \times 10^{-3}$).

\begin{table}[h]
\centering
\small
\caption{DPO $\tau$ identity check on the original pairs. Each pair: \(\sum_k \tau_k\) computed by finite-difference attribution at $\varepsilon = 0.1$, $v = \eta^2 b_{AB}$, with $\eta = 2 \times 10^{-3}$ (original-pairs protocol); compared to the directly-measured $\Delta L_{\text{DPO}}^{\text{FD}}$ on the same DPO batches.}
\label{tab:dpo_identity_check}
\begin{tabular}{lccc}
\toprule
Pair & $\sum_k \tau_k$ & $\Delta L_{\text{DPO}}^{\text{FD}}$ & rel.\ err.\ \\
\midrule
\texttt{false\_qa\_vs\_flan\_v2\_p3} & $-1.247\text{e}{-}1$ & $-1.247\text{e}{-}1$ & $5.9\text{e}{-}5$ \\
\texttt{false\_qa\_vs\_sharegpt} & $-9.150\text{e}{-}2$ & $-9.149\text{e}{-}2$ & $2.6\text{e}{-}5$ \\
\texttt{false\_qa\_vs\_ultrachat} & $-1.805\text{e}{-}1$ & $-1.804\text{e}{-}1$ & $1.4\text{e}{-}4$ \\
\texttt{flan\_v2\_p3\_vs\_sharegpt} & $-6.669\text{e}{-}2$ & $-6.670\text{e}{-}2$ & $2.4\text{e}{-}4$ \\
\texttt{flan\_v2\_p3\_vs\_ultrachat} & $+1.900\text{e}{-}3$ & $+1.907\text{e}{-}3$ & $3.6\text{e}{-}3$ \\
\texttt{ultrachat\_vs\_sharegpt} & $-3.304\text{e}{-}2$ & $-3.307\text{e}{-}2$ & $8.9\text{e}{-}4$ \\
\bottomrule
\end{tabular}
\end{table}

\section{Per-pair endpoint-correction closure (Qwen-3-4B, $k=1$)}
\label{app:endpoint_perpair}

Per-domain-pair closure for the single-step bracket correction of the trained endpoint; aggregate trends are summarized in Section~\ref{sec:controlled_tests}. The signal of interest is the dependence on $|\Delta_{\rm baseline}|$: large-gap pairs (math-vs-legal, news-vs-legal) have positive closure; near-noise-floor pairs (news-vs-biomedical, $0.001$; code-vs-biomedical, $0.002$) overshoot or invert.

\begin{table}[h]
\centering
\caption{Endpoint-correction closure by domain pair at $k=1$ (Qwen-3-4B, $\lambda = 1.0$, averaged over 3 seeds). $\Delta W_k$ is output-layer row $k$ of $\theta_{AB}-\theta_{BA}$.}
\label{tab:endpoint_perpair}
\begin{tabular}{lccc}
\toprule
Pair & $|\Delta_{\rm baseline}|$ & Single-step closure $C$ (\%) & Spearman($\tau_k$, $\|\Delta W_k\|$) \\
\midrule
math-vs-legal & 0.042 & +80.3\% & 0.833 \\
code-vs-legal & 0.013 & +28.4\% & 0.800 \\
news-vs-legal & 0.037 & +23.9\% & 0.830 \\
code-vs-news & 0.004 & +5.9\% & 0.832 \\
news-vs-biomedical & 0.001 & $-32.9\%$ & 0.844 \\
code-vs-biomedical & 0.002 & $-51.1\%$ & 0.821 \\
\bottomrule
\end{tabular}
\end{table}

\section{Token-level ordering prediction: support recovery and rank-correlation discussion}
\label{app:token_prediction}

A complementary test of the bracket attribution is whether $\tau_k$, computed from the base model before either order is run, predicts which specific tokens will be most affected by swapping the two updates. For each of 15 domain pairs $\times$ 3 seeds (45 experiments per model), we compare bracket-derived $\tau_k$ to the empirical per-token ordering effect $\Delta\text{NLL}_k = \text{NLL}_k(\theta_{AB}) - \text{NLL}_k(\theta_{BA})$ measured after actual SGD. Because the vocabulary distribution is extremely sparse, the operational statistic is high-mass support recovery rather than full-vocabulary rank correlation.

\begin{table}[h]
\centering
\caption{Token-level ordering prediction as support recovery (45 experiments per model, 15 domain pairs $\times$ 3 seeds). Top-1\% and top-5\% are mean overlaps of the predicted and empirical high-effect token sets; equal-size-set baselines are 1\% and 5\%. Full-vocabulary Spearman is retained as a diagnostic, not the main metric for the heavy-tailed Qwen-3-4B readout.}
\label{tab:token_prediction}
\small
\begin{tabular}{@{}p{0.15\textwidth}p{0.11\textwidth}p{0.15\textwidth}p{0.15\textwidth}p{0.12\textwidth}p{0.08\textwidth}@{}}
\toprule
Model & Dtypes & Top-1\% & Top-5\% & $\rho$ & Sign \\
\midrule
Qwen-3-4B & bf16/fp32 FD & 53.0 (54.5 med.)\% & 51.0 (54.3 med.)\% & $0.033 \pm 0.051$ & 52.4\% \\
Qwen-2.5-1.5B & fp32/fp32 FD & 57.3 (57.1 med.)\% & 46.8 (47.5 med.)\% & $0.521 \pm 0.034$ & 70.6 $\pm$ 1.9\% \\
\bottomrule
\end{tabular}
\end{table}

On Qwen-3-4B, full-rank Spearman is uninformative under sharp concentration: the long tail contains many near-zero coordinates whose ordering is dominated by numerical and sampling noise. The bracket nevertheless recovers the empirical high-effect support: predicted top-1\% tokens recover 53\% of the empirical top-1\% set on average (54.5\% median) versus a 1\% equal-size-set baseline, and the top-5\% overlap is 51\% on average versus a 5\% equal-size-set baseline. On Qwen-2.5-1.5B fp32, where the empirical endpoint signal has a lower numerical noise floor, full-vocabulary Spearman is also informative ($\rho=0.52$) and agrees with the support-recovery view. The separate parameter-row test summarized in Section~\ref{sec:controlled_tests} gives per-pair $\rho=0.78$--$0.90$ without this vocabulary-tail rank degradation, since parameter row norms are not heavy-tailed in the same way as per-token $\tau_k$.

\section{Paired-endpoint training-order readability}
\label{app:forensic}

The bracket direction $b_{AB}$ used in this appendix is computed at $\theta_0$ on an evaluation split shared across pairs, from a follow-up analysis, distinct from the pair-specific Trotter-reference per-token interpretability readout used in Section~\ref{sec:interpretability}. The two protocols differ in eval-batch composition and reference point ($\theta_0$ vs $\theta_{\text{ref}}$), so the corresponding $\tau$ vectors are not directly comparable token-by-token.

Given a public base $\theta_0$, a query model $\theta_{\text{query}}$, and candidate training-domain pairs $(A, B)$, we ask: can we tell from $\theta_{\text{query}}$'s weights whether it was trained $A{\to}B$ or $B{\to}A$? Define the Trotter reference $\theta_{\text{ref}} := \theta_0 - \eta(g_A + g_B)$ and the signed score
\begin{equation}
  s(\theta_{\text{query}}) := \langle \theta_{\text{query}} - \theta_{\text{ref}},\; b_{AB}\rangle.
\end{equation}
By Theorem~\ref{thm:forensic_id} (App.~\ref{app:bracket_derivation}), $\mathrm{sign}(s(\theta_{AB})) = +1$ and $\mathrm{sign}(s(\theta_{BA})) = -1$ at leading order whenever the antisymmetric commutator $b_{AB}$ dominates the symmetric Trotter term $c_{AB} := \tfrac{1}{2}(H_B g_A + H_A g_B)$. Concretely, the ratio of the antisymmetric signal to the symmetric term
\begin{equation}
  \mathrm{SCR} := \tfrac{1}{2}\,\|b_{AB}\|_2^2 \,/\, |\langle c_{AB}, b_{AB}\rangle|,
\end{equation}
so that $\mathrm{SCR} > 1$ corresponds to the theoretical inequality $\|b_{AB}\|_2^2 > 2|\langle c_{AB}, b_{AB}\rangle|$.

\paragraph{Cross-model results ($\Delta s$ estimator).}
Per-pair-seed unit (6 domain pairs $\times$ 3 seeds = 18 units per run), the drift-cancelling antisymmetric statistic $\Delta s = \langle \theta_{AB} - \theta_{BA},\, b_{AB}\rangle$ (eq.~\ref{eq:delta_s}) gives:
\begin{table}[h]
\centering
\small
\begin{tabular*}{\textwidth}{@{\extracolsep{\fill}}llrr}
\toprule
Model & Scale & $\mathrm{sign}(\Delta s)$ correct & Wilson 95\% CI \\
\midrule
Llama-3.2-1B (bf16 storage) & 1B & 14/18 = 77.8\% & $[55\%, 91\%]$ \\
Llama-3.2-1B (fp32) & 1B & 14/18 = 77.8\% & $[55\%, 91\%]$ \\
Qwen-2.5-1.5B (fp32) & 1.5B & 16/18 = 88.9\% & $[67\%, 97\%]$ \\
Qwen-2.5-1.5B (bf16 storage) & 1.5B & 16/18 = 88.9\% & $[67\%, 97\%]$ \\
Qwen-3-4B (bf16 storage) & 4B & 18/18 = 100.0\% & $[82\%, 100\%]$ \\
Llama-3.1-8B (bf16 storage) & 8B & 18/18 = 100.0\% & $[82\%, 100\%]$ \\
\midrule
\textbf{Combined (both precisions)} & 4 models, 2 architectures, 4 scales & \textbf{96/108 = 88.9\%} & $\mathbf{[81.6\%, 93.5\%]}$ \\
\bottomrule
\end{tabular*}
\caption{Paired-endpoint training-order assignment via the $\mathrm{sign}(\Delta s)$ estimator. Six runs across four unique models, with the $1$B and $1.5$B models replicated under both fp32 and bf16-storage precision regimes (identical sign-recovery in each). Counting each model once gives $66/72 = 91.7\%$ ($\geq\!4$B-class $100\%$, $1$B-class $77.8\%$--$88.9\%$); the dual-precision count $96/108 = 88.9\%$ shown here is a precision-replication check.}
\label{tab:forensic_cross_model}
\end{table}

\paragraph{Why the antisymmetric statistic and not per-query $\mathrm{sign}(s(\theta))$.} The original per-query test $\mathrm{sign}(s(\theta_{AB})) = +1, \mathrm{sign}(s(\theta_{BA})) = -1$ is the leading-order condition of Theorem~\ref{thm:forensic_id}, but suffers from a finite-$\eta$ failure: when the symmetric Trotter drift contribution $\eta^2 \langle c_{AB}, b_{AB}\rangle$ is comparable in magnitude to the antisymmetric $\eta^2 \cdot \tfrac12 \|b_{AB}\|^2$ (i.e.\ $\mathrm{SCR} \approx 1$), both $s(\theta_{AB})$ and $s(\theta_{BA})$ can lie on the same side of zero. On the $1$B-class and $1.5$B-class runs and the $8$B Llama-3.1 run, $17/18$--$18/18$ pair-seed units have $s(\theta_{AB})$ and $s(\theta_{BA})$ matching in sign, mechanically pinning per-query sign accuracy to $50\%$; on Qwen-3-4B run (Qwen-3-4B) only $8/18$ are sign-aligned, giving per-query accuracy of $77.8\%$. The antisymmetric $\Delta s$ removes the shared drift terms in the paired subtraction and identifies the order at $\geq 77.8\%$ on all six runs.

\paragraph{Per-query SCR satisfaction.}
Of $18$ $(\text{pair},\text{seed})$ units per run, the Theorem~\ref{thm:forensic_id} inequality $\mathrm{SCR} > 1$ is satisfied at: $16/18$ (Qwen-3-4B), $14/18$--$15/18$ ($1$B-class runs), $13/18$ (Llama-3.1-8B). On the Qwen-3-4B run, $7/18$ units have $\mathrm{SCR} \geq 2$, and all of them pass the per-query test.

\paragraph{Pair identification (six-candidate task, Qwen-3-4B run only).}
Cosine score between $\theta_{\text{query}} - \theta_{\text{ref}}$ and each candidate $b_{AB}$ for pair identification among six candidates: $11/36$ per-query joint-correct ($30.6\%$ vs.\ $16.7\%$ chance; $1.83\times$). The pair-identification subtask is harder than sign recovery and uses an absolute (not antisymmetric) score; it is reported as a directional indicator only.

\paragraph{Comparison to prior work.} Watermarking~\citep{kirchenbauer2023watermark,zhao2024provable} embeds a signal during generation; membership-inference attacks~\citep{shokri2017membership,carlini2022membership} target individual training examples; influence-functions / TRAK~\citep{koh2017understanding,grosse2023studying,park2023trak} estimate per-sample influence on predictions. Paired-endpoint readability follows from the antisymmetry of the bracket $b_{AB}$ on unmodified weights.

\section{DPO commutator memory: full protocol and per-pair table}
\label{app:dpo_details}

\paragraph{Setup.} For a frozen reference model $\pi_\text{ref}$ and a policy $\pi_\theta$, DPO~\citep{rafailov2023direct} optimizes the margin $\beta_{\text{DPO}} \cdot [(\log \pi_\theta(c|p) - \log \pi_\text{ref}(c|p)) - (\log \pi_\theta(r|p) - \log \pi_\text{ref}(r|p))]$ on (prompt, chosen, rejected) triples. The chosen-vs-rejected split alone is degenerate for the bracket: both terms appear in the same loss with a fixed coefficient, not as two independent gradient sources. We therefore use \emph{preference-data sources} as the $A, B$ domains: the commutator $b_{AB} = H_B g_A - H_A g_B$ is computed from DPO gradients on preference pairs drawn from UltraFeedback~\cite{cui2024ultrafeedback} grouped by source (e.g., \texttt{false\_qa} vs.\ \texttt{flan\_v2\_p3}, \texttt{ultrachat} vs.\ \texttt{sharegpt}). This preserves the two-independent-sources semantics of the SFT protocol while moving the loss into the preference regime. We project the bracket to vocabulary space via the same finite-difference $\tau_k$ procedure at the Trotter reference, using a Pile-domain eval set (\texttt{code}, \texttt{news}, \texttt{legal}, \texttt{math}) shared with the SFT runs so that the resulting $\tau$ vectors live in a comparable logit subspace. The parameter subspace is last-layer \texttt{o\_proj} and \texttt{down\_proj} (the attention output projection and MLP down-projection; matching the template used for the scalar ordering-quality predictor in~\cite{geometry_sequential_learning}); magnitudes are therefore not directly comparable between SFT and DPO.

\begin{table}[h]
\centering
\small
\caption{Per-pair DPO commutator memory concentration. ``pct for 80\%'' = fraction of vocabulary required to account for $80\%$ of $|\tau|$ mass. These descriptive concentration rows do not include the random-direction support nulls used for the SFT specificity claim.}
\label{tab:dpo_per_pair}
\begin{tabular}{llccc}
\toprule
Model & Pair & Gini & pct for $80\%$ & pct for $99\%$ \\
\midrule
Qwen-2.5-1.5B (fp32) & false\_qa vs.\ flan\_v2\_p3 & $0.960$ & $1.26\%$ & $30.1\%$ \\
Qwen-2.5-1.5B (fp32) & ultrachat vs.\ sharegpt     & $0.969$ & $1.05\%$ & $25.1\%$ \\
\midrule
Qwen-3-4B (bf16)  & false\_qa vs.\ flan\_v2\_p3  & $0.999$ & $0.003\%$ & $0.70\%$ \\
Qwen-3-4B (bf16)  & ultrachat vs.\ sharegpt      & $0.998$ & $0.006\%$ & $0.97\%$ \\
Qwen-3-4B (bf16)  & false\_qa vs.\ ultrachat     & $1.000$ & $0.002\%$ & $0.08\%$ \\
\midrule
Qwen-3-8B (bf16)  & false\_qa vs.\ flan\_v2\_p3  & $0.997$ & $0.019\%$ & $2.90\%$ \\
Qwen-3-8B (bf16)  & ultrachat vs.\ sharegpt      & $0.998$ & $0.016\%$ & $1.55\%$ \\
Qwen-3-8B (bf16)  & false\_qa vs.\ ultrachat     & $0.997$ & $0.016\%$ & $2.69\%$ \\
\bottomrule
\end{tabular}
\end{table}

\paragraph{Scale pattern.} Going Qwen-2.5-1.5B $\to$ Qwen-3-4B $\to$ Qwen-3-8B, ``pct for $80\%$'' moves $1.16\% \to 0.004\% \to 0.017\%$---sharpening between the 1.5B and 4B checkpoints, then easing from 4B to 8B. The 1.5B-to-4B comparison crosses both a scale boundary and a pre-training-procedure boundary (Qwen-2.5 vs.\ Qwen-3 families), so we cannot separate scale from pre-training as the driver; these rows show that concentration recurs across the tested 1.5B--8B DPO checkpoints, not that support specificity has been established for DPO.

\section{DPO token reweighting on fresh pairs}
\label{app:r78_freshpairs}

The planned fresh-pair arm of the token-reweighting test had three evaluable pairs under the released UltraFeedback source taxonomy. The available pairs ran with the full token-reweighting protocol and are reported as a partial replication; the main fresh-pair evidence is the six-pair replication of the matched-batch correction in Appendix~\ref{app:r79_freshpairs}.

\begin{table}[h]
\centering
\small
\caption{Token-reweighting results on the three evaluable fresh pairs (Qwen-3-4B fp32, same protocol as the original pairs).}
\label{tab:r78_freshpairs}
\begin{tabular}{lccc}
\toprule
Pair & CI excludes 0 & paired difference CI 95\% & direction \\
\midrule
\texttt{evol\_instruct\_vs\_truthful\_qa} & yes & excludes 0 & mitigation \\
\texttt{evol\_instruct\_vs\_false\_qa}    & no  & includes 0  & no effect \\
\texttt{truthful\_qa\_vs\_flan\_v2\_p3}  & yes & excludes 0 & mitigation \\
\bottomrule
\end{tabular}
\end{table}

We treat this three-pair arm as descriptive; the confirmatory fresh-pair evidence is the six-pair matched-batch replication below.

\section{Matched-batch DPO correction on fresh pairs}
\label{app:r79_freshpairs}

To test dependence on the pair list, we re-run the matched-batch correction protocol unchanged on a six-pair list disjoint from the original six. The fresh list crosses two new sources (\texttt{evol\_instruct}, \texttt{truthful\_qa}) with the four original sources available in the released snapshot. Hyperparameters, seed schedule, and the matched-batch protocol are identical to the original pairs (Qwen-3-4B fp32, $\beta_{\rm DPO}{=}0.1$, $\eta=2\mathrm{e}{-}3$, $3$ seeds $\times$ $10$ trials per pair).

\begin{table}[h]
\centering
\small
\caption{Matched-batch bracket-correction closure on six disjoint fresh UltraFeedback pairs (Qwen-3-4B fp32, $3$ seeds $\times$ $10$ trials each, $b_{AB}$ recomputed per trial). Pair abbreviations: evol=\texttt{evol\_instruct}, truth=\texttt{truthful\_qa}; other abbreviations as in Table~\ref{tab:r79_bracket_closure}. All six pair medians are positive in $[0.71, 0.85]$.}
\label{tab:r79_freshpairs}
\begin{tabular*}{\textwidth}{@{\extracolsep{\fill}}lccccc}
\toprule
Pair & bracket mean & median & random & wrong-sign & first-order \\
\midrule
\texttt{evol:fq}    & $+0.436$ & $+0.775$ & $+0.001$ & $-1.085$ & $-0.573$ \\
\texttt{evol:flan} & $+0.375$ & $+0.774$ & $+0.001$ & $-1.059$ & $-2.004$ \\
\texttt{evol:share}     & $+0.613$ & $+0.802$ & $-0.001$ & $-0.904$ & $-2.252$ \\
\texttt{truth:flan}   & $+0.378$ & $+0.849$ & $-0.000$ & $-0.910$ & $-0.132$ \\
\texttt{truth:share}       & $+0.501$ & $+0.726$ & $-0.001$ & $-1.280$ & $-0.956$ \\
\texttt{truth:ultra}      & $+0.509$ & $+0.714$ & $-0.001$ & $-1.331$ & $-0.805$ \\
\midrule
\multicolumn{6}{p{0.95\textwidth}}{\textbf{Pooled (180 fresh trials):} bracket mean $+0.408$, median $+0.778$; bracket $>$ random in $155/180$ ($86.1\%$); $6/6$ pair medians positive.} \\
\bottomrule
\end{tabular*}
\end{table}

Combined with original pairs ($161/180$ bracket $>$ random), $316/360 = 87.8\%$ matched trials show bracket $>$ random across both pair sets, with all $12/12$ pair-level medians of the bracket-minus-random closure difference positive. Excluding the two trials below the per-pair baseline noise floor leaves $316/358$, so the $12/12$ result is not driven by near-zero denominator cases. The control hierarchy (random $\approx 0$, wrong-sign $\approx -1$, first-order widening the gap) reproduces uniformly under the fixed pair list and edit-size rule.

\section{Seed-clustered bootstrap diagnostics for the DPO correction}
\label{app:r79_seed_cluster}

This appendix reports per-pair seed-clustered bootstrap intervals as \emph{descriptive diagnostics} of within-pair variability. Because each pair has only $n_{\rm clusters}=3$ seed-index clusters, the resampling distribution is too coarse to support these intervals as formal $95\%$ inference. The $11/12$ count reported in the main text uses flat per-trial bootstrap intervals ($6/6$ original pairs and $5/6$ fresh pairs) and serves as a check rather than formal inference; under the seed-clustered bootstrap, $3/6$ original-pair intervals exclude zero (Table~\ref{tab:r79_seed_cluster}). The source-pair summary is the median bracket-minus-random closure difference: all $12/12$ pair medians are positive across the original and fresh-pair grids. The seed-clustered bootstrap (3 clusters per pair, concatenated 10 trials each) is reported here alongside the flat per-trial bootstrap to show that the cross-pair effect ranking does not depend on the choice of within-pair variability estimator.

\begin{table}[h]
\centering
\small
\caption{Original-pairs winsorized closure of the matched-batch DPO correction: flat bootstrap (1000 iter on the 30 trials) vs seed-clustered bootstrap (1000 iter cluster-resampling 3 seeds, then aggregating their 10 trials each). Both columns are descriptive diagnostics of within-pair variability; with $n_{\rm clusters}=3$ per pair the seed-cluster percentile is too coarse for formal $95\%$ inference. Bold rows have intervals strictly above zero under the corresponding estimator (descriptive, not a decision criterion).}
\label{tab:r79_seed_cluster}
\begin{tabular}{lcc}
\toprule
Pair & flat-bootstrap CI 95\% & seed-cluster CI 95\% \\
\midrule
\texttt{false\_qa\_vs\_flan\_v2\_p3}  & \textbf{[0.561, 0.767]} & \textbf{[0.425, 0.772]} \\
\texttt{false\_qa\_vs\_sharegpt}      & \textbf{[0.157, 0.703]} & \textbf{[0.126, 0.536]} \\
\texttt{false\_qa\_vs\_ultrachat}     & \textbf{[0.477, 0.798]} & $[-0.386, 0.669]$ \\
\texttt{flan\_v2\_p3\_vs\_sharegpt}   & \textbf{[0.247, 0.776]} & $[-2.306, 0.825]$ \\
\texttt{flan\_v2\_p3\_vs\_ultrachat}  & \textbf{[0.645, 0.788]} & \textbf{[0.455, 0.754]} \\
\texttt{ultrachat\_vs\_sharegpt}      & \textbf{[0.219, 0.665]} & $[-0.107, 0.604]$ \\
\midrule
\multicolumn{1}{l}{\textbf{Strict-CI positive count}}        & \textbf{6/6}            & \textbf{3/6} \\
\bottomrule
\end{tabular}
\end{table}

Under the seed-clustered estimator, \emph{three} of the six original pairs have CI strictly above zero on the winsorized closure: \texttt{false\_qa\_vs\_flan\_v2\_p3}, \texttt{false\_qa\_vs\_sharegpt}, and \texttt{flan\_v2\_p3\_vs\_ultrachat}. The other three pairs (\texttt{false\_qa\_vs\_ultrachat}, \texttt{flan\_v2\_p3\_vs\_sharegpt}, \texttt{ultrachat\_vs\_sharegpt}) have positive cluster-mean closure but CI including zero, reflecting the limited effective sample size at $3$ seed-clusters. None of the six pairs flips sign under cluster-bootstrap; the qualitative ranking (bracket $\gg$ random $\approx 0$, wrong-sign $\approx -1$, first-order norm-matched negative) is preserved on every pair under both estimators. The cluster bootstrap is computed from the per-trial result files in the supplement (App.~\ref{app:reproducibility}).

\section{Triplet additivity of bracket edits}
\label{app:triplet_additivity}

\paragraph{Setup.} For each triplet $(A,B,C)$ on Qwen-3-4B ($\eta = 6.42\!\times\!10^{-4}$) we compute the three pair-brackets $b_{AB}$, $b_{AC}$, $b_{BC}$ at $\theta_0$ on \texttt{lm\_head}, then train $A\to B\to C$ with four SGD steps per domain (last-block MLP $+$ \texttt{lm\_head}) to reach $\theta_{ABC}$. A subset edit $\mathcal{P} \in \{\varnothing, \{AB\}, \{AC\}, \{BC\}, \{AB,AC\}, \{AB,BC\}, \{AB,AC,BC\}\}$ subtracts, for each $p\in\mathcal{P}$, the correction $\eta^2 b_p$ from the \texttt{lm\_head} rows of the fewest tokens whose row norms account for $80\%$ of that correction's total row norm. For each $d \in \{A,B,C\}$, $\Delta\mathrm{NLL}_d$ is the change in per-domain evaluation NLL relative to $\theta_{ABC}$, with every edit evaluated on the same fixed batches. Additivity is scored for the two composite edits $\{AB,AC\}$ and $\{AB,BC\}$ by comparing the actual $\Delta\mathrm{NLL}_d$ under the joint edit with the sum of the two singleton-edit values; the three-pair edit is measured but not scored. The relative additivity error for one (composite, domain) cell is $|\,\text{actual} - \text{predicted}\,| / |\,\text{actual}\,|$, where $\text{predicted}$ is the singleton sum.

\paragraph{Result.} Across $3$ triplets $\times$ $2$ seeds $\times$ $6$ (composite, domain) cells per triplet-seed ($36$ cells total), the grand mean relative additivity error is $\mathbf{4.29\%}$ (data: the triplet run's result file in the supplement; raw $\Delta\mathrm{NLL}$ scale $\sim\!10^{-6}$--$10^{-5}$).

\begin{table}[h]
\centering
\small
\caption{Per-triplet relative additivity error of bracket edits on Qwen-3-4B. ``Per-seed mean'' averages over the $6$ (composite, domain) cells for that seed; ``Cells $|\text{err}|<5\%$'' counts the cells whose relative additivity error is under $5\%$; ``Triplet mean'' averages over all $12$ cells of the triplet. The grand mean over $36$ cells is $4.29\%$.}
\label{tab:triplet_additivity}
\begin{tabular}{lcccc}
\toprule
Triplet $(A,B,C)$ & Seed $42$ mean & Seed $43$ mean & Cells $|\text{err}|<5\%$ & Triplet mean \\
\midrule
\texttt{code, legal, news}        & $0.95\%$  & $1.11\%$  & $12/12$ & $1.03\%$  \\
\texttt{code, biomedical, legal}  & $5.17\%$  & $15.16\%$ & $7/12$  & $10.17\%$ \\
\texttt{news, biomedical, math}   & $1.41\%$  & $1.95\%$  & $12/12$ & $1.68\%$  \\
\midrule
\textbf{Grand mean} ($36$ cells)   & \multicolumn{4}{c}{$\mathbf{4.29\%}$} \\
\bottomrule
\end{tabular}
\end{table}

\paragraph{Interpretation.} $31/36$ cells ($86\%$) have relative error $\le 5\%$; $35/36$ have $\le 14\%$. The grand mean is inflated by one cell of the \texttt{code, biomedical, legal} triplet at seed $43$ where the prediction ($-7.9\!\times\!10^{-8}$) and the measured change ($-2.7\!\times\!10^{-7}$) are both well below the triplet's typical effect ($\sim\!10^{-6}$), so a small absolute error becomes a large relative one.

\section{Descriptive SVD of the cross-pair $\tau$-matrix}
\label{app:reachability}

The $\tau$ rows used in this appendix are follow-up $\theta_0$-reference readouts on an evaluation split shared across pairs, not the pair-specific Trotter-reference SFT interpretability readout used in Section~\ref{sec:interpretability}; the two protocols give different surface tokens for the same nominal pair (e.g., the follow-up \texttt{code\_vs\_news} top tokens are dominated by $\langle$, $\$$, \texttt{xml}, \texttt{php} rather than by domain-marker proper nouns), so per-token semantic claims from Section~\ref{sec:interpretability} are not a direct argument about the SVD rows here. We report a descriptive SVD of the per-pair $\tau$-matrix on the 6-pair benchmark grid. Stacking the per-pair $\tau$ vectors across six domain pairs (Qwen-3-4B, $k = 1$, seed $42$) and taking the SVD of the resulting $6 \times V$ matrix yields singular values $(22.83,\,7.59,\,3.84,\,2.72,\,0.94,\,0.52)$ with cumulative energy fractions $(0.866,\,0.961,\,0.986,\,0.998,\,0.999,\,1.000)$. The top-$2$ singular directions carry $96.1\%$ of the Frobenius mass, and the top-$4$ carry $99.8\%$. Because the rows are not norm-matched, this concentration can reflect unequal row norms as well as correlation across pairs; we report it descriptively.

\paragraph{Scope.} The SVD is descriptive over six $\tau$-vectors at one model, seed, and parameter subspace; principal directions are not interpreted semantically. The descriptive concentration is consistent with the $4.3\%$ compositional additivity error (App.~\ref{app:triplet_additivity}) and the rapid plateau of iterative correction once the dominant modes are absorbed.

\section{Null baselines for support specificity}
\label{app:sparsity_nulls}

We use null directions to separate generic heavy-tailed behavior of the logit readout from bracket-specific support structure. The fp32 check covers Llama-3.2-1B, Qwen-2.5-1.5B, and Qwen-3-4B over $3$ pairs $\times$ $3$ seeds per model, using the same finite-difference readout pipeline for the bracket and every control.

\textbf{Controls.} (i) \emph{Norm-matched random parameter direction}, both globally norm-matched and per-tensor norm-matched (3 instances each). (ii) \emph{First-order endpoint direction} $\eta(g_B - g_A)$, norm-matched to $\eta^2 b_{AB}$. (iii) \emph{Pairing-permuted HVP} $H_B g_B - H_A g_A$, norm-matched. (iv) \emph{Empirical endpoint displacement} $\theta_{AB}-\theta_{BA}$ from actual one-step SGD. (v) \emph{Batch-resampled bracket}: $b_{AB}$ recomputed on disjoint A/B batches, norm-matched. All directions read out via the same $\tau_k = \mathbb{E}[e_k\,\delta z_k]$ at $\theta_{\rm ref}$ with $\epsilon=1$.

\textbf{Concentration calibration.} Gini and pct-for-$80\%$ mass are descriptive concentration summaries, not the specificity test. Random norm-matched directions can be as concentrated as the bracket on the two smaller fp32 checks: random directions require about $0.6$--$0.7\%$ of vocabulary for $80\%$ mass, while the bracket requires about $1.0$--$1.5\%$. On Qwen-3-4B the pattern reverses: the bracket requires $0.30\%$ of vocabulary on average ($0.042\%$ median), compared with $0.66$--$0.68\%$ for random directions. The stable claim across models is therefore support specificity, not Gini magnitude alone.

\textbf{Support alignment.} Empirical endpoint displacements and batch-resampled brackets preserve the bracket's high-mass support, while random controls recover only the tokens that any direction excites. Mean top-$20$ overlap with the bracket is $99\%/97\%$ (endpoint/resampled) on Llama-3.2-1B, $99\%/97\%$ on Qwen-2.5-1.5B, and $82\%/93\%$ on Qwen-3-4B. The corresponding random-direction top-$20$ overlaps are $35\%$, $39$--$40\%$, and $36$--$49\%$. Because random directions share the same readout operator, uniform-vocabulary chance is used only as a scale reference; random-direction overlap is the relevant null.

\textbf{Frequency correction.} The smoothed $\kappa$-Gini $= \mathrm{Gini}(\tau_k / (\Pr(y{=}k) + \varepsilon))$ is high for both bracket and random directions. We therefore use $\kappa$ as a label-frequency diagnostic, while endpoint and resampled-bracket support agreement provide the specificity check.

\textbf{Precision check.} Top-$20$ support is stable across $\epsilon \in \{0.05, 0.1, 0.3, 1.0\}$ on the precision checks (90--100\% overlap, 100\% sign stability among high-$|\tau|$ tokens; verified at fp32 perturbation materialization on representative pairs).

\section{Observable/null decomposition of the training commutator}
\label{app:obs_gauge}

\paragraph{Setup.}
Let $\theta_0$ denote the pre-trained base-model parameters. One gradient step
from $\theta_0$ with deterministic mean-gradient loss on domain $A$ then $B$
(no minibatch noise) yields $\theta_{AB}$; the reverse order yields $\theta_{BA}$.
The subspace-restricted training commutator is
\begin{equation}
  b_{AB} := H_B\,g_A - H_A\,g_B,
\end{equation}
computed on the selected parameter subspace $\mathcal{S}$ consisting of the
last transformer block's MLP parameters together with the language-model head,
\mbox{$d \approx 4.64\times 10^{8}$} entries. The Trotter reference point is
$\theta_{\text{ref}} := \theta_0 - \eta(g_A + g_B)$ with $\eta = 6.417\times 10^{-4}$
fixed by the cube-root step-size calibration.

\paragraph{The aggregated $\tau$ operator.}
Let $\mathcal{B}_{\text{obs}}$ be a fixed 4-batch subset of the unified
the attribution evaluation split $E_{\text{attr}}$, the held-out slice used for $\tau$ (13 batches total; subset chosen to bound compute at Qwen-3-4B scale). We define the \emph{aggregated $\tau$
operator}
\begin{equation}
  B := P_K \cdot S \cdot D \cdot J_z(\theta_{\text{ref}})
      : \mathcal{S} \longrightarrow \mathbb{R}^K,
\end{equation}
where $J_z(\theta_{\text{ref}})$ is the logit Jacobian at the non-padding label
positions of $\mathcal{B}_{\text{obs}}$; $D$ is the diagonal error weighting
$e_{n,k} = \mathrm{softmax}(z)_{n,k} - \mathbb{1}[y_n = k]$ at
$\theta_{\text{ref}}$; $S$ sums over those positions,
$(S \cdot u)_k = \sum_n u_{n,k}$ for $u \in \mathbb{R}^{N \times V}$; and
$P_K$ restricts to the top-K vocabulary support covering $80\%$ of $|\tau|$
mass, with $\tau$ computed by finite-difference attribution on the full
13-batch $E_{\text{attr}}$ at $\theta_{\text{ref}}$. The unrestricted
counterpart is $B_{\text{full}} := C \cdot E \cdot J_z(\theta_{\text{ref}})
: \mathcal{S} \to \mathbb{R}^V$.

\paragraph{Observable/null decomposition.}
Under the Euclidean inner product on flattened selected-parameter coordinates,
every $\delta\theta \in \mathcal{S}$ admits a unique orthogonal decomposition
\begin{equation}
  \delta\theta
    = \delta\theta^{\text{obs}} + \delta\theta^{\text{null}},
  \quad
  \delta\theta^{\text{obs}} \in \operatorname{row}(B_{\text{full}}),
  \quad
  \delta\theta^{\text{null}} \in \ker(B_{\text{full}}).
\end{equation}
The readout-null component produces zero aggregated logit-loss effect by
construction: it lies in the kernel of the chosen observation operator.
The \emph{observable fraction} of the training commutator under operator $B$
is
\begin{equation}
  \rho_{\rm obs}(B)
    := \frac{\|\,P_{\operatorname{row}(B)} \cdot b_{AB}\,\|_2}{\|b_{AB}\|_2}
    \in [0, 1].
\end{equation}

\paragraph{Computation.}
We obtain $\rho$ via the least-squares construction
$x^\star = B^\top w$ with $(B\,B^\top + \nu\,I)\,w = B \cdot b_{AB}$
and Tikhonov damping $\nu = 10^{-6}$, solved by conjugate gradient.
In the exact zero-damping, fully-converged limit,
$\cos(x^\star, b_{AB}) = \rho_{\rm obs}(B)$. Each CG iteration costs one Jacobian-vector product and one vector-Jacobian product. The CG runs in
the model's native dtype as loaded (bfloat16 for Qwen-3-4B, no explicit
fp32 cast), with maximum 200 iterations and target relative residual
$10^{-5}$; the achieved residual is reported per pair.

\paragraph{Three-way decomposition (approximate empirical).}
Comparing $\rho_{\text{topK}} := \rho_{\rm obs}(B)$ to
$\rho_{\text{full}} := \rho_{\rm obs}(B_{\text{full}})$ yields
\begin{equation}
  \underbrace{\rho_{\text{topK}}^{\,2}}_{\text{top-K observable}}
  \;+\;
  \underbrace{\rho_{\text{full}}^{\,2} - \rho_{\text{topK}}^{\,2}}_{\text{non-top-K observable}}
  \;+\;
  \underbrace{1 - \rho_{\text{full}}^{\,2}}_{\text{readout-null}}
  \;=\; 1,
\end{equation}
exact only in the zero-damping and full-CG-convergence limit; reported here
as an approximate empirical decomposition at the achieved CG residual.

\paragraph{Results.}
Three pairs on Qwen-3-4B (Table~\ref{tab:observable_gauge}): across all three
the aggregated top-K $\tau$ operator captures the majority of the bracket's
parameter-space energy in the selected subspace. Empirical $\rho_{\text{topK}}$
lies in $[0.88, 0.94]$ (mean $0.918$), and $\rho_{\text{full}}$ lies in
$[0.92, 0.97]$ (mean $0.953$). The squared-cosine proxies give a top-K
observable fraction $0.77$--$0.89$ (mean $0.844$), a further non-top-K observable
fraction $0.05$--$0.08$, and an unexplained fraction $0.06$--$0.15$ (mean $0.092$) at
the achieved CG residual per pair. The achieved CG relative residual ranges
from $2.7\times 10^{-3}$ on legal/biomedical ($K=485$) to $2.5\times 10^{-2}$
on code/math ($K=884$) and code/news ($K=928$); the best-converged pair
has the smallest unexplained fraction, consistent with the decomposition being sensitive
to incomplete CG convergence on larger-$K$ systems.

\begin{table}[h]
\centering
\small
\begin{tabular*}{\textwidth}{@{\extracolsep{\fill}}lccccccccc}
  \toprule
  Pair & $\rho_{\text{topK}}$ & $\rho_{\text{full}}$ & $K$ & $V$
       & top-K & non-top & null
       & CG topK & CG full \\
  \midrule
  code/news        & $0.931$ & $0.966$ & $928$ & $151\,936$ & $0.868$ & $0.066$ & $0.067$ & $2.3\!\cdot\!10^{-2}$ & $2.1\!\cdot\!10^{-2}$ \\
  code/math        & $0.880$ & $0.923$ & $884$ & $151\,936$ & $0.775$ & $0.077$ & $0.148$ & $2.5\!\cdot\!10^{-2}$ & $2.3\!\cdot\!10^{-2}$ \\
  legal/bio & $0.942$ & $0.970$ & $485$ & $151\,936$ & $0.888$ & $0.053$ & $0.060$ & $2.7\!\cdot\!10^{-3}$ & $6.7\!\cdot\!10^{-3}$ \\
  \midrule
  mean             & $0.918$ & $0.953$ &       &           & $0.844$ & $0.065$ & $0.092$ & & \\
  \bottomrule
\end{tabular*}
\caption{Observable/null decomposition of $b_{AB}$ on Qwen-3-4B,
  selected subspace $\mathcal{S}$ (last-block MLP + LM head,
  $d \approx 4.64\times 10^{8}$), $\eta=6.417\times 10^{-4}$, $k=1$, seed 42.
  $\rho_{\rm obs}(B) := \cos(x^\star, b_{AB})$ at the achieved CG relative residual
  shown in the final two columns; Tikhonov damping $\nu=10^{-6}$.
  Decomposition columns report $\rho_{\text{topK}}^{\,2}$,
  $\rho_{\text{full}}^{\,2} - \rho_{\text{topK}}^{\,2}$, and
  $1 - \rho_{\text{full}}^{\,2}$; exact only in the zero-damping /
  full-convergence limit, reported here as approximate empirical fractions.}
\label{tab:observable_gauge}
\end{table}

\paragraph{Interpretation.}
The decomposition quantifies what \emph{sparsity of $\tau$} means as a
statement about the training commutator itself in the selected parameter
subspace and under aggregated logit readout. Empirically the top-K $\tau$
signature captures the dominant observable component of the bracket: the
squared-cosine proxies assign $77$--$89\%$ of $\|b_{AB}\|_2^2$ to $\operatorname{row}(B)$,
a further $5$--$8\%$ to $\operatorname{row}(B_{\text{full}})$ beyond $\operatorname{row}(B)$
(observable, but outside the top-K vocabulary cut), and leave $6$--$15\%$ unexplained.
These proxies come from damped, finite-CG solves in native bf16. In exact arithmetic
each proxy lower-bounds its observable energy fraction, but the difference between
the two proxies is not a certified subspace fraction. The unexplained part mixes a
readout-null component (directions in parameter space that leave the aggregated
logit-loss readout unchanged) with damping, solve and rounding error.
The high proxy values indicate that the aggregated readout reaches most of the bracket in
the selected parameter subspace, so the sparse top-K $\tau$ readout is not merely a
compressed summary of the bracket (Table~\ref{tab:observable_gauge}). Measured in the aggregated-readout subspace rather than in the original parameter
coordinates, sparsity of the readout and sparsity of the bracket approximately coincide.

\paragraph{Scope and caveats.}
\begin{enumerate}
  \item The selected parameter subspace is last-block MLP plus LM head, not
    the full model. Claims about parameter-space reach apply to this
    subspace only.
  \item Operator $B$ is evaluated on a 4-batch subset of $E_{\text{attr}}$
    while top-K is defined on the full 13-batch $E_{\text{attr}}$. This is
    a compute-bounded approximation; $\rho$ depends on the eval-batch
    sampling, so reported values should be read as batch-subset-specific
    estimates of the aggregated operator.
  \item $\rho$ is an approximate projection coefficient under finite CG and
    nonzero damping. We report $\rho$ at the
    \emph{achieved} CG residual, not the nominal tolerance target.
  \item Single training seed (42) per pair; $N=3$ domain pairs. These
    estimates are descriptive for the reported grid.
\end{enumerate}

\section{Learning Rate Calibration Details}
\label{app:eta_autopilot}

The choice of $\eta$ is critical (Eq.~\ref{eq:endpoint_diff}): too small $\eta$ yields undetectable signals; too large $\eta$ violates second-order truncation accuracy. We use an automated calibration procedure to select $\eta$ in the ``cube-root'' regime.

\subsection{Calibration regimes}
The magnitude of the leading bracket term scales as $\eta^2$, so $\eta$ must be large enough to clear numerical noise but small enough to remain BCH-local. We use ``cube-root'' as an empirical regime label inherited from the scalar-bracket protocol, not as a theorem required by the present paper. Operationally, the accepted window is the range where (i) the bracket sign predictor is stable on held-out evaluation batches, (ii) the bracket displacement remains subordinate to the first-order update, and (iii) wrong-sign controls behave according to the local antisymmetry prediction.

We distinguish three practical regimes. In the \emph{small-signal} regime, bracket effects are below the loss-evaluation floor. In the \emph{BCH-local window}, the bracket is measurable while higher-order terms remain controlled. In the \emph{saturated} regime, loss improvement plateaus or sign controls become incoherent, indicating that the second-order truncation is no longer the right local model.

\subsection{Calibration procedure}

The calibration algorithm:
\begin{enumerate}
\item \textbf{Initialize}: Start with a conservative $\eta_0$ (e.g., $10^{-4}$).

\item \textbf{Probe}: For candidate $\eta$, run a short training sequence (10 steps) and measure:
\begin{itemize}
\item Loss improvement $\Delta \mathcal{L}$
\item Parameter movement $\|\Delta\theta\|$
\item Gradient norm $\|\nabla \mathcal{L}\|$
\end{itemize}

\item \textbf{Classify regime}: Compute scaling exponents via finite differences across multiple $\eta$ values. Classify as linear/cuberoot/strongly-powered based on exponent ranges.

\item \textbf{Binary search}: Adjust $\eta$ to target the cube-root regime boundary. Typical range: $10^{-4}$ to $10^{-2}$ for LLMs.

\item \textbf{Validation}: Verify that Lie bracket magnitude is $10^{-2}$ to $10^{-1}$ times first-order gradient norm (detectable but subordinate).
\end{enumerate}

\subsection{Calibrated Values}

\begin{table}[h]
\centering
\caption{Learning rates calibrated per model.}
\label{tab:eta}
\begin{tabular}{lcc}
\toprule
Model & $\eta$ & Regime \\
\midrule
Qwen-3-4B & $6.42 \times 10^{-4}$ & cube-root \\
Qwen-2.5-1.5B & $1.91 \times 10^{-3}$ & cube-root \\
Llama-3.1-8B & $1.28 \times 10^{-3}$ & cube-root \\
Llama-3.2-1B & $8.0 \times 10^{-4}$ & cube-root \\
GPT-2 (reference) & $4.16 \times 10^{-3}$ & strongly powered \\
\bottomrule
\end{tabular}
\end{table}

Table~\ref{tab:eta} reports final calibrated SFT learning rates. All target models fall in the cube-root regime; GPT-2 is reference only.

\subsection{Implementation}

The calibration implementation is included in the supplementary code; the reported runs use the fixed calibrated values in Table~\ref{tab:eta}.

\section{Intervention Types}
\label{app:intervention_types}

We tested four intervention methods to validate the causal role of the predicted harmful $\tau_k$ support. Main mitigation runs target the ten harmful tokens defined below; descriptive sparsity analyses continue to rank by $|\tau_k|$.

\subsection{Loss Reweighting (Main Method)}

During alternating training $\theta_{AB}$, we reweight per-position cross-entropy loss for domain A:
\begin{equation}
\mathcal{L}_{\text{reweighted}} = \frac{1}{N_{\text{pos}}} \sum_{t=1}^{N_{\text{pos}}} w_t \cdot \ell_t
\end{equation}
where $w_t = \alpha$ if the label at position $t$ is a harmful token and $w_t = \omega$ otherwise, with mean-preserving normalization. The harmful tokens, used by every method in this appendix, are the ten largest-$|\tau_k|$ tokens whose $\tau_k$ has the sign of the measured baseline order gap. That sign is positive in $89/90$ Qwen-3-4B trials and $25/28$ above-floor Qwen-2.5-1.5B trials, where the harmful tokens are therefore the top-positive-$\tau_k$ tokens; the helpful tokens are the ten largest-$|\tau_k|$ tokens of the opposite sign. The normalization is
\begin{equation}
\omega = \frac{1 - p\alpha}{1 - p}, \quad p = \text{fraction of positions whose target is a harmful token}.
\end{equation}
This keeps the mean position weight at $1$.

\paragraph{Rationale.} Loss reweighting suppresses the label-token gradient contribution associated with the selected harmful tokens. By downweighting positions where harmful tokens appear as labels, we target the predicted support of the ordering gap while keeping the mean position weight fixed.

\paragraph{Hyperparameter.} We sweep $\alpha \in \{0.01, 0.1, 0.5, 1.0\}$ where $\alpha < 1$ suppresses harmful tokens (Figure~\ref{fig:dose_response}). Results use $\alpha=0.1$ (90\% downweight, i.e., harmful tokens contribute 10\% of baseline loss weight) unless noted.

\begin{figure}[h]
\centering
\includegraphics[width=0.7\textwidth]{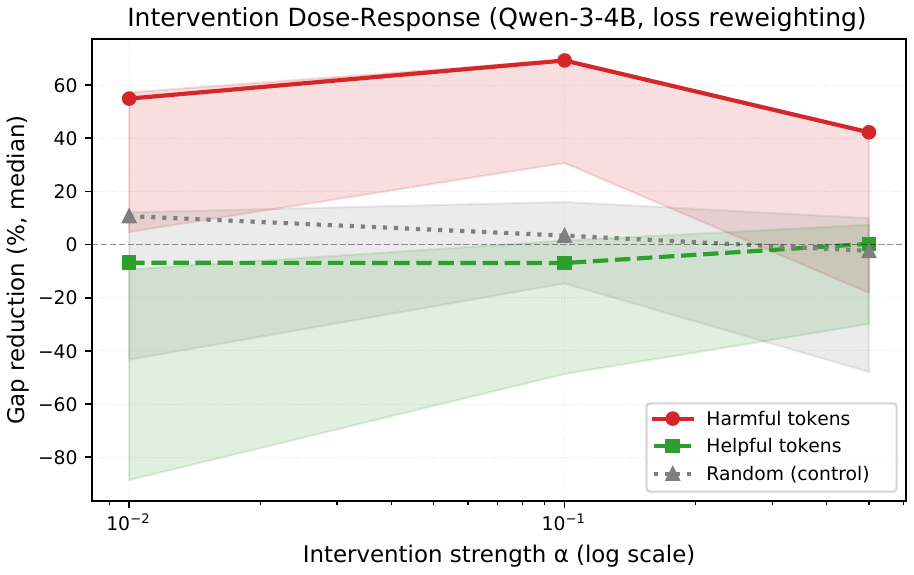}
\caption{Dose-response of loss reweighting on Qwen-3-4B (an earlier sweep, run before the held-out evaluation of Table~\ref{tab:interventions}). Median gap reduction (\%) against the weight $\alpha$ applied to harmful (red), helpful (green), and random (gray) token sets; $\alpha<1$ downweights. Downweighting harmful tokens gives a positive response at every tested $\alpha<1$, while random controls remain near zero. Shaded bands show 95\% confidence intervals (30 trials each).}
\label{fig:dose_response}
\end{figure}

\subsection{Learning Rate Scaling}

Modify the per-token learning rate by scaling gradients for lm\_head rows corresponding to harmful tokens:
\begin{equation}
g_{\text{scaled}}[k, :] = \begin{cases}
\alpha \cdot g[k, :] & \text{if } k \in \text{top-10 harmful} \\
g[k, :] & \text{otherwise}
\end{cases}
\end{equation}

\paragraph{Difference from loss reweighting.} Loss reweighting modifies the loss landscape (affects both the hidden-state and the weight pathway). Learning rate scaling modifies only the gradient applied to lm\_head weights (weight pathway). For small-gap models, the more targeted learning-rate scaling can reduce noise by avoiding amplification through the hidden-state pathway.

\subsection{Projection (PCGrad-Inspired)}

Project harmful token gradients orthogonal to helpful token gradients:
\begin{equation}
g_{\text{proj}}[k, :] = g[k, :] - \frac{g[k, :] \cdot g_{\text{helpful}}[k, :]}{\|g_{\text{helpful}}[k, :]\|^2} g_{\text{helpful}}[k, :]
\end{equation}
for $k \in$ top-10 harmful tokens, where $g_{\text{helpful}}$ is the mean gradient for top-10 helpful tokens (the opposite sign).

\paragraph{Rationale.} Inspired by PCGrad~\citep{yu2020gradient}, this prevents harmful token learning from interfering with helpful token learning. In practice, this showed weaker effects than loss reweighting on Qwen-3-4B.

\subsection{Regularization}

Add L2 penalty driving harmful token embedding rows toward initialization:
\begin{equation}
\mathcal{L}_{\text{reg}} = \mathcal{L}_A + \lambda \sum_{k \in \text{harmful}} \|W[k, :] - W_0[k, :]\|^2
\end{equation}
where $W$ is the lm-head weight matrix, $W_0$ its value at $\theta_0$, and $\lambda = 10^{-3}$.

\paragraph{Rationale.} Prevents harmful tokens from accumulating large updates. Less effective than direct gradient manipulation (loss reweighting or learning-rate scaling).

\subsection{Comparative Results}

\begin{table}[h]
\centering
\caption{Intervention type comparison ($\alpha=0.1$, 30 trials/pair). Sep.: \checkmark\ clear harmful-vs-random separation, $\sim$ weak, $\times$ none. $^\dagger$Qwen-3-4B LR scaling, Projection, and Regularizer rows are from earlier runs without held-out evaluation; the Loss reweighting row uses held-out evaluation. The Qwen-2.5-1.5B row is the full-fp32 above-floor subset reported in Table~\ref{tab:interventions}.}
\label{tab:intervention_types}
\begin{tabular*}{\textwidth}{@{\extracolsep{\fill}}lllccc}
\toprule
Model & Precision & Method & Harmful (med.) & Random (med.) & Sep. \\
\midrule
Qwen-3-4B & bf16 & Loss reweighting & +32.0\% & +0.2\% & \checkmark \\
Qwen-3-4B & bf16 & LR scaling$^\dagger$ & +31.2\% & +1.7\% & \checkmark \\
Qwen-3-4B & bf16 & Projection$^\dagger$ & +14.6\% & --5.1\% & $\sim$ \\
Qwen-3-4B & bf16 & Regularizer$^\dagger$ & +8.3\% & +2.9\% & $\times$ \\
\midrule
Qwen-2.5-1.5B & \textbf{fp32} & \textbf{LR scaling} & \textbf{+47.9\%} & \textbf{--0.003\%} & \checkmark \\
\bottomrule
\end{tabular*}
\end{table}

\paragraph{Summary.}
\begin{itemize}
\item \textbf{Qwen-3-4B (bf16)}: Loss reweighting closes a median $32\%$ of the gap on harmful tokens and $\sim$0\% on random tokens under held-out evaluation (disjoint from training data); LR scaling gives a similar effect ($+31.2\%$) in earlier runs without held-out evaluation.
\item \textbf{Qwen-2.5-1.5B (fp32)}: LR scaling yields +47.9\% harmful closure on the 28/90 above-floor trials, with --0.003\% random closure. We treat this as a conditioned fp32 result and do not compare its magnitude directly to Qwen-3-4B.
\item \textbf{Attribution term}: Selecting tokens with the full bracket $H_Bg_A-H_Ag_B$ (Cohen's $d = 0.53$) outperforms the $H_Bg_A$ term alone ($d = 0.25$), confirming that both Hessian-vector product terms contribute to the ordering signal.
\end{itemize}

\section{What remains after the token edit}
\label{app:residual_gap}

Table~\ref{tab:interventions} evaluates one targeted edit: during the update on source $A$, positions whose target is one of the ten harmful tokens defined in Appendix~\ref{app:intervention_types} are downweighted by $90\%$, and other positions are rescaled to keep the mean weight unchanged. The edit changes only the source-$A$ update; it neither subtracts the predicted commutator displacement nor forces the $A$-then-$B$ and $B$-then-$A$ endpoints to agree. The $32\%$ median closure is therefore the effect of this edit, not the fraction of the ordering gap explained by commutator memory, and its complement is not a separately identified $68\%$ component.

To test whether higher-order terms carry the remainder, we applied the one-update-per-source Qwen-3-4B intervention protocol across $90$ trials (three source pairs, $30$ trials each) and compared three quantities after reweighting: (i) the held-out NLL gap predicted by the leading commutator; (ii) the linearized NLL gap obtained from the parameter difference produced by the two actual updates, which includes finite-step path effects; and (iii) the NLL gap measured at the two endpoints. For each source pair we flip signs so that the original gaps are positive, divide the summed leading-order predictions by the summed magnitudes of the original gaps, and average the three pair-level ratios equally.

After the edit, the leading-order commutator still predicts $71\%$ of the pre-intervention measured gap, and it points in the original direction for every source-pair average. The step from (i) to (ii) is where finite-step corrections, including any net higher-order BCH effect, enter: it reduces the remaining effect for code-vs-news and code-vs-math but increases it for code-vs-legal, so these corrections do not explain a common remainder. The step from (ii) to (iii), the loss-surface correction, reduces the gap for all three pairs. Although $71\%$ is numerically close to $68\%$, the two are different statistics: $68\%$ is the complement of the median absolute closure in Table~\ref{tab:interventions}, whereas $71\%$ is a leading-order prediction normalized within each source pair and then averaged across pairs.

\section{Magnus expansion validation}
\label{app:magnus}

Using the BCH/Magnus expansion logic underlying classical exponential integrators~\citep{magnus1954exponential,blanes2009magnus}, we compare, on Qwen-3-4B, the measured held-out gap $\mathcal{L}_E(\theta_{AB})-\mathcal{L}_E(\theta_{BA})$ after $k$ SGD steps per source with its leading bracket prediction $k^2\eta^2\langle g_E(\theta_{\rm ref}),b_{AB}(\theta_0)\rangle$ (Eq.~\ref{eq:target_score}). The validation uses 12 points (3 seeds $\times$ 4 step counts), with the same $A$ and $B$ batches repeated at every step.

\begin{table}[h]
\centering
\caption{Magnus/bracket validation details (Qwen-3-4B). The slope $0.706$ in Fig.~\ref{fig:magnus} is the pooled regression slope over $k\in\{1,2,4,8\}$; the $k=1$ ratio is reported separately because the BCH expansion is most accurate there.}
\label{tab:magnus}
\begin{tabular}{lc}
\toprule
Metric & Value \\
\midrule
Plotted points & $n=12$ \\
$R^2$ of the $k^2$ fit (mean over seeds) & 0.994 \\
Pooled predicted-vs-actual $R^2$ & 0.990 \\
Best-fit slope in pooled plot & 0.706 \\
Mean magnitude ratio at $k=1$ & $1.004 \pm 0.015$ \\
\bottomrule
\end{tabular}
\end{table}

The measured gap grows nearly as $k^2$, the scaling of the leading bracket term. The pooled slope is below one because higher-order BCH terms grow with $k$: the measured-to-predicted ratio falls from $1.004$ at $k=1$ to $0.97$, $0.90$, and $0.74$ at $k=2$, $4$, and $8$.

\begin{figure}[h]
\centering
\includegraphics[width=0.6\textwidth]{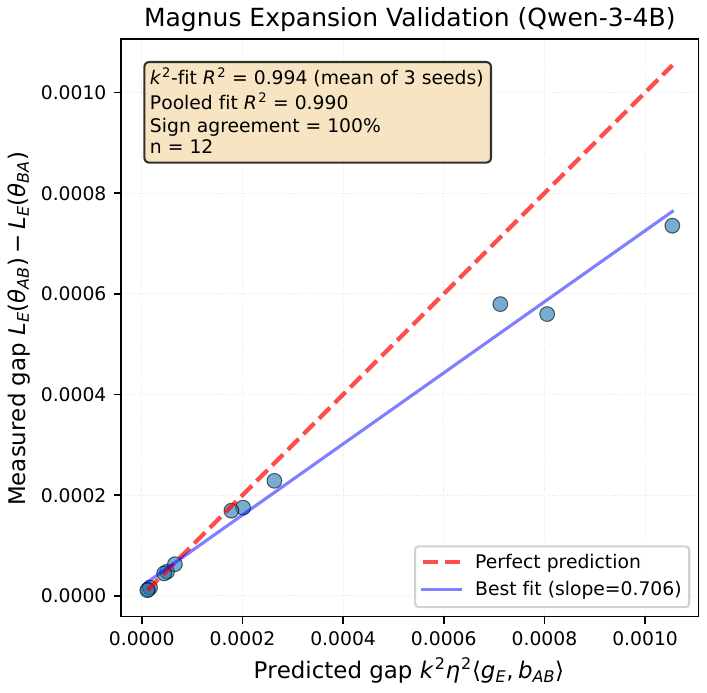}
\caption{Magnus/bracket validation (Qwen-3-4B). Measured held-out gap $\mathcal{L}_E(\theta_{AB})-\mathcal{L}_E(\theta_{BA})$ against its leading bracket prediction $k^2\eta^2\langle g_E(\theta_{\rm ref}),b_{AB}\rangle$ for 3 seeds and $k\in\{1,2,4,8\}$ SGD steps per source. The $k^2$ fit has mean $R^2=0.994$; the pooled predicted-vs-actual fit has slope $0.706$ and $R^2=0.990$; signs agree at all 12 points.}
\label{fig:magnus}
\end{figure}

\section{Per-Trial Results}
\label{app:per_trial}

Per-trial records for the Qwen-3-4B held-out intervention run (90 trials across 3 domain pairs $\times$ 30 seeds) and the Qwen-2.5-1.5B float32 run (90 trials) are included in the supplementary \texttt{runs/} directory with columns: trial, seed, domain\_A, domain\_B, gap\_baseline, gap\_intervened, gap\_reduction\_pct.

\section{Interpretability extended}
\label{app:interpretability_extended}

\begin{figure}[h]
\centering
\includegraphics[width=\textwidth]{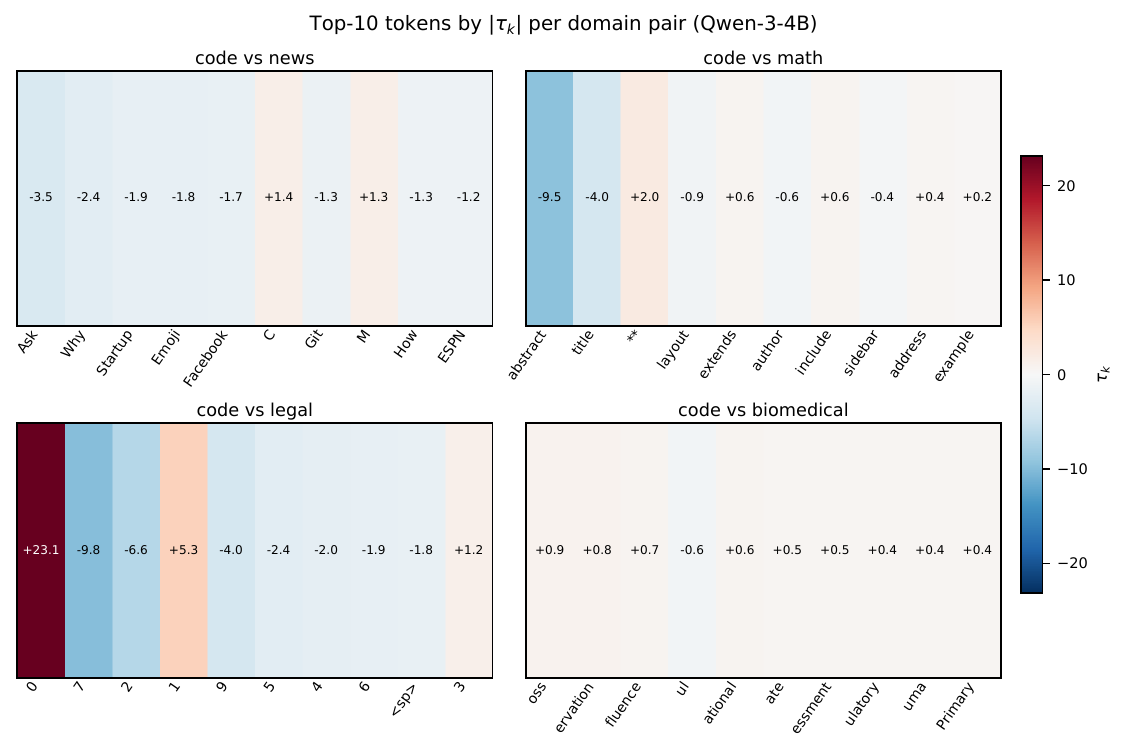}
\caption{Top-10 tokens by $|\tau_k|$ for each domain pair (Qwen-3-4B). Color encodes signed $\tau_k$ (red = order AB disadvantages token; blue = advantages). Tokens are almost entirely disjoint across pairs (Jaccard = 0.004), confirming domain specificity. Code vs.\ legal is dominated by digit tokens; code vs.\ biomedical by BPE suffix fragments.}
\label{fig:token_heatmap}
\end{figure}

\begin{figure}[h]
\centering
\includegraphics[width=\textwidth]{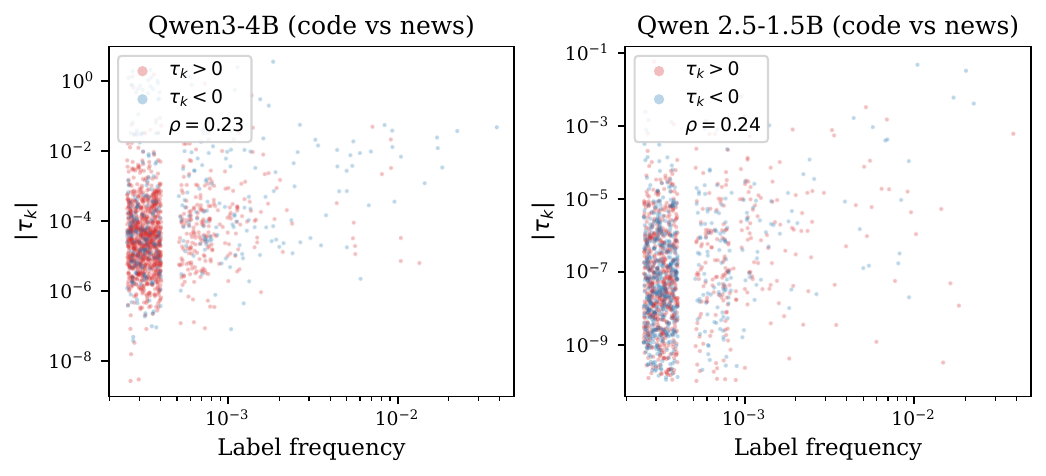}
\caption{Label frequency vs.\ $|\tau_k|$ for all vocabulary tokens (code vs.\ news). High-$|\tau_k|$ tokens span the full frequency range; low rank correlation ($\rho \approx 0.23$) argues against a Zipf-only explanation.}
\label{fig:freq_vs_tau}
\end{figure}

\begin{figure}[h]
\centering
\footnotesize
\setlength{\fboxsep}{1pt}\setlength{\fboxrule}{0.3pt}
\newcommand{\tkn}[1]{\fcolorbox{gray!45}{white}{\strut #1}}
\newcommand{\tkb}[1]{\fcolorbox{gray!45}{blue!20}{\strut #1}}
\newcommand{\tkr}[1]{\fcolorbox{gray!45}{red!20}{\strut #1}}
\begin{tabular}{@{}p{0.17\textwidth}p{0.79\textwidth}@{}}
\toprule
Slice (run) & Tokens, and the readout of the colored ones \\
\midrule
code/news & \tkb{Ask} \tkn{H}\tkn{N}\tkn{:} \tkn{Should} \tkn{I} \tkn{include} \tkn{my} \tkn{GPA} \tkn{and}\tkn{/or} \tkn{transcripts} \tkn{when} \tkn{applying} \tkn{for} \tkn{jobs}\tkn{?} \newline \texttt{Ask}: $\tau=-3.51$ (rank $1$) \\
\addlinespace
code/biomedical & \tkn{Comput}\tkn{ers} \tkn{are} \tkn{a} \tkn{ubiquitous} \tkn{part} \tkb{of} \tkn{the} \tkn{amb}\tkr{ulatory} \tkn{health} \tkn{care} \tkn{environment}\tkn{.} \newline \texttt{ulatory}: $\tau=+0.437$ (rank $8$); \texttt{of}: $\tau=-0.364$ (rank $19$) \\
\addlinespace
code/news (iterative run) & \tkb{Facebook} \tkn{Dating} \tkn{launch} \tkn{blocked} \tkn{in} \tkn{Europe} \tkn{after} \tkn{it} \tkn{fails} \tkn{to} \tkn{show} \tkn{privacy} \tkn{workings} \newline \texttt{Facebook}: $\tau=-1.75$ (rank $5$); its NLL is $10.074$ under AB and $10.896$ under BA, and $10.920$ after bracket correction ($97.0\%$ of its NLL gap closed) \\
\bottomrule
\end{tabular}
\caption{Tokenwise view of the held-out sentences in Section~\ref{sec:interpretability}, with one box per Qwen-3-4B token. A token is colored if it is among its pair's $20$ largest-$|\tau|$ tokens in the slice-level readout (rank in parentheses): blue for $\tau<0$ (order AB advantages the token) and red for $\tau>0$ (AB disadvantages it), as in Fig.~\ref{fig:token_heatmap}. Uncolored tokens are outside that set. The readout is aggregated over the evaluation slice (evaluation seed 44) rather than scored per sentence. The last row also tracks the \texttt{Facebook} position through the instrumented iterative-correction run, which targets $\theta_{BA}$ and overshoots it slightly (by $0.025$).}
\label{fig:examples}
\end{figure}

\paragraph{Per-pair token characters.} Code vs.\ news: 17/20 top tokens are complete words, predominantly proper nouns and brand names (``Ask'', ``Facebook'', ``ESPN'', ``Amazon'', ``Washington''); 87\% of $|\tau_k|$ mass. Code vs.\ math: 17/20 are complete words -- academic metadata (``abstract'' $\tau{=}-9.5$, ``title'' $\tau{=}-4.0$, ``author'', ``layout'', ``sidebar'') reflecting LaTeX source structure, plus code keywords (``extends'', ``pragma'', ``include''). Code vs.\ legal: 9/20 are single digit tokens contributing 92\% of mass; token ``0'' alone has $\tau=+23.1$.

\paragraph{Model scale.} Comparing Qwen-3-4B (4B) to Qwen-2.5-1.5B (1.5B), the same domain pairs show $|\tau_k|$ magnitudes 19--3700$\times$ larger in the larger model. The concentration pattern is consistent, while the absolute readout strength is larger in the larger model.

\paragraph{Prediction/label decomposition.} Since $e_k = p_k - \mathbb{1}_{y=k}$, $\tau_k = \tau^p_k - \tau^y_k$ where $\tau^p_k = \mathbb{E}[p_k \cdot \delta z_k]$ (predictions) and $\tau^y_k = \mathbb{E}[\mathbb{1}_{y=k} \cdot \delta z_k]$ (label occurrence). Aggregate $|\tau^p|/|\tau^y|=0.88$; both Gini $> 0.97$. 5 of top-20 are pure-prediction tokens with $\tau^y_k=0$, and $30\%$ of mass comes from such unobserved tokens. Cancellation between terms is rare (11\% of tokens). Smoothed frequency correction $\kappa_k = \tau_k/(\Pr(y{=}k)+\varepsilon)$ is reported as a frequency diagnostic (Spearman $0.98$, top-100 overlap $80\%$ within this check).

\paragraph{Token transfer within a model family (Qwen-3-4B vs.\ Qwen-2.5-1.5B; Llama-3.1-8B vs.\ Llama-3.2-1B).} Top-500 overlap is $28.6\times$ above the equal-size baseline on Qwen pairs ($47/500$ on code-vs-news; $19/500$ on code-vs-math, $11.5\times$), and $9.7\times$ on Llama pairs ($19/500$ on code-vs-math). Shared tokens are semantically coherent domain markers (e.g., ``Facebook'', ``Linux'', ``Programming''). Qwen and Llama use different tokenizers; cross-family is quantified only in semantic categories.

\paragraph{Rank vs.\ magnitude.} Overlap statistics measure rank-order agreement within each model's $|\tau|$ distribution; they do not imply absolute-magnitude convergence. Top-$20$ mean $|\tau|$ differs by $4+$ orders of magnitude across the four models ($\approx 1.31$ on Qwen-3-4B, $8.8\times 10^{-3}$ on Qwen-2.5-1.5B, $1.2\times 10^{-2}$ on Llama-3.1-8B, $4.3\times 10^{-5}$ on Llama-3.2-1B); $\sim\!30{,}000\times$ residual after $\eta^2$ normalization.

\paragraph{Why these categories.} The three token categories (domain markers, BPE morpheme fragments, digit asymmetry) all identify tokens with large error-weighted logit responses $e_k\,\delta z_k$ along the bracket. Domain-exclusive tokens carry large responses; morpheme fragments aggregate them via BPE compression; digit tokens reflect differences in how domains use numbers. The category counts and $|\tau|$-mass fractions are computed on the top-50 tokens per pair from the Qwen-3-4B main run; a held-out re-run reproduces the same top-token ranking on the two pairs available in both runs.

\section{Tables: ordering baselines, multi-step persistence, scaling in $k$}
\label{app:tables}

\paragraph{Ordering-prediction baselines.}\label{app:baselines}
\begin{table}[h]
\centering
\caption{Ordering prediction accuracy on 18 pair-seed units (6 domain pairs $\times$ 3 seeds per model).}
\label{tab:baselines}
\begin{tabular}{lccc}
\toprule
Method & Qwen-3-4B & Qwen-2.5-1.5B & Llama-3.1-8B \\
\midrule
Bracket & \textbf{72\%} (13/18) & \textbf{100\%} (18/18) & \textbf{100\%} (18/18) \\
Grad Cosine & 11\% (2/18) & 33\% (6/18) & 22\% (4/18) \\
Grad Norm & 11\% (2/18) & 56\% (10/18) & 22\% (4/18) \\
Random & 44\% (8/18) & 50\% (9/18) & 39\% (7/18) \\
\bottomrule
\end{tabular}
\end{table}

\paragraph{Multi-step persistence.}\label{app:multistep}
\begin{table}[h]
\centering
\caption{Multi-step persistence with fresh batches per step (averaged over 12 experiments per model). $\tau$ recall@1\%V is the fraction of the empirical top-1\% set (among tokens occurring as labels) at step count $k$ contained in the single-step bracket top-1\% full-vocabulary $\tau$ set. Qwen-3-4B uses bf16 storage with fp32 HVPs; Qwen-2.5-1.5B is full fp32. gap/gap$_1$: measured gap relative to $k=1$; $k^2$: the leading-order prediction.}
\label{tab:multistep}
\begin{tabular}{lllcccc}
\toprule
Model & Storage/HVP & $k$ & Gini & $\tau$ recall@1\%V & gap/gap$_1$ & $k^2$ \\
\midrule
\multirow{4}{*}{Qwen-3-4B} & \multirow{4}{*}{bf16/fp32}    & 1 & 0.9998 & 100\% & 1.0 & 1 \\
                           &                          & 2 & 0.9998 & 100\% & 1.5 & 4 \\
                           &                          & 4 & 0.9998 & 100\% & 2.0 & 16 \\
                           &                          & 8 & 0.9997 & 100\% & 3.1 & 64 \\
\midrule
\multirow{4}{*}{Qwen-2.5-1.5B} & \multirow{4}{*}{fp32} & 1 & 0.9982 & 99\% & 1.0 & 1 \\
                               &                       & 2 & 0.9982 & 99\% & 4.0 & 4 \\
                               &                       & 4 & 0.9983 & 98\% & 17.8 & 16 \\
                               &                       & 8 & 0.9983 & 98\% & 57.1 & 64 \\
\bottomrule
\end{tabular}
\end{table}

This recall statistic is intentionally asymmetric and is not the same as the equal-cardinality support-recovery metric in Table~\ref{tab:token_prediction}. Table~\ref{tab:token_prediction} compares predicted and empirical top sets of the same size within the active vocabulary for single-step token prediction; Table~\ref{tab:multistep} asks whether the empirical high-effect active tokens across multiple step counts remain covered by one fixed top-$1\%$ full-vocabulary $\tau$ set. Table~\ref{tab:multistep} draws fresh batches at every step, so its gap growth ($3.1\times$ at $k{=}8$ on Qwen-3-4B) is not comparable with the repeated-batch $k^2$ check of Appendix~\ref{app:magnus}.

\paragraph{$\tau$-vs-$\Delta W$ scaling in $k$.}\label{app:scaling_in_k}
The $\tau$-vs-$\Delta W$ Spearman correlation decays log-linearly across the six Qwen-3-4B domain pairs at $k \in \{1, 4, 8\}$: $\rho(k) \approx 0.83 - 0.055 \cdot \log k$ (mean decay rate $0.055 \pm 0.015$, $r^2 > 0.95$). The bracket's Lemma 2.1 expansion is a $k=1$ object; higher-order terms accumulate at roughly this rate for $k\le 8$.

\section{Iterative correction: per-position trajectories and downstream metrics}
\label{app:iterative_extended}

Iterative correction starts at $\theta_{AB}$. Each iteration recomputes the bracket $b$ at the current point $\theta$ on fresh $A$ and $B$ batches, sets $\gamma=\langle\theta-\theta_{BA},b\rangle/\langle b,b\rangle$, and updates the trained parameters by $\theta\leftarrow\theta-\gamma b$. It stops at $20$ iterations, or earlier if the remaining held-out gap falls below $1\%$ of the original or closure changes by less than $0.5$ percentage points over three iterations.

\begin{figure}[h]
\centering
\includegraphics[width=0.85\linewidth]{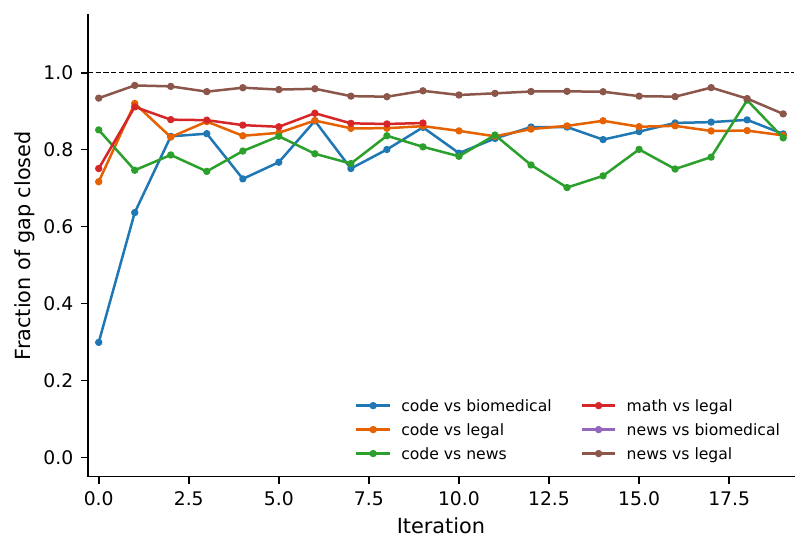}
\caption{Iterative bracket correction convergence (averaged per pair). Five of six pairs reach $>$85\% closure within 20 iterations; news-vs-biomedical ($k{=}1$ gap $8.5\times10^{-4}$, near the noise floor) diverges. On the other five pairs the iterative projection closes $90\%$, against $13.5\%$ for the single-step full-bracket edit averaged over all pairs and step counts ($6.7\times$; App.~\ref{app:endpoint_perpair}).}
\label{fig:convergence}
\end{figure}

\paragraph{Per-position movement.} Instrumented runs track NLL of individual sequence positions throughout the correction. For each experiment we select the 50 positions with the largest $|\text{NLL}(\theta_{AB}) - \text{NLL}(\theta_{BA})|$ and measure their NLL at every iteration. Mean per-position convergence is $85\%$ across five domain pairs; top-gap tokens (e.g., ``Netflix'', ``Facebook'', ``Google'' for code-vs-news) reach $93$--$97\%$ of their $\theta_{BA}$ targets. Figure~\ref{fig:trajectories} shows representative trajectories.

\begin{figure}[h]
\centering
\includegraphics[width=0.85\linewidth]{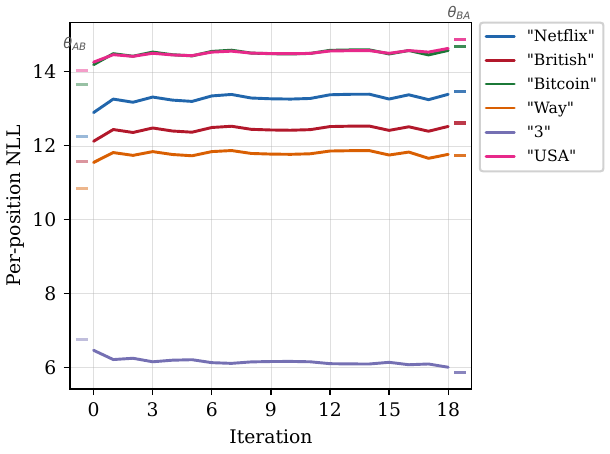}
\caption{Per-position NLL trajectories during iterative bracket correction (code-vs-news, seed=42, $k=1$). Each line tracks one sequence position's NLL; dashed lines show $\theta_{BA}$ targets.}
\label{fig:trajectories}
\end{figure}

\paragraph{Downstream likelihood detail.} On a representative code-vs-math pair on Qwen-3-4B at $k=1$, gold-answer teacher-forced NLL closes $14.2\%$ on GSM8K (1319 problems, 1000-iterate bootstrap CI $[0.126, 0.157]$) and $20.4\%$ on HumanEval (164 problems, CI $[0.178, 0.232]$). Discrete accuracy (pass@1, exact match) moves by at most $7$ percentage points at $k=1$; these changes are not tested for significance.

\paragraph{Selectivity diagnostic.} Across the 34 experiments that did not diverge (all but the two news-vs-biomedical seeds with negative closure), domain-$A$ NLL moves to $99.3\% \pm 3\%$ of $\theta_{BA}$'s value on $A$ and domain-$B$ NLL to $109.9\% \pm 15\%$ of $\theta_{BA}$'s value on $B$. The selectivity ratio $|\mathrm{movement}_B/\mathrm{movement}_A|$, where $\mathrm{movement}_D$ is the change in domain-$D$ NLL from $\theta_{AB}$ to the corrected endpoint, has median $1.04$, mean $1.12$ ($\pm 0.32$).

\FloatBarrier
\section{AdamW augmented-state commutator: endpoint check}
\label{app:adamw}

The Lie bracket structure of \citet{geometry_sequential_learning} is proved for SGD; the Adam family's per-coordinate normalization~\citep{kingma2014adam} and AdamW's decoupled weight decay~\citep{loshchilov2019decoupled} do not commute with the gradient-only commutator. We define an \emph{augmented-state} commutator on the lifted state $z = (\theta, m, v, s)$ (parameters, first- and second-moment buffers, step counter $s$):
\[
\tau_{\text{Adam}} = u(g_A; z_0) + u(g_B; z_A^*) - u(g_B; z_0) - u(g_A; z_B^*),
\]
where $u(g;z)$ is the AdamW parameter update produced by gradient $g$ from optimizer state $z$, and $z_A^*$ advances the optimizer buffers and step counter under $g_A$ but freezes $\theta$ at $\theta_0$ (and symmetrically for $z_B^*$). The associated frozen-surrogate state pair $\hat\theta_{AB}, \hat\theta_{BA}$ satisfies $\hat\theta_{AB} - \hat\theta_{BA} = -\eta\,\tau_{\text{Adam}} + O(\eta^2)$, the AdamW analogue of Lemma 2.1's $O(\eta^2)$ SGD identity. Up to sign, $\eta\,\tau_{\text{Adam}}$ is the first-order term of Theorem~1 in \citet{sweeney2026optimizermemory} applied to the two orders $AB$ and $BA$; that term comes from replaying the optimizer state with parameters frozen.

\paragraph{Fixed-clock regime (cited).} \citet{sweeney2026optimizermemory} characterizes the regime in which this check operates; we summarize and apply, but do not re-prove, two of its results:
\begin{itemize}
\item \textbf{Lifted-state clock theorem (\citealp{sweeney2026optimizermemory}, Thm.~1).} Fixed-$\beta$ moment buffers and de-biasing counters advance with the step index rather than with the learning-rate-scaled time $\eta k$, so the lifted-state step map is not a regular family $F_D^\eta = I + \eta X_D + O(\eta^2)$: the buffers keep changing even as the parameter displacement is scaled down. For such fixed-clock state, including full frozen-gradient AdamW replay, reordering the same data (in particular, swapping $AB$ for $BA$) changes the endpoint by $\eta$ times the difference of the two frozen-parameter optimizer-state replays, plus $O(\eta^2)$. The effect is therefore $\Theta(\eta)$ whenever that difference is nonzero. For regular (memoryless) optimizers the first-order term cancels, and the leading effect is the $O(\eta^2)$ bracket.
\item \textbf{Matched-clock restoration (\citealp{sweeney2026optimizermemory}, \S3 and App.~C.6).} For first-moment memory with a frozen coordinate preconditioner, matching both the memory and de-biasing clocks to the learning-rate-scaled time $\eta k$ ($\beta = e^{-a\eta}$, $s_0 = T_c/\eta$, with $a$ and $T_c$ fixed constants of that control) flattens the optimizer's per-gradient position weights to an $O(\eta)$ spread, restoring the regular $O(\eta^2)$ scaling. The supporting experiment is a regularized first-moment control, not full AdamW. The matched clock is not itself a regular family, since $\beta$ and $s_0$ depend on $\eta$; standard fine-tuning recipes use fixed-clock state.
\end{itemize}
\citet[\S5 and App.~C.4]{sweeney2026optimizermemory} measure this split in LoRA fine-tuning: under fixed-clock AdamW, the order effect has fitted $\eta$-slopes of $0.910$ (Pythia-1B) and $0.915$ (Llama-3.2-1B), compared with $1.960$--$2.005$ for SGD and $1.968$--$1.973$ for a matched-clock first-moment control. Adding only a fixed-$\beta$ momentum buffer moves the slope from $1.988$ (buffer-free SGD) to $0.998$. The paper's bracket-direct DPO edit and GRPO-style closure check (Appendices~\ref{app:dpo_extended} and~\ref{app:grpo}) are therefore restricted to SGD, where the bracket identity is exact at $O(\eta^2)$; the $O(\eta)$ fixed-clock regime requires the lifted-state object $\tau_{\text{Adam}}$.

\paragraph{Midpoint scoring (Remark).} The standard scalar score $\langle g_E(\theta_0), \theta_{AB}-\theta_{BA}\rangle$ has $O(\eta^2)$ error in fixed-$\beta$ AdamW because the symmetric drift $\bar d := \tfrac12(\theta_{AB}+\theta_{BA}) - \theta_0$ is $O(\eta)$ and the order split $\Delta_{\rm split} := \theta_{AB}-\theta_{BA}$ is $O(\eta)$, so the cross term $\bar d^\top H_E \Delta_{\rm split}$ contributes an $O(\eta^2)$ residual. Scoring at the midpoint $\theta_m := \tfrac12(\theta_{AB}+\theta_{BA})$ instead gives
\[
\mathcal{L}_E(\theta_{AB}) - \mathcal{L}_E(\theta_{BA}) \;=\; \langle \nabla \mathcal{L}_E(\theta_m), \theta_{AB} - \theta_{BA}\rangle + O(\eta^3),
\]
because the second-order Taylor term cancels exactly by symmetry around $\theta_m$. This is the AdamW analogue of the $\theta_{\rm ref}$ drift-cancellation identity that makes Proposition~\ref{prop:bo} clean in the SGD regime. It is a loss-linearization identity at the observed endpoints; an $O(\eta^3)$ AdamW correction step, which we leave to follow-up work, would also need an endpoint prediction accurate to that order.

\paragraph{Order-prediction transfer (cited).} \citet[App.~D.7]{geometry_sequential_learning} evaluates the unmodified SGD-derived order predictor against AdamW ground truth over $1{,}020$ Llama-3.2-1B trials. Sign accuracy is $47.7\%$, $51.3\%$, $57.1\%$, and $46.7\%$ at $k{=}5$, $10$, $20$, and $50$ AdamW steps per domain; an exploratory Adam-aware augmented-state score reaches $59.6\%$ at $k{=}10$. The SGD bracket is therefore at most weakly informative under AdamW, which is why the checks below use the lifted-state object $\tau_{\text{Adam}}$ rather than the SGD commutator.

\paragraph{AdamW gradient-space check (our experiment).} On Qwen-3-4B bf16 across six UltraFeedback source-pairs ($\eta=10^{-4}$, $\beta_1=0.9$, $\beta_2=0.999$, weight decay $10^{-2}$, last-layer down/up/gate + lm\_head, 463M params, 20-step warmup (AdamW steps before the two source updates), 3 seeds each), the augmented-state readout is concentrated in the SGD-comparable gradient-space metric:
\begin{itemize}
\item \textbf{Reweighted Gini} (multiply each row by $\sqrt{\hat v}$ before computing Gini---this undoes AdamW's per-coordinate rescaling and recovers the SGD-comparable metric): median $\mathbf{0.980}$ (bf16; range $0.977$--$0.982$ across pairs); $0.989$ on a fp32 control. Compare SGD's gradient-space $0.97$.
\item \textbf{Gradient-difference Gini} (Gini of the lm\_head row of $g_A - g_B$, a first-order proxy for the bracket direction in \emph{gradient-space row-norm structure} only, not a parameter-space edit; cf.\ the first-order controls below, where applying $g_A-g_B$ as a parameter step at norm-matched magnitude widens the ordering gap): mean $0.974$. This is consistent with the row-norm structure that drives the SGD readout.
\item \textbf{Cross-pair top-$1000$ Jaccard}: $\leq 0.051$ across all $15$ pair$\times$pair combinations (mean $0.024$, min $0.000$). Top tokens are pair-specific, matching the cross-pair Jaccard $\approx 0.004$ of SFT $\tau$ (Section~\ref{sec:interpretability}). The high-mass support is domain-specific, not universal, although the leading singular direction is shared across pairs (next item).
\item \textbf{Stable rank ($\|\cdot\|_F^2/\|\cdot\|_2^2$) of the $\tau_{\text{Adam}}$ subspace} mean $6.08$ across $18$ pair-seed units (range $5.26$--$6.75$); on a matched-row-$L_2$ Gaussian random control the stable rank is $963$, a $\sim$160$\times$ compression. The augmented-state commutator is low-rank, consistent with the cross-pair compositionality observed for SGD.
\item \textbf{Shared leading singular direction across pairs}: pairwise cosine between leading singular directions has off-diagonal mean $0.996$ across $306$ entries of the $18 \times 18$ matrix (min $0.989$, max $0.998$); same-pair mean $0.997$, different-pair mean $0.995$. The leading bracket direction is largely shared; pair-specificity emerges in the rank-$\sim$6 subspace beyond the leading direction. This is reported as a descriptive observation in this AdamW check.
\item \textbf{Raw lm\_head Gini} (without $\sqrt{\hat v}$ reweighting): $0.49$--$0.55$ -- \emph{by AdamW design}, not a sparsity loss. AdamW's update is $\hat m / \sqrt{\hat v}$; the $1/\sqrt{\hat v}$ factor intentionally rescales coordinates with high second-moment estimate, spreading the per-coordinate update magnitudes. The gradient-space Gini above removes that rescaling and shows the concentration metric in SGD-comparable coordinates.
\item \textbf{Top tokens} (raw): \texttt{"???"}, \texttt{" Neither"}, \texttt{" substance"}, \texttt{" seasons"}, \texttt{" Mostly"} on \texttt{false\_qa\_vs\_flan\_v2\_p3} ---a mix of formatting and content tokens consistent with DPO's chosen-vs-rejected structure.
\end{itemize}

\paragraph{Lifted-state endpoint closure (our experiment).} On Qwen-2.5-1.5B fp32 with AdamW ($\beta_1{=}0.9$, $\beta_2{=}0.999$, weight decay $10^{-2}$, $\eta{=}5{\times}10^{-5}$, $20$-step warmup, last-layer MLP $\sim\!41\text{M}$ params: \texttt{gate\_proj}, \texttt{up\_proj}, \texttt{down\_proj} on layer $27$), the framework's lifted-state object $\tau_{\text{Adam}}$ predicts the AdamW endpoint difference in direction and magnitude across $6$ source pairs $\times\,3$ seeds $=18$ experiments. The closure edit $\theta_{AB} \leftarrow \theta_{AB} + \lambda\eta\,\tau_{\text{Adam}}$ (sign convention: $\theta_{AB} - \theta_{BA} \approx -\eta\,\tau_{\text{Adam}}$, so closure \emph{adds}):
\begin{itemize}
\item \textbf{Parameter cosine} between $\theta_{BA}-\theta_{AB}$ and $\eta\,\tau_{\text{Adam}}$: mean $\mathbf{0.998}$ (the framework's predicted closure direction matches the actual endpoint difference direction at near-unit cosine).
\item \textbf{Relative endpoint error} at $\lambda{=}1$: $r_{\rm end}:=\|\Delta_{\rm end}-\eta\,\tau_{\text{Adam}}\|/\|\Delta_{\rm end}\|$ with $\Delta_{\rm end}=\theta_{BA}-\theta_{AB}$ has mean $\mathbf{0.068}$, i.e., the predicted edit leaves a residual of about $7\%$ of the endpoint difference (the companion statistic $1-r_{\rm end}^2$ has mean $\mathbf{0.995}$).
\item \textbf{NLL gap closure} at $\lambda{=}1$: mean $\mathbf{79.5\%}$, median $\mathbf{83.8\%}$, $95\%$ CI $\mathbf{[72.6\%, 85.9\%]}$.
\item \textbf{Sign-convention controls}: random norm-matched edit $-0.03\%$ CI $[-0.17\%, 0.10\%]$ (null); wrong-sign edit $\theta_{AB} - \eta\,\tau_{\text{Adam}}$ drives closure to $-96.6\%$ CI $[-102\%, -90\%]$ ($\sim\!2\times$ the gap, the predicted reflection).
\item \textbf{First-order baseline (our experiment)}: applying $g_A - g_B$ at the warmed-up state $\theta_{\rm warm}$, norm-matched to $\|\eta\,\tau_{\text{Adam}}\|$ with parameter-space $\cos(g_A-g_B,\,\tau_{\text{Adam}})=0.43$ (s.d.\ $0.02$) across all $18$ pair-seed units, drives closure to $\mathbf{-668\%}$ at $\lambda{=}1$ (median $-523\%$, CI $[-839\%, -503\%]$); the $\lambda$-scan is monotonic ($\lambda{=}0.5{:}-252\%$, $1.5{:}-1088\%$, $2.0{:}-1511\%$). $1-\cos^2 \approx 0.81$ of the first-order edit's energy is orthogonal to $\tau_{\text{Adam}}$, exciting high-curvature eval directions that wrong-sign (which stays along $\tau_{\text{Adam}}$) does not. The same gap-widening pattern under first-order norm-matched edits appears in the DPO (App.~\ref{app:r79_lambda}) and GRPO-style (App.~\ref{app:grpo}) controls.
\item \textbf{$\lambda$-scan} for $\tau_{\text{Adam}}$ ramps cleanly: $\lambda{=}0$ $\to 0\%$, $\lambda{=}0.5 \to 47\%$, $\lambda{=}1 \to 80\%$, $\lambda{=}1.5 \to 44\%$ (overshoot), $\lambda{=}2 \to -2\%$ (full overshoot). Optimum at the predicted $\lambda{=}1$ closure step.
\end{itemize}
This endpoint closure check gives an AdamW analogue when the SGD bracket $b_{AB}$ is replaced by the augmented-state commutator $\tau_{\text{Adam}}$ on the lifted state $z = (\theta, m, v, s)$. The control hierarchy (random $\approx 0$, wrong-sign $\approx -1$, first-order $\ll -1$, $\tau_{\text{Adam}} \to +0.80$) mirrors the DPO and GRPO-style pattern.

\paragraph{AdamW summary.} The unmodified SGD bracket is at most weakly predictive of AdamW order ($57.1\%$ at $k{=}20$; \citealp[App.~D.7]{geometry_sequential_learning}), consistent with the fixed-clock regime above. Our AdamW experiments add a gradient-space concentration analogue on Qwen-3-4B (Gini $0.980$, stable rank $\sim$6, Jaccard $\leq 0.051$) and a lifted-state endpoint closure check on Qwen-2.5-1.5B fp32 (parameter cosine $0.998$, relative endpoint error $0.068$, NLL gap closure $79.5\%$ $[72.6\%, 85.9\%]$). The main DPO edit and GRPO-style closure check use SGD, where the BCH identity is exact at $O(\eta^2)$; an AdamW analogue of this edit with midpoint scoring is left to follow-up work.

\paragraph{Implementation.} The supplementary code includes \texttt{src/lie\_fisher/adamw\_commutator.py}, with tests for agreement against \texttt{torch.optim.AdamW} on single steps and for the closure-edit sign convention. AdamW correction during training remains follow-up work.

\section{Continuation study: retention and attenuation}
\label{app:continuation}

\paragraph{Protocol.} This study measures how long the order-specific parameter component survives later training. It uses Qwen-3-4B at the step size of the main SFT runs ($\eta=6.42\times10^{-4}$), training the last-block MLP and the language-model head. It covers all ten source pairs over \{code, news, legal, biomedical, math\} (the six pairs over \{code, news, legal, biomedical\} and the four pairs that include math), each with two new seeds, for $20$ runs. After the two source updates, the endpoints $\theta_{AB}$ and $\theta_{BA}$ are both continued with SGD on a third domain $C\notin\{A,B\}$ (Table~\ref{tab:continuation}), using the same sixteen disjoint $50$-record blocks in the same order, with gradients recomputed at each branch's current parameters. With $D_h=\theta_{AB}^{(h)}-\theta_{BA}^{(h)}$ after $h$ later updates, we track
\[
M_h=\frac{\langle D_h,b_{AB}\rangle}{\eta^2\|b_{AB}\|^2},\qquad R_h=\frac{\langle D_h,b_{AB}\rangle}{\langle D_0,b_{AB}\rangle}.
\]
$M_h$ is the trace on the scale of the predicted initial effect, and $R_h$ is the signed fraction of the initial trace that remains. A run starts with an aligned trace if $M_0>0.01$, and it retains the aligned trace at horizon $h$ if also $M_h>0.01$. The criterion is one-sided: a run whose projection changes sign does not retain the aligned trace, even when the reversed component is large. No run is excluded from any count.

\begin{table}[t]
\centering
\small
\caption{Continuation study, all $20$ runs (Qwen-3-4B; SGD continuation on domain $C$). $M_0$ is the initial trace on the predicted scale, and $R_h$ is the signed fraction retained after $h$ later updates. $^{\circ}$: the run does not retain the aligned trace at that horizon ($M_h\le0.01$, including sign reversals). The two \texttt{news/math} runs do not start with an aligned trace ($M_0\le0.01$), so $R_h$ is not reported for them. Rows above the middle rule are non-math pairs; rows below include math.}
\label{tab:continuation}
\begin{tabular}{@{}lllrrrr@{}}
\toprule
Pair $(A,B)$ & $C$ & Seed & $M_0$ & $R_4$ & $R_8$ & $R_{16}$ \\
\midrule
\texttt{code/news} & legal & 8042 & 0.663 & $+$0.286 & $+$0.175 & $+$0.134 \\
\texttt{code/news} & legal & 9042 & 0.530 & $+$0.291 & $+$0.219 & $+$0.113 \\
\texttt{code/legal} & news & 8042 & 2.513 & $-$0.433$^{\circ}$ & $-$0.309$^{\circ}$ & $-$0.170$^{\circ}$ \\
\texttt{code/legal} & news & 9042 & 2.337 & $-$0.348$^{\circ}$ & $-$0.257$^{\circ}$ & $-$0.157$^{\circ}$ \\
\texttt{code/biomedical} & news & 8042 & 0.721 & $+$0.583 & $+$0.446 & $+$0.292 \\
\texttt{code/biomedical} & news & 9042 & 0.577 & $+$0.576 & $+$0.439 & $+$0.283 \\
\texttt{news/legal} & code & 8042 & 1.221 & $+$0.220 & $+$0.144 & $+$0.063 \\
\texttt{news/legal} & code & 9042 & 1.121 & $+$0.191 & $+$0.121 & $+$0.055 \\
\texttt{news/biomedical} & math & 8042 & 0.599 & $+$0.289 & $+$0.248 & $+$0.217 \\
\texttt{news/biomedical} & math & 9042 & 0.555 & $+$0.365 & $+$0.339 & $+$0.295 \\
\texttt{legal/biomedical} & math & 8042 & 2.023 & $-$0.010$^{\circ}$ & $-$0.015$^{\circ}$ & $-$0.015$^{\circ}$ \\
\texttt{legal/biomedical} & math & 9042 & 2.200 & $+$0.027 & $+$0.030 & $+$0.053 \\
\midrule
\texttt{code/math} & legal & 8042 & 0.275 & $+$0.197 & $+$0.145 & $+$0.086 \\
\texttt{code/math} & legal & 9042 & 0.033 & $+$0.502 & $+$0.016$^{\circ}$ & $-$0.170$^{\circ}$ \\
\texttt{news/math} & biomedical & 8042 & 0.005 & --- & --- & --- \\
\texttt{news/math} & biomedical & 9042 & $-$0.006 & --- & --- & --- \\
\texttt{legal/math} & biomedical & 8042 & 0.557 & $+$0.190 & $+$0.093 & $+$0.022 \\
\texttt{legal/math} & biomedical & 9042 & 0.436 & $+$0.252 & $+$0.094 & $+$0.022$^{\circ}$ \\
\texttt{biomedical/math} & code & 8042 & 0.158 & $+$0.455 & $+$0.428 & $+$0.287 \\
\texttt{biomedical/math} & code & 9042 & 0.081 & $+$0.433 & $+$0.327 & $+$0.301 \\
\midrule
\multicolumn{3}{@{}l}{Runs with an aligned trace} & $18/20$ & $15/20$ & $14/20$ & $13/20$ \\
\multicolumn{3}{@{}l}{Median $R_h$ over the $18$ runs that start aligned} & & $0.269$ & $0.144$ & $0.074$ \\
\bottomrule
\end{tabular}
\end{table}

\paragraph{Pass rule.} The rule was fixed before the runs. A horizon passes only if all of the following hold: (i) at least $15$ of the $20$ runs retain the aligned trace, including at least $9$ of the $12$ non-math runs and $6$ of the $8$ math-paired runs; (ii) at least $8$ of the $10$ pairs have at least one seed that retains it; (iii) the median $R_h$ over runs that start aligned is positive; and (iv) the continuation trains on $C$: at least $32$ of the $40$ branches lower their loss on $C$ while moving at least as far as the original source updates (i.e., for each branch $r\in\{AB,BA\}$, $\|\theta_r^{(h)}-\theta_r^{(0)}\|$ is at least the root-mean-square norm of the two source updates), and the median relative drop in $C$ loss is at least $0.1\%$.

\paragraph{Results.} Table~\ref{tab:continuation} lists every run. Eighteen of the $20$ runs start with an aligned trace; the two news-vs-math runs do not. The rule is met after four later updates ($15/20$ runs; $9/12$ non-math and $6/8$ math-paired) but not after eight ($14/20$; $5/8$ math-paired) or sixteen ($13/20$; $4/8$ math-paired). The non-math count stays at $9/12$ throughout, so both failures come from the math-paired runs. Condition (iv) holds at every horizon: all $40$ branches lower their loss on $C$; $35$, $40$, and $40$ of them also meet the displacement condition; and the median relative drop in $C$ loss is $2.2\%$, $3.1\%$, and $4.1\%$. Over the $18$ runs that start aligned, the median $R_h$ falls from $0.27$ to $0.14$ to $0.07$ after four, eight, and sixteen later updates. Both code-vs-legal runs reverse sign within four updates; they fail the one-sided criterion but keep a large reversed component ($M_4=-1.09$ and $-0.81$). The aligned trace therefore persists over a short horizon and then attenuates in the median, with sign reversals in some runs. This does not establish behavioral retrieval or long-term memory.

\FloatBarrier
\section{Cost accounting and cheaper curvature approximations}
\label{app:cost}

\paragraph{Exact matrix-free cost.} The exact calculation never forms or stores a Hessian. For one source pair it computes two gradients and two Hessian-vector products (HVPs); with each HVP costing roughly two forward--backward passes (Appendix~\ref{app:reproducibility}), the bracket costs about six forward--backward equivalents for the reported parameter block. The scalar score requires one additional held-out gradient, and the token readout adds three forward-only passes. The number of passes per pair is fixed, although the cost of each pass grows with model size, sequence length, sample count, and the number of parameters included. The paper applies this exact calculation to the reported parameter blocks at model sizes through 8B.

\paragraph{Exact reuse across pairs.} Suppose $N$ candidate sources share a base model $\theta_0$ and a held-out evaluation loss, with $g_D$, $H_D$, and $g_E$ evaluated at $\theta_0$. By Hessian symmetry, the leading predicted gap for pair $(A,B)$ is $\eta^2[g_A^\top(H_Bg_E)-g_B^\top(H_Ag_E)]$. Each $H_Dg_E$ is therefore computed once per source, so ranking all pairs among $N$ sources needs $N$ HVPs instead of $N(N-1)$, plus inexpensive dot products. This is the identity behind the Lie-Bracket Tournament of \citet[\S3.1]{geometry_sequential_learning}, and it is exact for the fixed-base leading-order score. On the Qwen-3-4B code-vs-news test, reused and direct scores agree to a maximum relative error of $1.6\times10^{-5}$. If different pairs use different evaluation losses, reuse applies only within groups sharing the same loss. The paper's primary score instead evaluates the held-out gradient at the pair-specific reference $\theta_0-\eta(g_A+g_B)$; replacing it by the base-model gradient changes the predicted gap at $O(\eta^3)$~\citep[Remark~2.6]{geometry_sequential_learning}, the order already neglected by the second-order approximation. The fixed-base formula can therefore screen all pairs cheaply; the pair-specific score, exact matrix-free HVPs, and token readout are then computed pair by pair.

\paragraph{Cheaper curvature approximations.} A second route avoids HVPs with a streamed diagonal curvature approximation. For the final vocabulary-prediction layer, with hidden states held fixed, the layer's own contribution to the Hessian diagonal at weight $(k,j)$ can be accumulated in one pass from $p_k(1-p_k)h_j^2$, where $p_k$ is the probability of token $k$ and $h_j$ is hidden-state coordinate $j$. This discards cross-coordinate interactions. Block-diagonal or low-rank generalized Gauss--Newton approximations~\citep{martens2015optimizing} retain more of those interactions and provide intermediate cost--fidelity trade-offs, but are not automatically HVP-free. The relevant fidelity tests are whether an approximation preserves held-out ordering signs and the largest token effects. All reported pair-specific scores and token readouts use exact matrix-free HVPs.

\section{Reproducibility}
\label{app:reproducibility}

The supplementary materials include source code, protocol definitions, and result files for reported claims. Run IDs in the paper are stable identifiers; artifact locations are documented in the supplement and need not correspond one-to-one to local directory names. The \texttt{src/lie\_fisher/} package contains the bracket, attribution, and intervention machinery; \texttt{src/lie\_fisher/protocols.py} contains the frozen configuration records (\texttt{Protocol} dataclasses) used by the reported intervention protocols.

\subsection{Software environment}

\begin{itemize}
\item Reported runs used CUDA/PyTorch/Transformers environments recorded in released run metadata; \texttt{requirements.txt} gives the reproduction baseline. The run envelope is Python 3.11--3.12, PyTorch 2.x, Transformers 4/5 depending on run, NumPy 1.26--2.x, and matplotlib 3.x.
\item One CUDA backend per run (no multi-GPU sharding for the reported experiments). Float32 is enforced for bracket/HVP/readout paths that require the fp32 precision policy in Appendix~\ref{app:precision}.
\item Seed control: each experiment locks \{Python, NumPy, PyTorch CPU, PyTorch CUDA\} seeds at the value reported in \texttt{run\_meta.base\_seed} of the corresponding result JSON. Per-trial seeds are derived deterministically from \texttt{base\_seed} via the experiment's protocol.
\end{itemize}

\subsection{Models and dataset revisions}

\begin{itemize}
\item \textbf{Qwen-3-4B} (\texttt{Qwen/Qwen3-4B}, 4B params); \textbf{Qwen-3-8B} for DPO concentration checks; \textbf{Qwen-2.5-1.5B} (\texttt{Qwen/Qwen2.5-1.5B}, 1.5B params); \textbf{Qwen-2.5-0.5B-Instruct} for the GRPO-style surrogate; \textbf{Llama-3.1-8B} (\texttt{meta-llama/Llama-3.1-8B}, 8B params); \textbf{Llama-3.2-1B} (\texttt{meta-llama/Llama-3.2-1B}, 1B params). All models loaded via \texttt{transformers.AutoModelForCausalLM} with \texttt{trust\_remote\_code=False}.
\item SFT training/eval data: The Pile~\citep{gao2020pile}, six domain partitions: \texttt{code} (StackExchange), \texttt{news} (CC-News), \texttt{legal} (FreeLaw), \texttt{biomedical} (PubMed Central), \texttt{math} (StackExchange Math), \texttt{wikipedia}. Section~\ref{sec:setup} lists the domain pairs used by each SFT experiment; the continuation study uses the ten pairs over $\{$code, news, legal, biomedical, math$\}$ (Appendix~\ref{app:continuation}). Eval offset $= 50$ documents past the training partition for held-out evaluation (Section~\ref{sec:setup}).
\item DPO training/eval data: UltraFeedback~\citep{cui2024ultrafeedback}, source partitions: \texttt{false\_qa}, \texttt{flan\_v2\_p3}, \texttt{sharegpt}, \texttt{ultrachat} for the original six-pair grid; \texttt{evol\_instruct}, \texttt{truthful\_qa} for the fresh-pairs replication (App.~\ref{app:r79_freshpairs}). The token-reweighting fresh-pair check uses the three source pairs evaluable under the released snapshot; the matched-batch fresh-pairs grid uses a pair-disjoint source-pair list.
\end{itemize}

\subsection{Compute budget and hardware}

\begin{table}[h]
\centering
\small
\caption{Indicative GPU-hours and hardware per experiment family. Total project compute is driven primarily by the DPO matched-batch runs and the endpoint-correction runs; sparsity/baseline runs are lower-cost. The rows sum to about 32 GPU-hours; the project total is cumulative over all runs. Run identifiers are listed under Repository pointers below.}
\label{tab:compute_budget}
\begin{tabular}{p{0.21\textwidth}p{0.17\textwidth}p{0.10\textwidth}p{0.40\textwidth}}
\toprule
Experiment family & Hardware & GPU-hours & Notes \\
\midrule
Sparsity (4 models)         & RTX 4090 / Pro 6000 & $\sim$2  & R8b, R9, R13\_sparsity, R14\_sparsity, R15\_sparsity, R16 \\
Baselines (4 models)        & RTX 4090 / A40      & $\sim$1  & R8b, R9, R13 \\
SFT interventions           & RTX Pro 6000 / RTX 4090 & $\sim$3 & R14a, R14b, R15 (full fp32 for Qwen-2.5) \\
Endpoint correction / multi-step        & A40 / L40 / Pro 6000    & $\sim$5 & R26, R27, R28, R19, R20 \\
Iterative bracket correction & L40                     & $\sim$3 & R29, R30 (instrumented per-position NLL) \\
DPO sparsity (3 models)     & A40 / Pro 6000          & $\sim$2 & R50, R52, R60 \\
DPO token reweighting & RTX Pro 6000 WS         & $\sim$3 & R78: 6 original pairs $+$ three-pair fresh check, 3 seeds $\times$ 20 trials \\
DPO matched-batch correction & RTX Pro 6000 WS / H100 NVL & $\sim$10 & R79: 6 original pairs (3 seeds $\times$ 10 trials) $+$ 6 disjoint fresh pairs (3 seeds $\times$ 10) \\
Frozen-rollout GRPO-style & A100 80GB / RTX-class & $\sim$2 & R81/R82 matched-rollout checks \\
AdamW endpoint closure       & RTX 4090 24GB           & $\sim$1 & R86 + R87 first-order baseline \\
\midrule
Project total (cumulative)  & mixed academic GPUs     & $\sim$48 & approximately USD 50--55 in cumulative spot-rental cost \\
\bottomrule
\end{tabular}
\end{table}

The main per-experiment cost is HVP computation (each Hessian-vector product is roughly $2\times$ a forward-backward pass; the bracket requires two HVPs). The hardware table is indicative; released result files correspond to the reported protocols.

\subsection{Fixed protocols and conflict checks}

Each intervention reported in the paper has a corresponding frozen configuration record (\texttt{Protocol} dataclass) in \texttt{src/lie\_fisher/protocols.py} that fixes \{model, dtype, $\eta$, $k$, batch composition, dose-match policy, identity-check tolerance, top-$K$ filter, control set\}. The runner aborts when the protocol's immutable fields conflict with command-line overrides. Released \texttt{run\_meta} blocks record the active protocols for the final reported runs.

\subsection{Repository pointers}

\begin{itemize}
\item Bracket and HVP machinery: \texttt{src/lie\_fisher/lie.py} (HVP via \texttt{torch.autograd.grad} on the gradient inner product, in-module).
\item Per-token attribution: \texttt{src/lie\_fisher/head\_moments.py} (SFT) and \texttt{compute\_token\_attribution\_dpo} (corrected DPO $\tau$).
\item Winsorized aggregation (5/95 mean + bootstrap CI): \texttt{src/lie\_fisher/stats.py:robust\_aggregate}.
\item $\eta$ calibration (cube-root regime): \texttt{src/lie\_fisher/eta\_autopilot.py}.
\item Protocol definitions: \texttt{src/lie\_fisher/protocols.py}.
\item \textbf{Run identifiers} (used in the compute table above and as directory names in the archive): R8b (Qwen-3-4B main run: sparsity, interpretability, baselines, Magnus check), R9 (Qwen-2.5-1.5B main run), R13/R16 (Llama runs), R8c/R10 (earlier Qwen-3-4B intervention sweeps), R14a/R14b (Qwen-3-4B held-out SFT interventions), R15 (Qwen-2.5-1.5B fp32 SFT interventions), R19/R20 (multistep persistence), R26/R28 (endpoint correction on Qwen-2.5-1.5B/Qwen-3-4B), R27 (triplet additivity), R29/R30 (iterative correction; R30 instrumented per position), R44 (follow-up $\theta_0$-reference readouts), R50/R52/R60 (DPO concentration), R71 (the original six DPO source pairs), R78 (DPO token reweighting with the corrected $\tau$), R79 (matched-batch DPO correction), R80 (AdamW gradient-space check), R81/R82 (GRPO-style single-step and block-sequential checks), R86/R87 (AdamW lifted-state closure and first-order baseline).
\end{itemize}

\section{Frozen-rollout GRPO-style test: protocol and controls}
\label{app:grpo}
\paragraph{Protocol.} These experiments test whether the BCH-local bracket remains structured in a frozen-rollout, KL-regularized reward-surrogate setting inspired by GRPO/PPO. We use Qwen-2.5-0.5B-Instruct, freeze the sampled rollouts shared by the two orders, and define two analytic reward sources: math correctness and brevity/format compliance. The bracket $b_{AB}$ is computed at $\theta_0$ on the realized rollouts, and endpoints $\theta_{AB},\theta_{BA}$ are trained on the matched rollout records. This is a matched-rollout, KL-regularized GRPO-style surrogate with analytic rewards, using the GRPO/PPO objective family as protocol motivation~\citep{shao2024deepseekmath,schulman2017ppo}.

\paragraph{Single-step matched-rollout BCH check.} Across 20 trials in 5 seed clusters, the endpoint-identity check is tight: the cosine between the measured $\theta_{AB}-\theta_{BA}$ and $\eta^2 b_{AB}$ (identity cosine) has median $0.999$ (mean $0.998$, CI $[0.996,0.999]$), and the projection coefficient $\zeta=\langle\theta_{AB}-\theta_{BA},\eta^2 b_{AB}\rangle/\|\eta^2 b_{AB}\|^2$ has median $1.002$ (mean $1.002$, CI $[0.995,1.009]$). The finite-difference HVP checks have minimum HVP cosine median $0.999998$ and maximum residual median $0.0026$; PPO/GRPO clipping is inactive in both orders. All checks pass together in $17/20$ trials ($0.85$).

\begin{table}[h]
\centering
\small
\caption{Single-step frozen-rollout GRPO-style controls. Values are medians over 20 trials unless a CI is shown. Unmatched rollouts: the bracket from one rollout set applied to endpoints trained on another. Reported $r$ are closure ratios $r := 1 - \|\text{edited}-\text{target}\|/\|\text{baseline}-\text{target}\|$ against the target AB-vs-BA displacement on held-out rollouts.}
\label{tab:grpo_r81_controls}
\begin{tabular}{lccc}
\toprule
Condition & Parameter closure $r$ & Log-prob closure $r$ & Loss closure $r$ \\
\midrule
Bracket edit & $0.952$ & $0.886$ & $0.805$ \\
Random norm-matched & $-0.414$ & $-0.001$ & $0.010$ \\
Wrong sign & $-0.997$ & $-0.987$ & $-1.008$ \\
First-order norm-matched & $-0.410$ & $-0.197$ & $-2.384$ \\
Unmatched rollouts (negative control) & $-0.396$ & $-0.286$ & $-0.158$ \\
\bottomrule
\end{tabular}
\end{table}

\paragraph{Block-sequential $T$-sweep.} The block-sequential sweep composes the correction over longer matched-rollout blocks and selects the largest $T$ in the fixed grid $\{2,4,8,16\}$ satisfying identity cosine $>0.7$, $\zeta\in[0.5,1.5]$, and parameter-space closure $C_\theta>0.6$, without reward-outcome tuning. All evaluable block lengths pass; $T^\star=16$ is selected by this rule.

\begin{table}[h]
\centering
\small
\caption{Block-sequential frozen-rollout $T$-sweep. $C_\theta$ is the parameter-space closure of the composed correction.}
\label{tab:grpo_r82_sweep}
\begin{tabular}{rrrrr}
\toprule
$T$ & identity cosine & $\zeta$ & relative residual & $C_\theta$ \\
\midrule
2  & $0.997$ & $0.967$ & $0.083$ & $0.917$ \\
4  & $0.998$ & $0.904$ & $0.122$ & $0.878$ \\
8  & $0.988$ & $0.818$ & $0.267$ & $0.733$ \\
16 & $0.991$ & $0.884$ & $0.188$ & $0.812$ \\
\bottomrule
\end{tabular}
\end{table}

\paragraph{Interpretation.} The GRPO-style result is strongest as a closure check: in a third objective regime, the same bracket satisfies a near-exact matched-path identity, wrong-sign edits obey antisymmetry, random controls stay near zero in log-probability and loss, and static cross-rollout controls fail. The caveat is also structural: a bracket computed for one realized path is not expected to transfer to a different path.

\end{document}